\documentclass[twoside,11pt]{article}

\usepackage{jmlr2e}
\usepackage{bm,mathtools,caption}
\usepackage[ruled,vlined]{algorithm2e}

\firstpageno{1}

\begin{document}

\title{Bayesian subset selection and variable importance for interpretable prediction and classification}

\author{\name Daniel R. Kowal \email daniel.kowal@rice.edu \\
       \addr Department of Statistics\\
       Rice University\\
       Houston, TX 77005, USA
       }

\editor{NA}

\maketitle

\begin{abstract}
Subset selection is a valuable tool for interpretable learning, scientific discovery, and data compression. However, classical subset selection is often avoided due to selection instability, lack of regularization, and difficulties with post-selection inference. We address these challenges from a Bayesian perspective. Given any Bayesian predictive model $\mathcal{M}$, we extract a \emph{family} of  near-optimal subsets of variables for linear prediction or classification. This strategy deemphasizes the role of a single ``best" subset and instead advances the broader perspective that often many subsets are highly competitive. The \emph{acceptable family}  of subsets offers a new pathway for model interpretation and is 
neatly summarized by key members such as the smallest acceptable subset, along with new (co-) variable importance metrics based on whether variables (co-) appear in all, some, or no acceptable subsets. More broadly, we apply Bayesian decision analysis to derive the optimal linear coefficients for \emph{any} subset of variables. These coefficients inherit both regularization and predictive uncertainty quantification via $\mathcal{M}$. For both simulated and real data, the proposed approach exhibits better prediction, interval estimation, and variable selection than competing Bayesian and frequentist selection methods. These tools are applied to a large education dataset with highly correlated covariates. Our analysis provides unique insights into the combination of environmental, socioeconomic, and demographic factors that predict educational outcomes, and identifies over 200 distinct subsets of variables that offer near-optimal out-of-sample predictive accuracy.
\end{abstract}


\section{Introduction}

Subset and variable selection are essential components of regression analysis, prediction, and classification. 
By identifying subsets of important covariates, the analyst can acquire 
simpler and more interpretable summaries of the data, 
improved prediction or classification, 
reduced estimation variability among the selected covariates, 
lower storage requirements, and 
insights into the factors that determine predictive accuracy \citep{Miller1984}.  Classical subset selection is expressed as the solution to the constrained least squares problem
\begin{equation}\label{classic}
\min_{\bm \beta} \Vert \bm y - \bm X \bm \beta \Vert_2^2  \quad \mbox{subject to} \quad \Vert \bm \beta \Vert_0 \le k
\end{equation}
where $\bm y$ is an $n$-dimensional response, $\bm X$ is an $n \times p$ matrix of covariates, and $\bm \beta$ is the $p$-dimensional vector of unknown coefficients. Traditionally, the goal is to determine (i) the best subset of each size $k$,  (ii) an estimate of the accompanying nonzero linear coefficients, and (iii) the best subset of any size.  Subset selection has garnered additional attention recently, in part due to the algorithmic advancements from \cite{Bertsimas2016} and the detailed comparisons of  \cite{Hastie2020}. More broadly, variable selection has been deployed as a tool for interpretable machine learning, including for highly complex and nonlinear models  
(e.g., \citealp{Ribeiro2016,Afrabandpey2020}). 

Although often considered the ``gold standard" of variable selection, subset selection remains underutilized due to several critical limitations. First, subset selection is inherently unstable: it is common to obtain entirely distinct subsets under perturbations or resampling of the data. This instability undermines the interpretability of a single ``best" subset, and is a significant motivating factor for the proposed methods. Second, the solutions to \eqref{classic} are unregularized. While it is advantageous to avoid \emph{overshrinkage}, \cite{Hastie2020} showed that the lack of any regularization in \eqref{classic} leads to deteriorating performance relative to penalized regression in low-signal settings. Third, inference about $\bm \beta$ requires careful adjustment for selection bias, which limits the available options for uncertainty quantification. Finally, solving \eqref{classic} is computationally demanding even for moderate $p$, which has spawned many algorithmic advancements spanning multiple decades (e.g., \citealp{Furnival2000,Gatu2006,Bertsimas2016}). For these reasons, penalized regression techniques that replace the $\ell_0$-penalty with convex or nonconvex yet computationally feasible alternatives \citep{Fan2010} are often preferred.

From a Bayesian perspective, \eqref{classic} is usually translated into a predictive model via  a Gaussian (log-) likelihood and a sparsity (log-) prior. Indeed, substantial research efforts have been devoted to both sparsity (e.g., \citealp{ishwaran2005spike}) and shrinkage (e.g., \citealp{polson2010shrink}) priors for $\bm \beta$.  Yet the prior alone cannot \emph{select} subsets:  the prior is the component of the data-generating process that incorporates prior beliefs, information, or regularization, while selection is ultimately a decision problem \citep{lindley1968choice}. For instance, \cite{barbieri2004optimal} and \cite{Liang2013} applied sparsity priors to obtain posterior inclusion probabilities, which were then used for \emph{marginal} selection and screening, respectively. 
\cite{Jin2020} selected subsets using marginal likelihoods, but required conjugate priors for Gaussian linear models. \cite{hahn2015decoupling} more fully embraced the decision analysis approach for variable selection, and augmented a squared error loss with an $\ell_1$-penalty for linear variable selection. Alternative loss functions have been proposed for  seemingly unrelated regressions \citep{Puelz2017}, graphical models \citep{Bashir2019}, nonlinear regressions \citep{woody2019model}, functional regression \citep{kowal2019bayesianfosr}, time-varying parameter models \citep{Huber2020}, and a variety of Kullback-Leibler approximations to the likelihood \citep{goutis1998model,nott2010bayesian,tran2012predictive,piironen2018projective}.  

Despite the appropriate deployment of  decision analysis for selection, each of these Bayesian methods relies on $\ell_1$-penalization or forward search. As such, they are restricted to limited search paths that cannot fully solve the (exhaustive) subset selection problem. More critically, these Bayesian approaches---as well as most classical ones---are unified in their emphasis on 
selecting a \emph{single} ``best" subset. However, in practice it is common for many subsets (or models) to achieve near-optimal predictive performance, known as the \emph{Rashomon effect}  \citep{breiman2001statistical}. This effect is particularly pronounced for correlated covariates, weak signals, or small sample sizes. Under these conditions, it is empirically and theoretically possible for the ``true" covariates to be predictively inferior to a proper subset  \citep{Wu2007}. As a result, the ``best" subset is not only less valuable but also less interpretable. Reporting a single subset---or a small number of subsets along a highly restricted search path---obscures the likely presence of many distinct yet equally-predictive subsets of variables. 


We advance an agenda that instead curates and summarizes a \emph{family} of optimal or near-optimal subsets. This broader analysis alleviates the instability issues of a single ``best" subset and provides  a more complete predictive picture. The proposed approach operates within a decision analysis framework and is compatible with \emph{any} Bayesian model $\mathcal{M}$ for prediction or classification. Naturally, $\mathcal{M}$ should represent the modeler's beliefs about the data-generating process and describe the salient features in the data. Several key developments are required:
\begin{enumerate}  \setcounter{enumi}{0}
\item We derive  \emph{optimal}  (in a decision analysis sense) linear coefficients for any subset of variables.  
\end{enumerate}
Crucially, these coefficients inherit regularization \emph{and} uncertainty quantification via $\mathcal{M}$, but avoid the overshrinkage induced by $\ell_1$-penalization. As such, these point estimators resolve multiple limitations of classical subset selection and ($\ell_1$-penalized) Bayesian decision-analytic variable selection, and further are adapted to both regression and classification problems.   Next,
\begin{enumerate}  \setcounter{enumi}{1}
\item We design a modified branch-and-bound algorithm for \emph{efficient exploration} over the space of candidate subsets. 
\end{enumerate}
The search process is a vital component of subset selection, and our modular framework is broadly compatible with other state-of-the-art search algorithms (e.g., \citealp{Bertsimas2016}). Until now, these algorithms have not been fully deployed or adapted for Bayesian subset selection. Additionally,
\begin{enumerate}  \setcounter{enumi}{2}
\item We leverage the predictive distribution under $\mathcal{M}$ to collect the \emph{acceptable family} of near-optimal subsets.
\end{enumerate}
A core feature of the acceptable family is that it is defined using out-of-sample metrics and predictive uncertainty quantification, yet is computed using \emph{in-sample posterior functionals} from a single model fit of $\mathcal{M}$. Hence, we maintain computational scalability and coherent uncertainty quantification that avoids data reuse. Lastly,
\begin{enumerate}  \setcounter{enumi}{3}
\item We \emph{summarize} the acceptable family with key member subsets, such as the ``best" (in terms of cross-validation error) predictor and the smallest acceptable subset, along with new (co-) variable importance metrics that measure the frequency with which variables (co-) appear in all, some, or no acceptable subsets. 
\end{enumerate}
Unlike variable importances based on effect size, this inclusion-based metric effectively measures how many  ``predictively plausible" explanations (i.e., near-optimal subsets)  contain each (pair of) covariate(s) as a member. Notably, each of these developments is presented for both prediction and classification.

The importance of curating and exploring a collection of subsets has been acknowledged previously. Existing approaches are predominantly frequentist, including fence methods \citep{Jiang2008}, Rashomon sets \citep{Semenova2019}, bootstrapped confidence sets \citep{Lei2019}, and subsampling-based forward selection  \citep{Kissel2021}. Although the acceptable family has appeared previously for Bayesian decision analysis \citep{Kowal2020target,KowalPRIME2020}, it was applied only along the $\ell_1$-path which does not enumerate a sufficiently rich collection of competitive subsets. Further, previous applications of the acceptable family did not address points 1., 2., and 4.\! above.


The paper is outlined as follows. Section~\ref{methods} contains the Bayesian subset search procedures, the construction of acceptable families, the (co-) variable importance metrics, and the predictive uncertainty quantification. Section~\ref{sims} details the simulation study. The methodology is applied to a large education dataset in Section~\ref{app}. We conclude in Section~\ref{disc}. The Appendix provides additional algorithmic details, simulation studies, and results from the application.  An \texttt{R} package is available online at \url{https://github.com/drkowal/BayesSubsets}.  Although the education dataset \citep{ChildrensEnvironmentalHealthInitiative2020} cannot be released due to privacy protections, access to the dataset can occur through establishing affiliation with the Children's Environmental Health Initiative (contact \href{mailto:cehi@nd.edu}{cehi@nd.edu}).

\section{Methods}\label{methods}
\subsection{Predictive decision analysis} 
Decision analysis establishes the framework for extracting actions, estimators, and predictions from a Bayesian model (e.g., \citealp{Bernardo2009}). These tools translate probabilistic models into practical decision-making and can be deployed to summarize or interpret complex models. However, additional methodology is needed to convert subset selection into a decision problem, and further to evaluate and collect many near-optimal subsets.

Let  $\mathcal{M}$ denote any Bayesian model with a proper posterior distribution $p_\mathcal{M}(\bm \theta | \bm y)$ and posterior predictive distribution $$p_\mathcal{M}(\bm{\tilde y} | \bm y ) \coloneqq \int p_\mathcal{M}(\bm{\tilde y} | \bm \theta, \bm y ) p_\mathcal{M}(\bm \theta | \bm y) d\bm\theta,$$ where $\bm \theta$ denotes the parameters of the model $\mathcal{M}$ and $\bm y$ is the observed data.  Informally,  $p_\mathcal{M}(\bm{\tilde y} | \bm y )$  defines the distribution of future or unobserved data $\bm{\tilde y}$ conditional on the observed data $\bm y$ and according to the model $\mathcal{M}$. 
Decision analysis evaluates  each \emph{action} $\bm \delta$ based on a loss function $\mathcal{L}(\bm{\tilde y}, \bm \delta)$ that enumerates the cost of each action when $\bm{\tilde y}$ is realized. Examples include point prediction (e.g., squared error loss) or classification (e.g., cross-entropy loss), interval estimation (e.g., minimum length subject to $1-\alpha$ coverage), and selection among a set of hypotheses (e.g., 0-1 loss). Since $\bm{\tilde y}$ is unknown yet modeled probabilistically under $\mathcal{M}$, an optimal action minimizes the posterior predictive expected loss
\begin{equation}\label{action}
\bm{\hat \delta} \coloneqq \arg\min_{\bm\delta} \mathbb{E}_{\bm{\tilde y} | \bm y} \mathcal{L}(\bm{\tilde y}, \bm \delta)
\end{equation}
with the expectation taken under the Bayesian model $\mathcal{M}$. The operation in \eqref{action} averages the predictive loss over the possible realizations of $\bm{\tilde y}$ according to the posterior probability under $\mathcal{M}$ and then minimizes the resulting quantity over the action space.

Yet without careful specification of the loss function $\mathcal{L}$, \eqref{action} does not provide a clear pathway for subset selection. To see this, let $\bm{\tilde y}(\bm { x})$ denote the predictive variable at covariate  $\bm x$. For point prediction of $\bm{\tilde y}(\bm{ x})$ under squared error loss, the optimal action is 
\begin{align*}
\bm{\hat \delta}(\bm{ x}) & \coloneqq \arg\min_{\bm\delta} \mathbb{E}_{\bm{\tilde y} | \bm y}  \Vert\bm{\tilde y}(\bm{ x}) - \bm \delta(\bm{ x})\Vert_2^2 \\
&=  \mathbb{E}_{\bm{\tilde y} | \bm y} \{\bm{\tilde y}(\bm{ x})\},
\end{align*}
i.e., the posterior predictive expectation at $\bm x$. Similarly, for classification of $\bm{\tilde y}(\bm{ x}) \in \{0,1\}$ under cross-entropy loss (see \eqref{cross-ent}), the optimal action  is the posterior predictive probability $\bm{\hat \delta}(\bm{ x}) =  p_\mathcal{M}\{\bm{\tilde y}(\bm{ x})  = 1 | \bm y\}$.  For a generic model $\mathcal{M}$, there is not necessarily a closed form for $\bm{\hat \delta}(\bm{ x})$: these actions are computed separately for each $\bm x$ with no clear mechanism for inducing sparsity or specifying distinct subsets. Hence, additional techniques are needed to supply actions that are not only optimal but also selective and interpretable.

Note that we use observation-driven rather than parameter-driven loss functions. Unlike the parameters $\bm \theta$, which are unobservable and model-specific, the predictive variables $\bm{\tilde y}$ are observable and directly comparable across distinct Bayesian models. A decision analysis based on $\bm{\tilde y}$ operates on the same scale and in the same units as the data that have been---and will be---observed, which improves interpretability. Perhaps most important, an observation-driven decision analysis enables empirical evaluation of  the selected actions.

\subsection{Subset search for linear prediction}\label{sec-linear}
The subset search procedure is built within the decision framework of \eqref{action}. For any Bayesian model $\mathcal{M}$, we consider \emph{linear} actions  $\bm{\delta}(\bm{ x}) = \bm{ x}'\bm \delta$ with $\bm \delta \in \mathbb{R}^p$, which offer both interpretability and the capacity for selection. Let $\bm \delta_\mathcal{S}$ denote the linear action with zero coefficients for all $j \not \in \mathcal{S}$, where  $\mathcal{S} \subseteq \{1,\ldots,p\}$ is a subset of active variables. For prediction, we assemble the aggregate and weighted squared error loss
\begin{equation}\label{sq-loss}
\mathcal{L}(\{{\tilde y}_i\}_{i=1}^{\tilde n}, \bm \delta_\mathcal{S}) = \sum_{i=1}^{\tilde n} \omega(\bm{\tilde x}_i) \Vert {\tilde y}_i - \bm{\tilde x}_i'\bm \delta_\mathcal{S}\Vert_2^2 
\end{equation}
where ${\tilde y}_i \coloneqq \bm{\tilde y}(\bm{\tilde x}_i)$ is the predictive variable at $\bm{\tilde x}_i$ for each $i=1,\ldots, \tilde n$ and $\omega(\bm{\tilde x}_i) > 0$ is a weighting function. The covariate values  $\{\bm{\tilde x}_i\}_{i=1}^{\tilde n}$ can be distinct from the observed covariate values $\{\bm{ x}_i\}_{i=1}^n$, for example to evaluate the action for a different  population.

The loss in \eqref{sq-loss} evaluates linear coefficients $\bm \delta_\mathcal{S}$ for any given subset $\mathcal{S}$ by accumulating the squared error loss over the covariate values $\{\bm{\tilde x}_i\}_{i=1}^{\tilde n}$. Since \eqref{sq-loss} depends on $\{\tilde y_i\}_{i=1}^{\tilde n}$, the loss inherits a joint predictive distribution $p_\mathcal{M}(\tilde y_1, \ldots, \tilde y_{\tilde n} | \bm y)$ from $\mathcal{M}$.  The loss is decoupled from the Bayesian model $\mathcal{M}$: the linear action does not require a linearity assumption for $\mathcal{M}$. The weights $\omega(\bm{\tilde x}_i)$ can be used to target actions locally, which provides a sparse and  local linear approximation to $\mathcal{M}$. For example, we  might parametrize the weighting function as $\omega(\bm{\tilde x}_i) \propto \exp(- \Vert \bm{\tilde x}_i - \bm{x}^* \Vert_2^2/\ell)$ with range parameter $\ell$ in order to weight based on proximity to some particular $\bm{x}^*$ of interest \citep{Ribeiro2016}. Alternatively, $\omega$ can be specified via a probability model for the likelihood of observing each covariate value, including $\omega(\bm{\tilde x}_i) = \tilde{n}^{-1}$ as a simple yet useful example, especially when using the observed covariate values $\{\bm{\tilde x}_i\}_{i=1}^{\tilde n} = \{\bm{ x}_i\}_{i=1}^n$; this is our default choice.

Crucially, the optimal action can be solved directly for any subset $\mathcal{S}$:
\begin{lemma}\label{sq-loss-action}
Suppose $ \mathbb{E}_{\bm{\tilde y} | \bm y} \Vert \bm{\tilde y}(\bm{\tilde x}_i) \Vert_2^2 < \infty$ for $i=1,\ldots, \tilde n$. The optimal action  \eqref{action} for the loss \eqref{sq-loss} is 
\begin{align}\label{sq-loss-action-1}
\bm{\hat\delta}_\mathcal{S} &= \arg\min_{{\bm \delta}_\mathcal{S}}  \sum_{i=1}^{\tilde n} \omega(\bm{\tilde x}_i)\Vert {\hat y}_i - \bm{\tilde x}_i'\bm \delta_\mathcal{S}\Vert_2^2 \\
\label{sq-loss-action-2}
&=( \bm{\tilde X}_\mathcal{S}'\bm \Omega \bm{\tilde X}_\mathcal{S})^{-1}\bm{\tilde X}_\mathcal{S}'\bm\Omega\bm{\hat y},
\end{align}
where ${\hat y}_i \coloneqq \mathbb{E}_{\bm{\tilde y} | \bm y} \{\bm{\tilde y}(\bm{\tilde x}_i)\}$, $\bm{\hat y} \coloneqq ({\hat y}_1,\ldots, {\hat y}_{\tilde n})'$, $\bm \Omega \coloneqq \mbox{diag}\{\omega(\bm{\tilde x}_i)\}_{i=1}^{\tilde n}$, and $\bm{\tilde X}_\mathcal{S}$ is the $\tilde n \times |\mathcal{S}|$ matrix of the active covariates in $\{\bm{\tilde x}_i\}_{i=1}^{\tilde n}$ for subset $\mathcal{S}$.
\end{lemma}
\begin{proof}
It is sufficient to observe that $\mathbb{E}_{\bm{\tilde y} | \bm y} \Vert \bm{\tilde y}(\bm{\tilde x}_i) - \bm{\tilde x}_i'\bm \delta_\mathcal{S}\Vert_2^2 = \mathbb{E}_{\bm{\tilde y} | \bm y} \Vert \{\bm{\tilde y}(\bm{\tilde x}_i) - \hat y_i\} + (\hat y_i - \bm{\tilde x}_i'\bm \delta_\mathcal{S} )\Vert_2^2 = \mathbb{E}_{\bm{\tilde y} | \bm y} \Vert \bm{\tilde y}(\bm{\tilde x}_i) - \hat y_i  \Vert_2^2 + \Vert\hat y_i - \bm{\tilde x}_i'\bm \delta_\mathcal{S} \Vert_2^2$, where the first term is a (finite) constant that does not depend on $\bm{\delta}_\mathcal{S}$. The remaining steps constitute a weighted least squares solution. 
\end{proof}
The consequence of Lemma~\ref{sq-loss-action} is that the optimal action for each subset $\mathcal{S}$ is simply the (weighted) least squares solution based on pseudo-data $(\bm{\tilde X}_\mathcal{S}, \bm{\hat y})$---i.e., a ``fit to the fit" from $\mathcal{M}$. The advantages of $\mathcal{M}$ can be substantial: the Bayesian model propagates regularization (e.g., shrinkage, sparsity, or smoothness) to the point predictions $\bm{\hat y}$, which typically offers sizable improvements in estimation and prediction relative ordinary (weighted) least squares. This effect is especially pronounced in the presence of high-dimensional (large $p$) or correlated covariates. The optimal action may be non-unique if $\bm{\tilde X}_\mathcal{S}'\bm \Omega \bm{\tilde X}_\mathcal{S}$ is noninvertible, in which case the inverse in \eqref{sq-loss-action-2} can be replaced by a generalized inverse. 

At this stage, the Bayesian model $\mathcal{M}$ is only needed to supply the pseudo-response variable ${\hat y}_i$; different choices of $\mathcal{M}$ will result in distinct values of  ${\hat y}_i$ and therefore distinct actions $\bm{\hat\delta}_\mathcal{S}$. An illuminating special case occurs for  linear regression:
\begin{corollary}\label{cor-linear}
For the linear regression model $\mathcal{M}$ with $\mathbb{E}_{\bm{\tilde y} | \bm \theta} \{\bm{\tilde y}(\bm{ x})\} = \bm{ x}'\bm\beta$ and \emph{any} set of covariate values $\{\bm{\tilde x}_i\}_{i=1}^{\tilde n}$ and weights  $\omega(\bm{\tilde x}_i) > 0$, the optimal action \eqref{action} under \eqref{sq-loss}  for the full set of covariates  is $\bm{\hat \delta}_{\{1,\ldots,p\}} = \bm{\hat \beta}$, where $ \bm{\hat \beta}\coloneqq \mathbb{E}_{\bm \theta | \bm y} \bm \beta $.
\end{corollary} 
Depending on the choice of $\{\bm{\tilde x}_i\}_{i=1}^{\tilde n}$ and $\omega$, $\bm{\hat \delta}_{\{1,\ldots,p\}}$ may be non-unique. Corollary~\ref{cor-linear} links the optimal action to the model parameters:  the posterior expectation $ \bm{\hat \beta} $ is also the optimal action under the parameter-driven squared error loss $\mathcal{L}(\bm\beta, \bm \delta) = \Vert \bm \beta - \bm \delta \Vert_2^2$. Similarly, a linear model for $\mathcal{M}$ implies that ${\hat y}_i = \bm{\tilde x}_i'\bm{\hat \beta}$, so the optimal action \eqref{sq-loss-action-2} for any subset $\mathcal{S}$ is intrinsically connected to the (regression) model parameters. This persists for other regression models as well. 
By contrast, these restrictions also illustrate the generality of \eqref{sq-loss}-\eqref{sq-loss-action-2}: the optimal linear actions are derived explicitly 
 under any model $\mathcal{M}$ (with $ \mathbb{E}_{\bm{\tilde y} | \bm y} \Vert \bm{\tilde y}(\bm{\tilde x}_i) \Vert_2^2 < \infty$) and using any set of covariate values $\{\bm{\tilde x}_i\}_{i=1}^{\tilde n}$, active covariates $\mathcal{S} \subseteq \{1,\ldots, p\}$, and weighting functions $\omega(\bm{\tilde x}) > 0$.

The critical remaining challenge is optimization---or at least evaluation and comparison---among the possible subsets $\mathcal{S}$. Our strategy emerges from the observation that there may be many subsets that achieve near-optimal predictive performance, often referred to as the \emph{Rashomon effect}  \citep{breiman2001statistical}. The goal is to collect, characterize, and compare these near-optimal subsets of linear predictors.  Hence, there are two core tasks: (i) identify candidate subsets and (ii) filter to include only those subsets that achieve near-optimal  predictive performance. These tasks must overcome both computational and methodological challenges---similar to classical (non-Bayesian) subset selection---which we resolve in the subsequent sections.

An exhaustive enumeration of all possible subsets presents an enormous computational burden, even for moderate $p$. Although tempting, it is misguided to consider direct optimization over all possible subsets of $\{1,\ldots,p\}$,
\begin{equation}\label{opt-subset}
\bm{\hat \delta}_{\mathcal{\widehat S}} \coloneqq \arg\min_{\mathcal{S}, \bm{\delta}_\mathcal{S}} \mathbb{E}_{\bm{\tilde y} | \bm y} \mathcal{L}(\{{\tilde y}_i\}_{i=1}^{\tilde n}, \bm \delta_\mathcal{S}),
\end{equation}
for the aggregate squared error loss \eqref{sq-loss}. 
To see this---and find suitable alternatives---consider the following result:
\begin{lemma}\label{rss-ord-lem}
Let $\mbox{RSS}(\bm y, \bm \mu) \coloneqq \Vert \bm y - \bm \mu \Vert_2^2$ and $\bm{\hat y} \coloneqq \mathbb{E}_{\bm{\tilde y} | \bm y} \bm{\tilde y}$, and suppose $ \mathbb{E}_{\bm{\tilde y} | \bm y} \Vert \bm{\tilde y}\Vert_2^2 < \infty$. For any point predictors $\bm \mu_1$ and $\bm \mu_2$, we have 
\begin{equation} \label{rss-ord}
\mathbb{E}_{\bm{\tilde y} | \bm y} \{\mbox{RSS}(\bm{\tilde y}, \bm \mu_1)\} \le \mathbb{E}_{\bm{\tilde y} | \bm y} \{\mbox{RSS}(\bm{\tilde y}, \bm \mu_2)\}  \iff \mbox{RSS}(\bm{\hat y}, \bm \mu_1) \le \mbox{RSS}(\bm{\hat y}, \bm \mu_2 ).
\end{equation}
\end{lemma}
\begin{proof}
Since $\mathbb{E}_{\bm{\tilde y} | \bm y} \{\mbox{RSS}(\bm{\tilde y}, \bm \mu)\} -  \mbox{RSS}(\bm{\hat y}, \bm \mu) = \mathbb{E}_{\bm{\tilde y} | \bm y}  \Vert \bm{\tilde y}\Vert_2^2 - \Vert \bm{\hat y} \Vert_2^2$ is finite and does not depend on $\bm \mu$, the ordering of $\bm \mu_1$ and $\bm \mu_2$ will be identical whether using  $\mathbb{E}_{\bm{\tilde y} | \bm y} \{\mbox{RSS}(\bm{\tilde y}, \bm \mu)\}$ or $\mbox{RSS}(\bm{\hat y}, \bm \mu)$.
\end{proof}
Notably, $\mathbb{E}_{\bm{\tilde y} | \bm y} \{\mbox{RSS}(\bm{\tilde y}, \bm \mu)\}$ is the key constituent in optimizing the predictive squared error loss \eqref{sq-loss}, while $\mbox{RSS}(\bm{\hat y}, \bm \mu)$ is simply the usual residual sum-of-squares (RSS) with $\bm{\hat y}$ replacing $\bm y$.

Now recall the optimization of \eqref{opt-subset}. For any subset $\mathcal{S}$, the optimal action is the least squares solution \eqref{sq-loss-action-2} with pseudo-data $\bm{\hat y}$. However, RSS in linear regression is ordered by nested subsets:  $\mbox{RSS}(\bm{\hat y}, \bm{\tilde X} \bm{\hat \delta}_{\mathcal{S}_1}) \le \mbox{RSS}(\bm{\hat y}, \bm{\tilde X} \bm{\hat \delta}_{\mathcal{S}_2})$ whenever $\mathcal{S}_2 \subseteq \mathcal{S}_1$. By Lemma~\ref{rss-ord-lem}, it follows that the solution of \eqref{opt-subset} is simply 
$$
\bm{\hat \delta}_{\mathcal{\widehat S}} = ( \bm{\tilde X}'\bm \Omega \bm{\tilde X})^{-1}\bm{\tilde X}'\bm\Omega\bm{\hat y}, \quad \mathcal{\widehat S} = \{1,\ldots,p\}
$$ 
for $\bm{\tilde X} = (\bm{\tilde x}_1,\ldots,\bm{\tilde x}_{\tilde n})'$. As with \eqref{sq-loss-action-2}, a generalized inverse can be substituted if necessary. 
The main consequence is that the optimal actions in \eqref{sq-loss-action-1} and \eqref{opt-subset} alone cannot select variables or subsets: \eqref{sq-loss-action-1}  provides the optimal action for a given subset $\mathcal{S}$, while \eqref{opt-subset} trivially returns the full set of covariates.  Despite the posterior predictive expectation in \eqref{opt-subset}, this optimality is only valid in-sample and is unlikely to persist for out-of-sample prediction. Hence, this optimality is unsatisfying.

Yet Lemma~\ref{rss-ord-lem} provides a path forward. Rather than fixing a subset $\mathcal{S}$ in  \eqref{sq-loss-action-1} or optimizing over all subsets in \eqref{opt-subset}, suppose we compare among subsets of a fixed size $k < p$. Equivalently, this constraint can be representation as an $\ell_0$-penalty augmentation to the loss function \eqref{sq-loss}, i.e., 
\begin{equation}\label{opt-subset-k}
\bm{\hat \delta}_k \coloneqq \arg\min_{\{\mathcal{S}:\vert\mathcal{S}\vert \le k\}, \bm{\delta}_\mathcal{S}} \mathbb{E}_{\bm{\tilde y} | \bm y} \mathcal{L}(\{{\tilde y}_i\}_{i=1}^{\tilde n}, \bm \delta_\mathcal{S}),
\end{equation}
where the inequality follows from Lemma~\ref{rss-ord-lem}. In direct contrast with previous approaches for Bayesian variable selection via decision analysis (e.g., \citealp{hahn2015decoupling,woody2019model,KowalPRIME2020}) we do \emph{not}  use convex relaxations to $\ell_1$-penalties, which create unnecessarily restrictive search paths and introduce additional bias in the coefficient estimates. 

Under the loss \eqref{sq-loss}, the solution to \eqref{opt-subset-k} reduces to  
\begin{equation}\label{solve-subset-k}
\bm{\hat \delta}_k =  \arg\min_{\{\mathcal{S}:\vert\mathcal{S}\vert \le k\}} \sum_{i=1}^{\tilde n} \omega(\bm{\tilde x}_i)\Vert {\hat y}_i - \bm{\tilde x}_i'\bm \delta_\mathcal{S}\Vert_2^2 
\end{equation}
using the same argument as Lemma~\ref{sq-loss-action}. In particular, the solution $\bm{\hat \delta}_k$ resembles classical subset selection \eqref{classic}, but uses the fitted values $\hat y_i$ from $\mathcal{M}$ instead of the data $\bm y$ and further generalizes to include possibly distinct covariates $\{\bm{\tilde x}_i\}$ and weights $\{\omega(\bm{\tilde x}_i)\}$. 

Because of the representation in \eqref{solve-subset-k}, we can effectively solve \eqref{opt-subset-k} by leveraging existing algorithms for subset selection. However, our broader interest in the collection of near-optimal subsets places greater emphasis on the search and filtering process. For any two subsets $\mathcal{S}_1$ and $\mathcal{S}_2$ of equal size $k = |\mathcal{S}_1| = |\mathcal{S}_2|$, Lemma~\ref{rss-ord-lem} implies that  $\mathbb{E}_{\bm{\tilde y} | \bm y} \{\mbox{RSS}(\bm{\tilde y}, \bm{\tilde X} \bm\delta_{\mathcal{S}_1})\}  \le \mathbb{E}_{\bm{\tilde y} | \bm y} \{\mbox{RSS}(\bm{\tilde y}, \bm{\tilde X} \bm\delta_{\mathcal{S}_2})\}$ if and only if $\mbox{RSS}(\bm{\hat y}, \bm{\tilde X} \bm\delta_{\mathcal{S}_1}) \le \mbox{RSS}(\bm{\hat y}, \bm{\tilde X} \bm\delta_{\mathcal{S}_2})$. Therefore, we can order the linear actions from \eqref{sq-loss-action-1} among all equally-sized subsets simply by ordering the values of $\mbox{RSS}(\bm{\hat y}, \bm{\tilde X} \bm\delta_{\mathcal{S}})$. This result resembles the analogous scenario in classical linear regression on $\{(\bm x_i, y_i)\}_{i=1}^n$:  subsets of fixed size maintain the same ordering whether using RSS or information criteria such as AIC, BIC, or Mallow's $C_p$. Here, the criterion of interest is the posterior predictive expected RSS, $\mathbb{E}_{\bm{\tilde y} | \bm y} \{\mbox{RSS}(\bm{\tilde y}, \bm{\tilde X} \bm\delta_{\mathcal{S}})\}$, and the RSS reduction occurs with the model $\mathcal{M}$ fitted values $\hat y_i = \mathbb{E}_{\bm{\tilde y} | \bm y} \bm{\tilde y}(\bm{\tilde x}_i)$ serving as pseudo-data. 

Because of this RSS-based ordering among equally-sized linear actions, we can leverage the computational advantages of the \emph{branch-and-bound algorithm} (BBA)  for efficient subset exploration \citep{Furnival2000}.  Using a tree-based enumeration of all possible subsets, BBA avoids an exhaustive subset search by carefully eliminating non-competitive subsets (or branches) according to RSS. BBA is particularly advantageous for  %
(i) selecting $s_{max} \le p$ covariates, 
(ii) filtering to $m_k \le {p \choose k}$ subsets for each size $k$, and 
(iii) exploiting the presence of  covariates that are almost always present for low-RSS models \citep{Miller1984}. 
In the present setting, the key inputs to the algorithm are the covariates $\{\bm{\tilde x}_i\}$, the pseudo-data $\{\hat y_i\}$, and the weights $\{\omega(\bm{\tilde x}_i)\}$. In addition, we specify the maximum number of subsets $m_k \le {p \choose k}$ to return for each size $k$. As $m_k$ increases, BBA returns more subsets---with higher RSS---yet computational efficiency deteriorates. We consider default values of $m_k = 100$ and $m_k = 15$ and compare the results in Section~\ref{app}. An efficient implementation of BBA is available in the  \texttt{leaps} package in \texttt{R}. Note that our framework is also compatible with many other subset search algorithms (e.g., \citealp{Bertsimas2016}).


In the case of moderate to large $p$, we screen to $s_{max} \le p$ covariates using the original model $\mathcal{M}$. Specifically, we select the $s_{max}$ covariates with the largest absolute (standardized) linear regression coefficients. When $\mathcal{M}$ is a nonlinear model, we use the optimal linear coefficients $\bm{\hat \delta}_\mathcal{S}$ on the full subset $\mathcal{S} = \{1,\ldots,p\}$ of (standardized) covariates. The use of marginal screening is common in both frequentist \citep{Fan2008} and Bayesian \citep{bondell2012consistent} high-dimensional linear regression models, with accompanying consistency results in each case. Here, sceening is motivated by {computational scalability} and {interpretability}: BBA is quite efficient for $p \le 35$ and $m_k \le 100$, while the interpretation of a subset of covariates---acting jointly to predict accurately---is muddled as the subset size increases. By default, we fix $s_{max} = 35$.  We emphasize that although this screening procedure relies on marginal criteria, it is based on a joint model $\mathcal{M}$ that incorporates all $p$ covariates. In that sense, our \emph{screening} procedure resembles the most popular Bayesian variable \emph{selection} strategies based on posterior inclusion probabilities  \citep{barbieri2004optimal} or hard-thresholding  \citep{datta2013asymptotic}. 



\subsection{Subset search for logistic classifiers}\label{sec-logit}
Classification and binary regression operate on $\{0,1\}$, rendering the squared error loss  \eqref{sq-loss} unsuitable. Consider a binary predictive functional $h({\tilde y}) \in \{0,1\}$ where ${\tilde y} \sim p_\mathcal{M}({\tilde y} | \bm y )$. In this framework, binarization can come from one of two sources. Most common, the data are binary $y_i \in \{0,1\}$,  paired with the identity functional $h({\tilde y}) = \tilde y$, and $\mathcal{M}$ is a Bayesian classification (e.g., probit or logistic regression) model. Less common, non-binary data can be modeled  via $\mathcal{M}$ and paired with a functional $h$ that maps to  $\{0,1\}$. For example, we may be interested in selecting variables to predict exceedance of a threshold, $h(\tilde y) = \mathbb{I}\{\tilde y \ge \tau\}$, for some $\tau$ based on real-valued data $\bm y$. The latter case is an example of \emph{targeted prediction} \citep{Kowal2020target}, which customizes posterior summaries or decisions for any functional $h$. This approach is distinct from fitting separate models to each empirical functional $\{h(y_i)\}_{i=1}^n$---which is still compatible with the first setting above---and instead requires only a single Bayesian reference model $\mathcal{M}$ for all target functionals $h$. 

For classification or binary prediction of $h(\tilde y) \in \{0,1\}$, we replace the squared error loss \eqref{sq-loss} with the {aggregate} and {weighted} cross-entropy loss,
\begin{equation}\label{cross-ent}
\mathcal{L}[\{h({\tilde y}_i)\}_{i=1}^{\tilde n}, \bm \delta_\mathcal{S}] = \sum_{i=1}^{\tilde n}  \omega(\bm{\tilde x}_i)  \big[ h(\tilde y_i) \log \{\pi_\mathcal{S}(\bm{\tilde x}_i) \} + \{1 - h(\tilde y_i)\} \log\{1 - \pi_\mathcal{S}(\bm{\tilde x}_i) \}\big],
\end{equation}
where $h(\tilde y_i) \in \{0,1\}$ is the predictive variable at $\bm{\tilde x}_i$ under $\mathcal{M}$ and $\pi_\mathcal{S}(\bm{\tilde x}_i) \coloneqq \{1 + \exp(-\bm{\tilde x}_i' \bm \delta_\mathcal{S})\}^{-1}$. The cross-entropy is also the \emph{deviance} or negative log-likelihood of a series of independent Bernoulli random variables $h(\tilde y_i)$ each with success probability $\pi_\mathcal{S}(\bm{\tilde x}_i)$ for $i=1,\ldots, \tilde n$. However, \eqref{cross-ent} does not imply a distributional assumption for the decision analysis; all distributional assumptions are encapsulated within $\mathcal{M}$, including the posterior predictive distribution of  $h(\tilde y_i)$.  As before, $\bm \delta_\mathcal{S}$ is the linear action with zero coefficients for all $j \not \in \mathcal{S}$, where  $\mathcal{S} \subseteq \{1,\ldots,p\}$ is a subset of active variables. 

The optimal action \eqref{action} is obtained for each subset $\mathcal{S}$ by computing expectations with respect to the posterior predictive distribution under $\mathcal{M}$ and minimizing the ensuing quantity. As in the case of squared error loss, key simplifications are available: 
\begin{equation}\label{opt-logit}
\bm{\hat \delta}_\mathcal{S} = \arg\min_{\bm{\delta}_\mathcal{S}} 
\sum_{i=1}^{\tilde n}  \omega(\bm{\tilde x}_i)  \big[ \hat h_i \log \{\pi_\mathcal{S}(\bm{\tilde x}_i) \} + (1 - \hat h_i)\log\{1 - \pi_\mathcal{S}(\bm{\tilde x}_i) \}\big]
\end{equation}
where $\hat h_i  \coloneqq \mathbb{E}_{\bm{\tilde y} | \bm y} h( \tilde y_i)$ is the posterior predictive expectation of $h(\tilde y_i)$ under $\mathcal{M}$. The representation in  \eqref{opt-logit} is quite useful: it is the negative log-likelihood for a logistic regression model with pseudo-data $\{(\bm{\tilde x}_i, \hat h_i)\}_{i=1}^{\tilde n}$. Standard algorithms and software, such as iteratively reweighted least squares (IRLS) in the \texttt{glm} package in \texttt{R}, can be applied to solve \eqref{opt-logit} for any subset $\mathcal{S}$. 

Instead of fitting a logistic regression to the observed binary variables $h(y_i) \in \{0,1\}$, the optimal action under cross-entropy \eqref{cross-ent} fits to the posterior predictive probabilities $\hat h_i  = p_\mathcal{M}\{h(\tilde y_i)  = 1 | \bm y\} \in [0,1]$ under $\mathcal{M}$. For a well-specified model $\mathcal{M}$, these posterior probabilities $\hat h_i$ can be more informative than the binary empirical functionals $h(y_i)$: the former lie on a continuum between the endpoints zero and one. Furthermore, for non-degenerate models $\mathcal{M}$ with  $\hat h_i \in (0,1)$, the optimal action \eqref{opt-logit} resolves the issue of separability, which is a persistent challenge in classical logistic regression.

Despite the efficiency of IRLS for a fixed subset $\mathcal{S}$, the computational savings of BBA for subset search rely on the RSS from a linear model. As such, solving \eqref{opt-logit} for all or many subsets incurs a much greater computational cost. Yet IRLS is intrinsically linked with RSS. At convergence, IRLS obtains a weighted least squares solution 
\begin{equation}\label{irls-est}
\bm{\hat \delta}_\mathcal{S} = ( \bm{\tilde X}_\mathcal{S}'\bm{W} \bm{\tilde X}_\mathcal{S})^{-1}\bm{\tilde X}_\mathcal{S}'\bm{W} \bm{z}
\end{equation}
where $\bm{W}  \coloneqq \mbox{diag}\{w_i\}_{i=1}^{\tilde n}$ for weights $w_i \coloneqq \omega(\bm{\tilde x}_i) \hat \pi_\mathcal{S}(\bm{\tilde x}_i)\{1 - \hat \pi_\mathcal{S}(\bm{\tilde x}_i)\}$, fitted probabilities $\hat \pi_\mathcal{S}(\bm{\tilde x}_i) \coloneqq \{1 + \exp(-\bm{\tilde x}_i' \bm{\hat \delta}_\mathcal{S})\}^{-1}$, and pseudo-data 
\begin{equation}\label{data-0}
z_i \coloneqq  \log \frac{\hat \pi_\mathcal{S}(\bm{\tilde x}_i)}{1-\hat \pi_\mathcal{S}(\bm{\tilde x}_i)}+ \frac{\hat h_i - \hat \pi_\mathcal{S}(\bm{\tilde x}_i)}{\hat \pi_\mathcal{S}(\bm{\tilde x}_i)\{1 - \hat \pi_\mathcal{S}(\bm{\tilde x}_i)\}}
\end{equation}
with $\bm{z} \coloneqq (z_1,\ldots,  z_{\tilde n})'$. By design, the weighted least squares objective associated with \eqref{irls-est} is a second-order Taylor approximation to the predictive expected cross-entropy loss:
\begin{equation} \label{wirls} 
\mathbb{E}_{\bm{\tilde y} | \bm y}\mathcal{L}[\{h({\tilde y}_i)\}_{i=1}^{\tilde n}, \bm{\hat \delta}_\mathcal{S}] \approx \sum_{i=1}^{\tilde n} \omega(\bm{\tilde x}_i)  \frac{\{\hat h_i - \hat \pi_\mathcal{S}(\bm{\tilde x}_i)\}^2}{\hat \pi_\mathcal{S}(\bm{\tilde x}_i)\{1 - \hat \pi_\mathcal{S}(\bm{\tilde x}_i)\}} = 
\sum_{i=1}^{\tilde n} w_i (z_i - \bm{\tilde x}_i' \bm{\hat \delta}_\mathcal{S})^2.
\end{equation}

The weighted least squares approximation in \eqref{wirls} summons BBA for subset search. \cite{Hosmer1989} adopted this strategy for subset selection in classical logistic regression on $y_i\in\{0,1\}$. Here, both the goals and the optimization criterion are distinct: we are interested in curating a collection of near-optimal subsets---rather than selecting a single ``best" subset---and the weighted least squares objective \eqref{wirls} inherits the  fitted probabilities $\hat h_i$ from the Bayesian model $\mathcal{M}$ along with the weights $\omega(\bm{\tilde x}_i)$. 

Ideally, we might apply BBA directly based on the covariates $\{\bm{\tilde x}_i\}$, the pseudo-data $\{z_i\}$, and the weights $\{w_i\}$. However, both $z_i$ and $w_i$ depend on $ \hat \pi_\mathcal{S}(\bm{\tilde x}_i)$ and therefore are subset-specific. As a result, BBA cannot be applied without significant modifications. A suitable alternative is to construct subset-invariant psuedo-data and weights by replacing   $\hat \pi_\mathcal{S}(\bm{\tilde x}_i)$ with the corresponding estimate from the full model, $\hat h_i$. Specifically, let 
\begin{equation}\label{modify}
\hat z_i \coloneqq \log\{\hat h_i/(1-\hat h_i)\} , \quad \hat w_i \coloneqq \omega(\bm{\tilde x}_i) \hat h_i ( 1 - \hat h_i),
\end{equation}
both of which depend on $\mathcal{M}$ rather than the individual subsets $\mathcal{S}$. The pseudo-data $\hat z_i$ is defined similarly to $z_i$ in \eqref{data-0}, where the second term now cancels. Finally, BBA subset search can be applied using the covariates $\{\bm{\tilde x}_i\}$, the pseudo-data $\{\hat z_i\}$, and the weights $\{\hat w_i\}$. As for squared error, we restrict each subset size to $m_k = 100$ or $m_k = 15$ for all subsets of size $k=1,\ldots,p$. Despite the critical role of the weighted least squares approximation in \eqref{wirls} for subset search, all further evaluations and comparisons rely on the exact cross-entropy loss \eqref{cross-ent}.

\subsection{Predictive evaluations for identifying near-optimal subsets}\label{sec-accept}
The BBA subset search filters the $2^p$ possible subsets to the $m_k$ best subsets for each size $k=1,\ldots,s_{max}$ according to posterior predictive expected loss using the weighted squared error loss for prediction (Section~\ref{sec-linear}) or the cross-entropy loss for classification (Section~\ref{sec-logit}). However, as noted below Lemma~\ref{rss-ord-lem}, further comparisons based on these expected losses are trivial: the full set $\mathcal{S} = \{1,\ldots,p\}$ always obtains the minimum, which precludes variable selection. Selection of a single ``best" subset of each size $k$ invites additional difficulties: if multiple subsets perform similarly---which is common for correlated covariates---then selecting $m_k = 1$ subset will not be robust or stable against perturbations of the data. Equally important, restricting to $m_k = 1$ subset is blind to competing subsets that offer similar predictive performance yet may differ substantially in the composition of covariates; see Sections~\ref{sims}-\ref{app}. These challenges persist for both classical and Bayesian approaches.

We instead curate and explore a collection of near-optimal subsets.  The notion of ``near-optimal" derives from the \emph{acceptable family} of \cite{Kowal2020target}. Informally, the acceptable family uses out-of-sample predictive metrics to gather those predictors that match or nearly match the performance of the best out-of-sample predictor with nonnegligible posterior predictive probability under $\mathcal{M}$. The out-of-sample evaluation uses a modified $K$-fold cross-validation procedure. Let $\mathcal{I}_k \subset \{1,\ldots,n\}$ denote the $k$th validation set, where each data point appears in (at least) one validation set, $\cup_{k=1}^K \mathcal{I}_k = \{1,\ldots,n\}$. By default, we use $K=10$ validation sets that are equally-sized, mutually exclusive, and selected randomly from $\{1,\ldots,n\}$. Define an evaluative loss function $L(y, \bm{\tilde x}' \bm{\hat \delta}_{\mathcal{S}})$ for the optimal linear coefficients of subset $\mathcal{S}$, and let $\mathbb{S}$ denote the collection of subsets obtained from the BBA filtering process.  Typically, the evaluative loss $L$ is identical to the loss $\mathcal{L}$ used for optimization, but this restriction is not required. For each data split $k$ and each subset $\mathcal{S} \in \mathbb{S}$, the out-of-sample \emph{empirical} and \emph{predictive} losses are
\begin{equation}
\label{out-loss-k}
{L}_{\mathcal{S}}^{out}(k) \coloneqq  \frac{1}{|\mathcal{I}_k|} \sum_{i \in \mathcal{I}_k} L(y_i, \bm x_i' \bm{\hat \delta}_{\mathcal{S}}^{-\mathcal{I}_k}), \quad 
\widetilde{{L}}_{\mathcal{S}}^{out}(k)  \coloneqq  \frac{1}{|\mathcal{I}_k|} \sum_{i \in \mathcal{I}_k} L(\tilde y_i^{-\mathcal{I}_k}, \bm x_i'\bm{\hat \delta}_{\mathcal{S}}^{-\mathcal{I}_k})
\end{equation}
respectively, where $\bm{\hat \delta}_{\mathcal{S}}^{-\mathcal{I}_k} \coloneqq \arg\min_{\bm\delta_\mathcal{S}} \mathbb{E}_{\bm{\tilde y} | \bm y^{-\mathcal{I}_k}} \mathcal{L}(\{{\tilde y}_i\}_{i \not\in \mathcal{I}_k}, \bm \delta_\mathcal{S}) $ is estimated using only the training data  $\bm y^{-\mathcal{I}_k} \coloneqq \{ y_i\}_{i \not \in \mathcal{I}_k}$ and $\tilde y_i^{-\mathcal{I}_k} \sim p_{\mathcal{M}}({\tilde y}_i | \bm y^{-\mathcal{I}_k})$ is the posterior predictive distribution at $\bm x_i$ conditional only on the training data. Averaging across all data splits, we obtain ${L}_{\mathcal{S}}^{out} \coloneqq K^{-1} \sum_{k=1}^K {L}_{\mathcal{S}}^{out}(k)$ and $ \widetilde{{L}}_\mathcal{S}^{out}  \coloneqq K^{-1} \sum_{k=1}^K\widetilde{{L}}_\mathcal{S}^{out}(k).$

The distinction between the empirical loss ${L}_{\mathcal{S}}^{out}$ and the predictive loss $ \widetilde{{L}}_\mathcal{S}^{out}$ is important. The empirical loss is a point estimate of the risk under predictions from $\bm{\hat \delta}_\mathcal{S}$ based on the data $\bm y$. By comparison, the predictive loss provides a distribution of the out-of-sample loss based on the model $\mathcal{M}$. Both are valuable: ${L}_{\mathcal{S}}^{out}$ is entirely empirical and captures the classical notion of $K$-fold cross-validation, while $ \widetilde{{L}}_\mathcal{S}^{out}$ leverages the Bayesian model to propagate the uncertainty from the model-based data-generating process.

For any two subsets $\mathcal{S}_1$ and $\mathcal{S}_2$, consider the   percent increase in out-of-sample 
 predictive loss from $\bm{\hat \delta}_{\mathcal{S}_1}$ to $\bm{\hat \delta}_{\mathcal{S}_2}$:
\begin{equation}\label{Dtilde}
\widetilde{{D}}_{\mathcal{S}_1,\mathcal{S}_2}^{out} \coloneqq 100\times(\widetilde{{L}}_{\mathcal{S}_2}^{out} - \widetilde{{L}}_{\mathcal{S}_1}^{out})/\widetilde{{L}}_{\mathcal{S}_1}^{out}.
\end{equation}
 Since $\widetilde{{D}}_{\mathcal{S}_1,\mathcal{S}_2}^{out} $  inherits a predictive distribution under $\mathcal{M}$, we can leverage the accompanying uncertainty quantification to determine whether the predictive performances of $\bm{\hat \delta}_{\mathcal{S}_1}$ and $\bm{\hat \delta}_{\mathcal{S}_2}$ are sufficiently distinguishable. In particular, we are interested in comparisons to the best subset for out-of-sample prediction, 
 \begin{equation}\label{best}
\mathcal{S}_{min} \coloneqq \arg\min_{\mathcal{S} \in \mathbb{S}}  {L}_{\mathcal{S}}^{out},
\end{equation}
so that $\bm{\hat \delta}_{\mathcal{S}_{min}}$ is the optimal linear action associated with the subset $\mathcal{S}_{min}$ that minimizes the empirical $K$-fold cross-validated loss.  Unlike the RSS-based in-sample evaluations from Section~\ref{sec-linear}, the subset $\mathcal{S}_{min}$ can---and usually will---differ from the full set $\{1,\ldots,p\}$, which enables variable selection driven by out-of-sample predictive performance.

 Using  $ \mathcal{S}_{min} $ as an anchor, the acceptable family broadens to include near-optimal subsets:
 \begin{equation}\label{accept}
\mathbb{A}_{\eta, \varepsilon} \coloneqq \big\{ \mathcal{S} \in \mathbb{S}: \mathbb{P}_\mathcal{M}\big(\widetilde{{D}}_{\mathcal{S}_{min},\mathcal{S}}^{out}  < \eta \big) \ge \varepsilon \big\}, \quad \eta \ge 0, \varepsilon \in [0,1].
\end{equation}
 With \eqref{accept}, we collect all subsets $\mathcal{S}$ that perform within a {margin} $\eta \ge 0\%$ of the best subset, $\widetilde{{D}}_{\mathcal{S}_{min},\mathcal{S}}^{out}  < \eta$, with {probability} at least $\varepsilon \in [0,1]$. Equivalently, a subset $\mathcal{S}$ is acceptable   if and only if there exists a lower $(1- \varepsilon)$ posterior prediction interval for $\widetilde{{D}}_{\mathcal{S}_{min},\mathcal{S}}^{out}$ that includes $\eta$. Hence, \emph{unacceptable} subsets are those for which there is insufficient predictive probability (under $\mathcal{M}$) that the out-of-sample  accuracy of $\mathcal{S}$ is within a predetermined margin of the best subset. The acceptable family is nonempty, since $\mathcal{S}_{min} \in \mathbb{A}_{\eta, \varepsilon}$, and is expanded by increasing $\eta$ or decreasing $\varepsilon$. By default, we select $\eta = 0\%$ and $\varepsilon = 0.10$ and assess robustness in the simulation study (see also \citealp{Kowal2020target,KowalPRIME2020}).

Within the acceptable family, we isolate two subsets of particular interest: the best subset $\mathcal{S}_{min}$ from \eqref{best} and the smallest acceptable subset,
\begin{equation}\label{min-accept}
\mathcal{S}_{small} \coloneqq \arg\min_{\mathcal{S} \in \mathbb{A}_{\eta, \varepsilon}} |\mathcal{S}|.
\end{equation}
When $\mathcal{S}_{small}$ is nonunique, so that multiple subsets of the same minimal size are acceptable, we select from those subsets the one with the smallest empirical loss ${L}_{\mathcal{S}}^{out}$. By definition, $\mathcal{S}_{small}$ is the smallest set of covariates that satisfies the near-optimality condition in \eqref{accept}.
As noted  by \cite{hastie2009elements} and others, selection based on minimum cross-validated error, such as $\mathcal{S}_{min}$, often produces models or subsets that are more complex than needed for adequate predictive accuracy. The acceptable family $\mathbb{A}_{\eta, \varepsilon}$, and in particular $\mathcal{S}_{small}$, exploits this observation to provide alternative---and often much smaller---subsets of variables.

 To ease the computational burden, we adapt the importance sampling algorithm from  \cite{Kowal2020target}  to compute \eqref{accept} and all constituent quantities (see Appendix~\ref{sir-approx}). Crucially, this algorithm is based entirely on the in-sample posterior distribution under model $\mathcal{M}$, which avoids both (i) the intensive process of re-fitting $\mathcal{M}$ for each data split $k$ and (ii) data reuse that adversely affects downstream uncertainty quantification. Briefly, the algorithm uses the complete data posterior $p_\mathcal{M}(\bm \theta | \bm y)$ as a proposal for the training data posterior $p_\mathcal{M}(\bm \theta | \bm y^{-\mathcal{I}_k})$. The importance weights are then computed using the likelihood $p_\mathcal{M}(\bm y^{\mathcal{I}_k} | \bm \theta)$ under model $\mathcal{M}$.  Similar algorithms have been deployed for Bayesian model selection \citep{gelfand1992model}, evaluating prediction distributions \citep{vehtari2012survey}, and $\ell_1$-based Bayesian variable selection \citep{KowalPRIME2020}.

 In the case of new (out-of-sample) covariate values $\{\bm{\tilde x}_i\}_{i=1}^{\tilde n} \ne \{\bm{ x}_i\}_{i=1}^n$, the predictive loss may be defined without cross-validation: $ \widetilde{{L}}_\mathcal{S}^{out} \coloneqq  {\tilde n}^{-1} \sum_{i=1}^{\tilde n} L(\tilde y_i, \bm{\tilde x}_i'\bm{\hat \delta}_{\mathcal{S}})$, where ${\tilde y}_i \coloneqq \bm{\tilde y}(\bm{\tilde x}_i)$ is the predictive variable at $\bm{\tilde x}_i$ for each $i=1,\ldots, \tilde n$. The empirical loss is undefined without a corresponding observation of the response variable for each $\bm{\tilde x}_i$. Hence, we modify the acceptable family \eqref{accept} to instead reference the full subset $\mathcal{S} = \{1,\ldots,p\}$ in place of $\mathcal{S}_{min}$, which is no longer defined. When $\mathcal{M}$ is a linear model,  Corollary~\ref{cor-linear} implies that the corresponding reference is simply the posterior predictive expectation under $\mathcal{M}$.

\subsection{Co-variable importance}
 Although it is common to focus on a single subset for selection,  the  acceptable family $\mathbb{A}_{\eta, \varepsilon}$ provides a broad assortment of competing explanations: linear actions with distinct sets of covariates that all provide near-optimal predictive accuracy. Hence, we seek to further summarize $\mathbb{A}_{\eta, \varepsilon}$ beyond $\mathcal{S}_{min}$ and $\mathcal{S}_{small}$ to identify (i) which covariates appear in {any} acceptable subset, (ii) which covariates appear in {all} acceptable subsets, and  (iii) which covariates appear together in the same acceptable subsets. 

For covariates $j$ and $\ell$, the sample proportion of joint inclusion in $\mathbb{A}_{\eta, \varepsilon}$ achieves each of these goals: 
\begin{equation}\label{vi-co}
\mbox{VI}_{\rm incl}(j, \ell)  \coloneqq  |\mathbb{A}_{\eta, \varepsilon}|^{-1} \sum_{\mathcal{S} \in \mathbb{A}_{\eta, \varepsilon}} \mathbb{I}\{j,\ell \in \mathcal{S}\}
\end{equation}
and measures \emph{(co-) variable importance.}  Naturally, \eqref{vi-co} is generalizable to more than two covariates, but is particularly interesting for a single covariate: $\mbox{VI}_{\rm incl}(j) \coloneqq \mbox{VI}_{\rm incl}(j, j)$ is the proportion of acceptable subsets to which covariate $j$ belongs. When $\mbox{VI}_{\rm incl}(j) > 0$, covariate $j$ belongs to at least one acceptable subset. Such a result does not imply that covariate $j$ is necessary for accurate prediction, but rather that covariate $j$ is a component of at least one near-optimal linear subset. When $\mbox{VI}_{\rm incl}(j) = 1$, we refer to covariate $j$ as a \emph{keystone covariate}: it belongs to all acceptable subsets and therefore is deemed essential.  For $j\ne \ell$, $\mbox{VI}_{\rm incl}(j, \ell)$ highlights not only the covariates that co-occur, but also the covariates that rarely appear together. It is particularly informative to identify covariates $j$ and $\ell$ such that $\mbox{VI}_{\rm incl}(j)$ and $\mbox{VI}_{\rm incl}(\ell)$ are both large yet $\mbox{VI}_{\rm incl}(j, \ell)$ is small. In that case, covariates $j$ and $\ell$ are both important yet also redundant, as might be expected for highly correlated variables.

Although $\mbox{VI}_{\rm incl}$ is defined based on linear predictors for each subset $\mathcal{S}$, the variable importance metric applies broadly to (possibly nonlinear) Bayesian models $\mathcal{M}$ and a variety of evaluative loss functions $L$, and can be targeted locally via the weights $\omega(\bm{\tilde x}_i)$. The inclusion-based metric $\mbox{VI}_{\rm incl}$ can be extended to incorporate effect size, which is a more common strategy for variable importance. Related, \cite{Dong2019} aggregated model-specific variable importances across many ``good models".  Alternative approaches use leave-one-covariate-out predictive metrics (e.g., \citealp{Lei2018}), but are less appealing in the presence of correlated covariates.

\subsection{Posterior predictive uncertainty quantification}\label{sec-uq}
A persistent challenge in classical subset selection is the lack of accompanying uncertainty quantification. Given a subset $\mathcal{\hat S}$ selected using the data  $(\bm X, \bm y)$, familiar frequentist and Bayesian inferential procedures applied to $(\bm X_{\mathcal{\hat S}}, \bm y)$ are in general no longer valid. In particular, we cannot simply proceed as if only the selected covariates $\mathcal{\hat S}$ were supplied from the onset. Such analyses are subject to selection bias \citep{Miller1984}.  

A crucial feature of our subset filtering (Section~\ref{sec-linear}~and~Section~\ref{sec-logit}) and predictive evaluation (Section~\ref{sec-accept}) techniques is that, despite the broad searching and the out-of-sample targets, these quantities all remain \emph{in-sample posterior functionals} under $\mathcal{M}$. There is no data re-use or model re-fitting: every requisite term is a functional of the complete data posterior $p_\mathcal{M}(\bm \theta | \bm y)$ or $p_\mathcal{M}(\bm{\tilde y} | \bm y)$ from a single Bayesian model. Hence, the posterior distribution under $\mathcal{M}$ remains a valid facilitator of uncertainty quantification. 

We elicit a posterior predictive distribution for the action by removing the expectation in \eqref{action}:
\begin{equation}\label{p-action}
\bm{\tilde \delta} \coloneqq \arg\min_{\bm\delta} \mathcal{L}(\bm{\tilde y}, \bm \delta)
\end{equation}
which no longer integrates over $p_\mathcal{M}(\bm{\tilde y} | \bm y)$ and hence propagates the posterior predictive uncertainty. For the squared error loss \eqref{sq-loss}, the \emph{predictive action} is 
\begin{equation}\label{sq-loss-p-action}
\bm{\tilde \delta}_\mathcal{S} = ( \bm{\tilde X}_\mathcal{S}'\bm \Omega \bm{\tilde X}_\mathcal{S})^{-1}\bm{\tilde X}_\mathcal{S}'\bm\Omega\bm{\tilde y}
\end{equation}
akin to \eqref{sq-loss-action-2}, where $\bm{\tilde y} = ({\tilde y}_1,\ldots, {\tilde y}_{\tilde n})' \sim p_\mathcal{M}(\bm{\tilde y} | \bm y)$. The linear coefficients $\bm{\tilde \delta}_\mathcal{S}$ inherit a posterior predictive distribution from $\bm{\tilde y}$ and can be computed for any subset $\mathcal{S}$. Similar computations are available for the cross-entropy loss \eqref{cross-ent}. In both cases, draws $\{\bm{\tilde y}^s\}_{s=1}^S \sim p_\mathcal{M}(\bm{\tilde y} | \bm y)$ from the posterior predictive distribution are sufficient for uncertainty quantification of $\bm{\tilde \delta}_\mathcal{S}$. Under the usual assumption that $p_\mathcal{M}(\bm{\tilde y} | \bm \theta, \bm y ) = p_\mathcal{M}(\bm{\tilde y} | \bm \theta)$, these draws are easily obtained by repeatedly sampling $\bm\theta^s \sim p_\mathcal{M}(\bm \theta | \bm y)$  from the posterior and $\bm{\tilde y}^s \sim p_\mathcal{M}(\bm{\tilde y} |  \bm \theta = \bm \theta^s ) $ from the likelihood.

The predictive actions are computable for any subset $\mathcal{S}$, include those selected based on the predictive evaluations in Section~\ref{sec-accept}. Let $\mathcal{\hat S}$ denote a subset identified based on the posterior (predictive) distribution under $\mathcal{M}$, such as $\mathcal{S}_{min}$ or $\mathcal{S}_{small}$. The predictive action $\bm{\tilde \delta}_\mathcal{\hat S}$ using this subset in \eqref{sq-loss-p-action} is a posterior predictive functional. Unlike for a generic subset $\mathcal{S}$, the predictive action $\bm{\tilde \delta}_\mathcal{\hat S}$ is a functional of both  $\bm{\tilde y}$ and $\mathcal{\hat S}$, which factors into the interpretation. 

A similar projection-based approach is developed by \cite{woody2019model} for  Gaussian regression. In place of $ \mathcal{L}(\bm{\tilde y}, \bm \delta)$ in \eqref{p-action}, \cite{woody2019model} suggest  $ \mathbb{E}_{\bm{\tilde y} | \bm \theta} \mathcal{L}(\bm{\tilde y}, \bm \delta)$, which is a middle ground between \eqref{action} and \eqref{p-action}: it integrates over the uncertainty of $\bm{\tilde y}$ given model parameters $\bm \theta$, but preserves uncertainty due to $\bm \theta$. 
For example, under the linear regression model $\mathbb{E}_{\bm{\tilde y} | \bm \theta} \{\bm{\tilde y}(\bm{\tilde x})\} = \bm{\tilde x}'\bm\beta$, the analogous result in \eqref{sq-loss-p-action} is $( \bm{\tilde X}_\mathcal{S}'\bm \Omega \bm{\tilde X}_\mathcal{S})^{-1}\bm{\tilde X}_\mathcal{S}'\bm\Omega \bm{\tilde X} \bm \beta$; when  $\mathcal{S} = \{1,\ldots,p\}$, this simplifies to the regression coefficient $\bm \beta$. Both approaches have their merits, but we prefer \eqref{p-action} because of the connection to the observable random variables $\bm{\tilde y}$ rather than the model-specific parameters $\bm \theta$.


\section{Simulation study}\label{sims}
	
We conduct an extensive simulation study to evaluate competing subset selection techniques for prediction accuracy, uncertainty quantification, and variable selection. Prediction for real-valued data is discussed here; classification for binary data is presented in Appendix~\ref{sims-class}. Although we do evaluate individual subsets from the proposed framework---namely, $\mathcal{S}_{min}$ and $\mathcal{S}_{small}$---we are more broadly interested in the performance of the \emph{acceptable family} of subsets. In particular, the  acceptable family  is designed to collect \emph{many subsets} that offer \emph{near-optimal} prediction; both cardinality and aggregate predictive accuracy are critical. 

The simulation designs feature varying signal-to-noise ratios (SNRs) and dimensions $(n,p)$, including $p \gg n$, with correlated covariates and sparse linear signals. Covariates $x_{i,j}$ are generated from marginal standard normal distributions with $\mbox{Cor}(x_{i,j}, x_{i,j'}) = (0.75)^{|j - j'|}$ for $i=1,\ldots, n$ and $j=1,\ldots,p$. The $p$ columns are randomly permuted and augmented with an intercept. The true linear coefficients $\bm \beta^*$ are constructed by setting $\beta_0^* = -1$ and fixing $p_* = 5$ nonzero coefficients, with $\lceil p_*/2 \rceil$ equal to $1$ and $\lfloor p_*/2 \rfloor$ equal to $-1$, and the rest at zero. The data are generated independently as $y_i \sim N(y_i^*, \sigma_*^2)$ with $y_i^* \coloneqq \bm x_i'\bm\beta^*$ and $\sigma_* \coloneqq \mbox{sd}(\bm y^*)/\sqrt{\mbox{SNR}}$. We consider $(n, p) \in \{(50, 50), (200, 400), (500, 50)\}$ and $\mbox{SNR} \in \{0.25, 1\}$ for low and high SNR. 
Evaluations are conducted over 100 simulated datasets.

First, we evaluate point prediction accuracy across competing collections of near-optimal subsets. Each collection is built from the same Gaussian linear regression model with horseshoe priors \citep{carvalho2010horseshoe} $\mathcal{M}$  and estimated using \texttt{bayeslm} \citep{Hahn2019}.  The collections are generated from different candidate sets $\mathbb{S}$ based on distinct search methods: the proposed BBA method (\texttt{bbound(bayes)}), the adaptive lasso search (\texttt{lasso(bayes)}) proposed by \cite{hahn2015decoupling} and \cite{KowalPRIME2020}, and both forward (\texttt{forward(bayes)}) and backward (\texttt{backward(bayes)}) search. For each collection of  candidate subsets, we compute the acceptable families for $\eta = 0\%$, $\varepsilon = 0.1$,  $m_k = 15$, and the observed covariate values $\{\bm{\tilde x}_i\}_{i=1}^{\tilde n} = \{\bm{ x}_i\}_{i=1}^n$; variations of these specifications are presented in Appendix~\ref{sims-app}. These methods differ \emph{only} in the search process: if a particular subset $\mathcal{S}$ is identified by more than one of these competing methods, the corresponding point predictions---based on the optimal action  $\bm{\hat \delta}_{\mathcal{S}}$ in \eqref{sq-loss-action-2}---will be identical.

The competing collections of near-optimal subsets are evaluated in aggregate. At each simulation, we compute the $q$th quantile of the root mean squared errors (RMSEs) for the true mean $\{y_i^*\}_{i=1}^n$ across all subsets within that acceptable family, and then average that quantity across all simulations. For example, $q=1$ describes the predictive performance if we were to use the worst acceptable subset at each simulation, as determined by an oracle. These quantities summarize the distribution of RMSEs \emph{within} each acceptable family; we report the values for $q\in\{0, 0.25, 0.5, 0.75, 1\}$ and present the results as boxplots in Figure~\ref{fig:sims-pred}. Most notably, the proposed \texttt{bbound(bayes)} strategy produces up to 10 times the number of acceptable subsets as the other methods, yet crucially maintains equivalent predictive performance. Naturally, this expanded collection of subsets also produces a minimum RMSE ($q=0$) that outperforms the remaining methods. Clearly, the proposed approach is far superior at collecting \emph{more} subsets that provide \emph{near-optimal} predictive accuracy.

\begin{figure}[h!]
\begin{center}
\includegraphics[width=.49\textwidth]{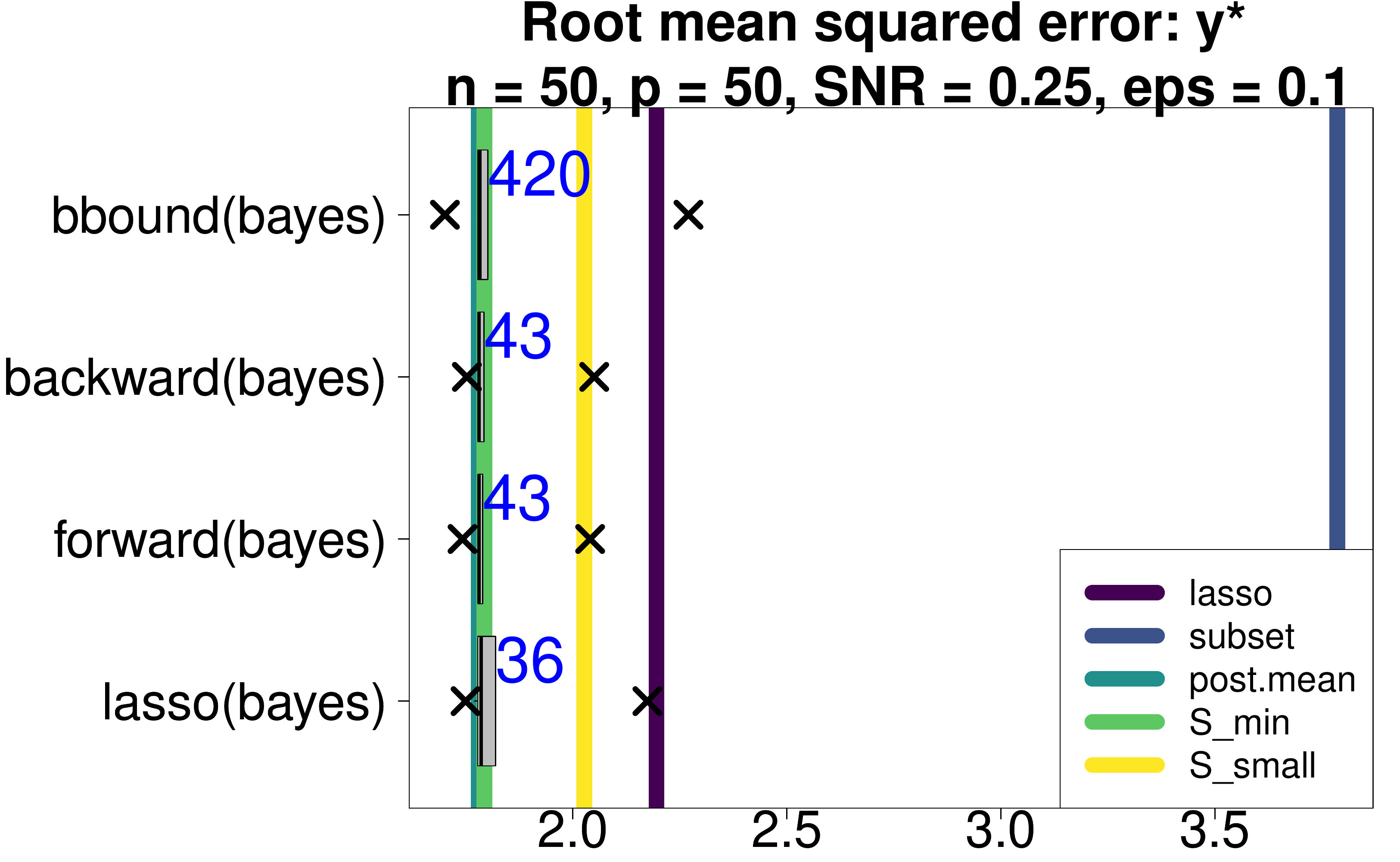}
\includegraphics[width=.49\textwidth]{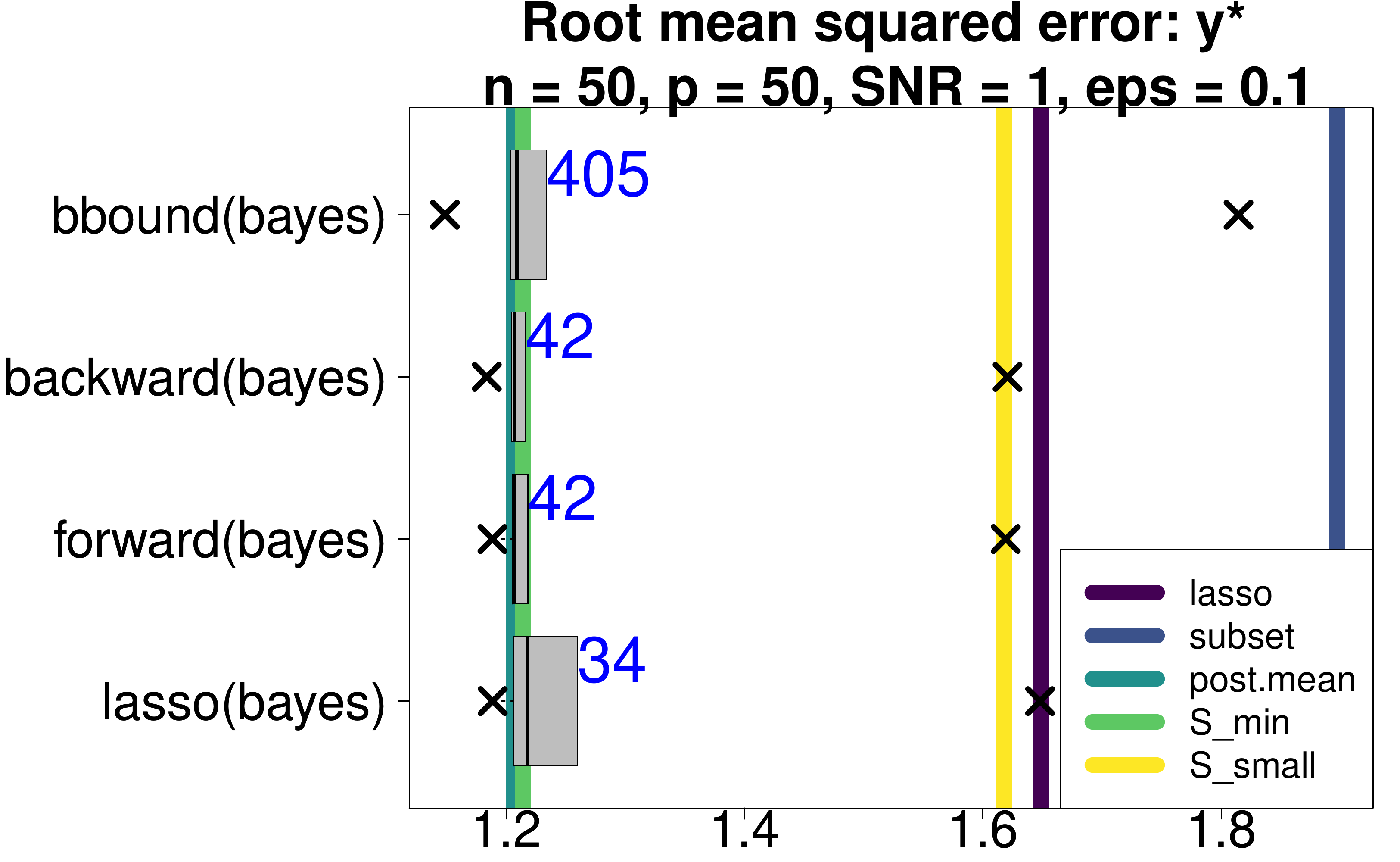}

\vspace{2mm}

\includegraphics[width=.49\textwidth]{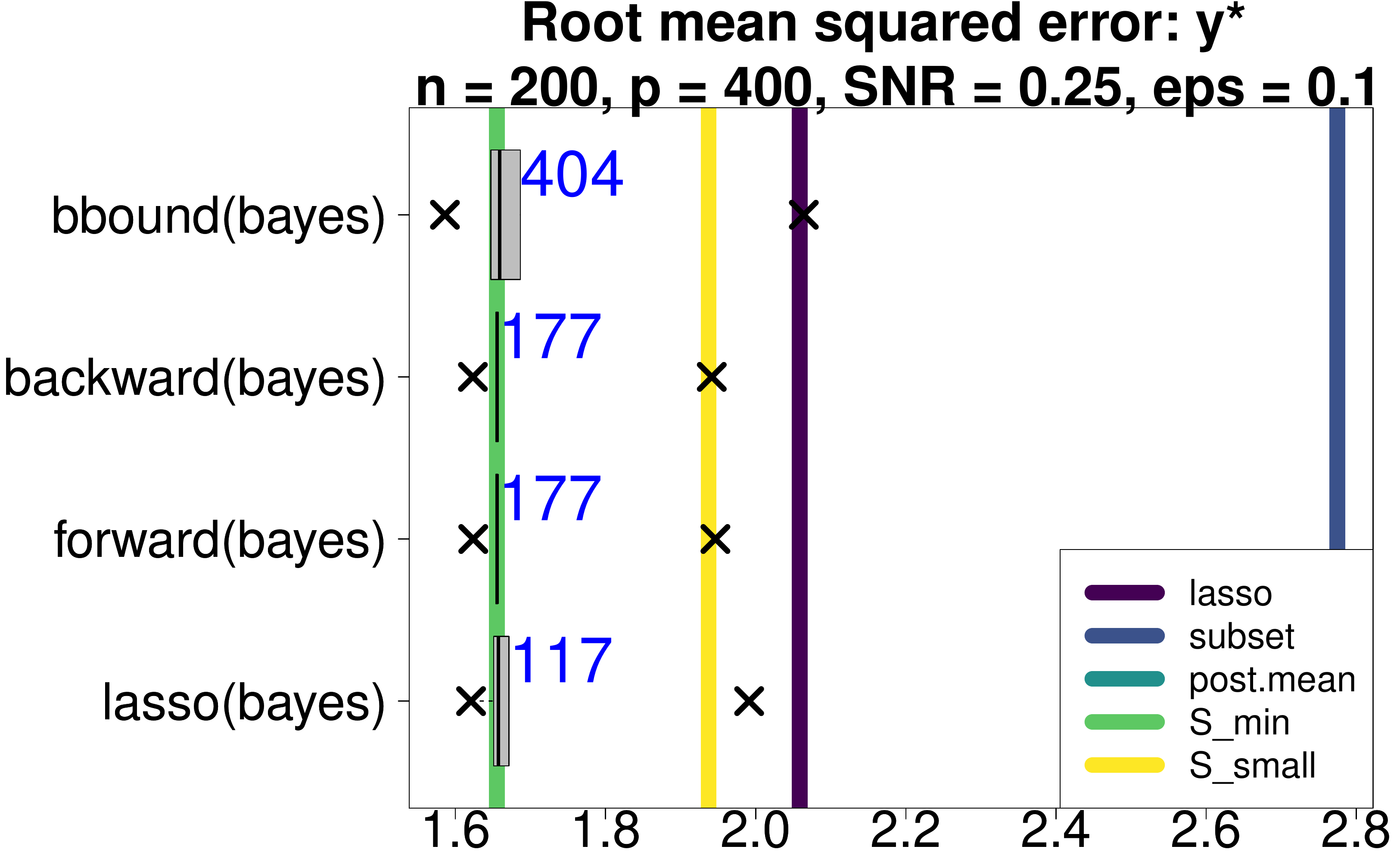}
\includegraphics[width=.49\textwidth]{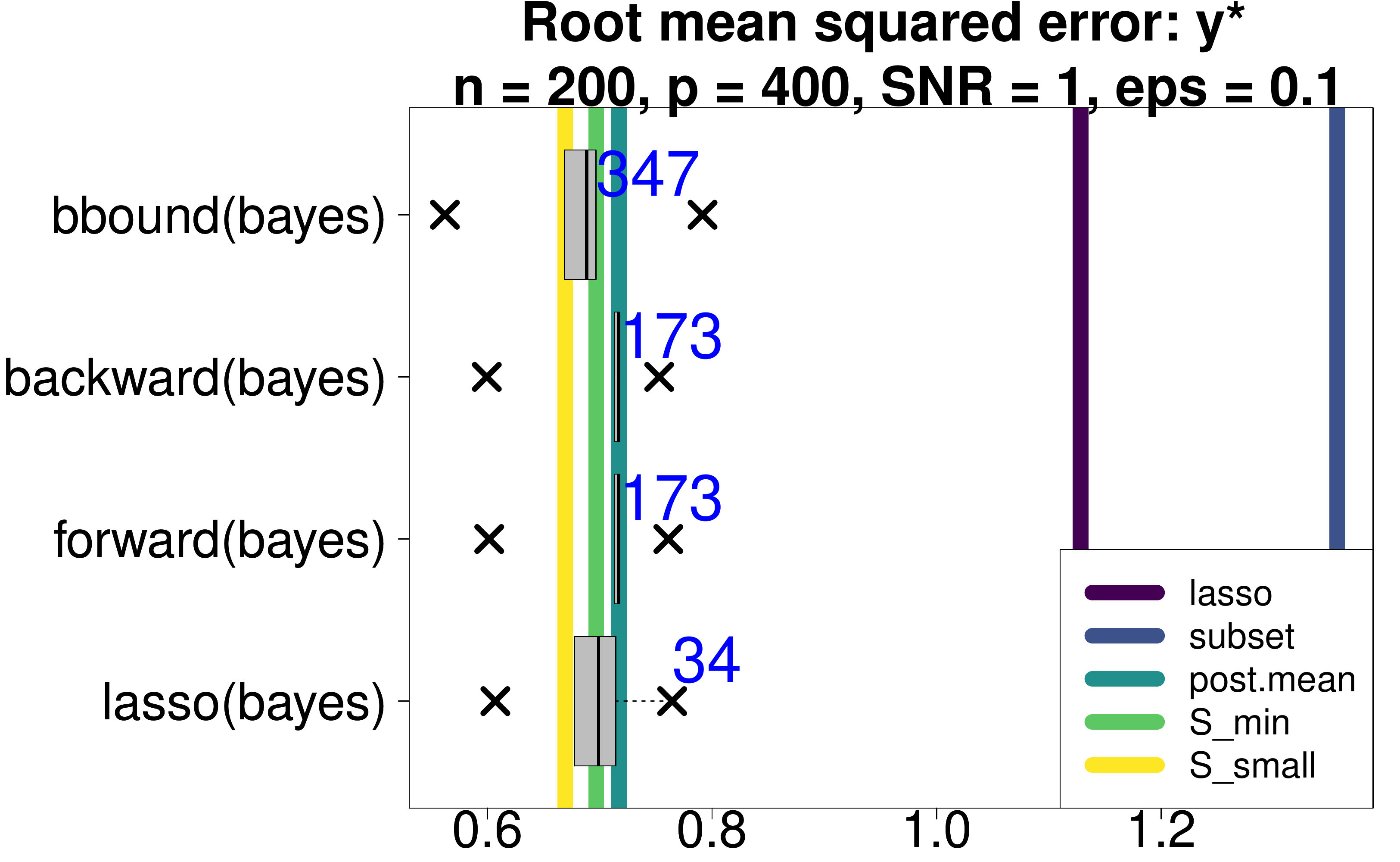}

\vspace{2mm}

\includegraphics[width=.49\textwidth]{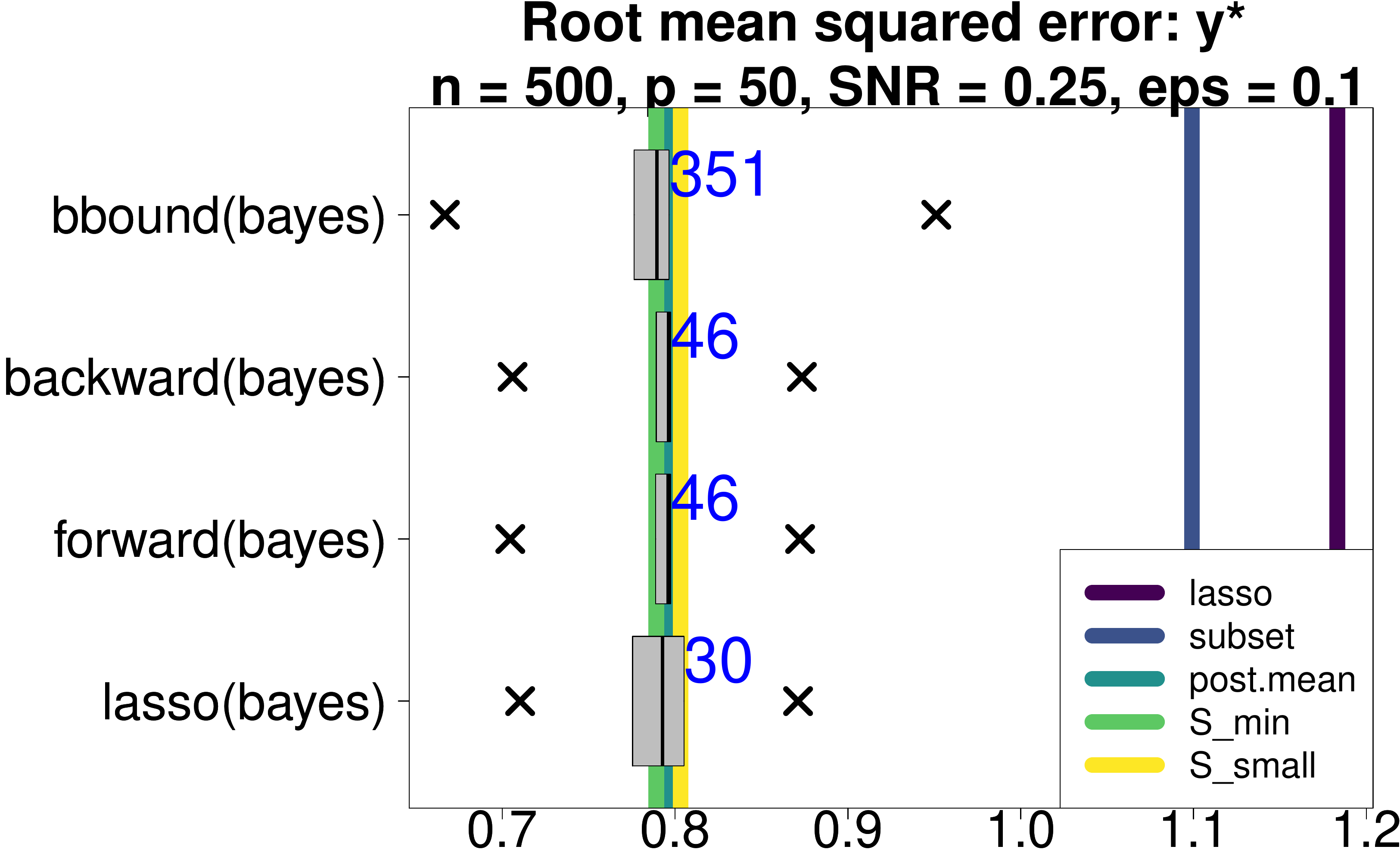}
\includegraphics[width=.49\textwidth]{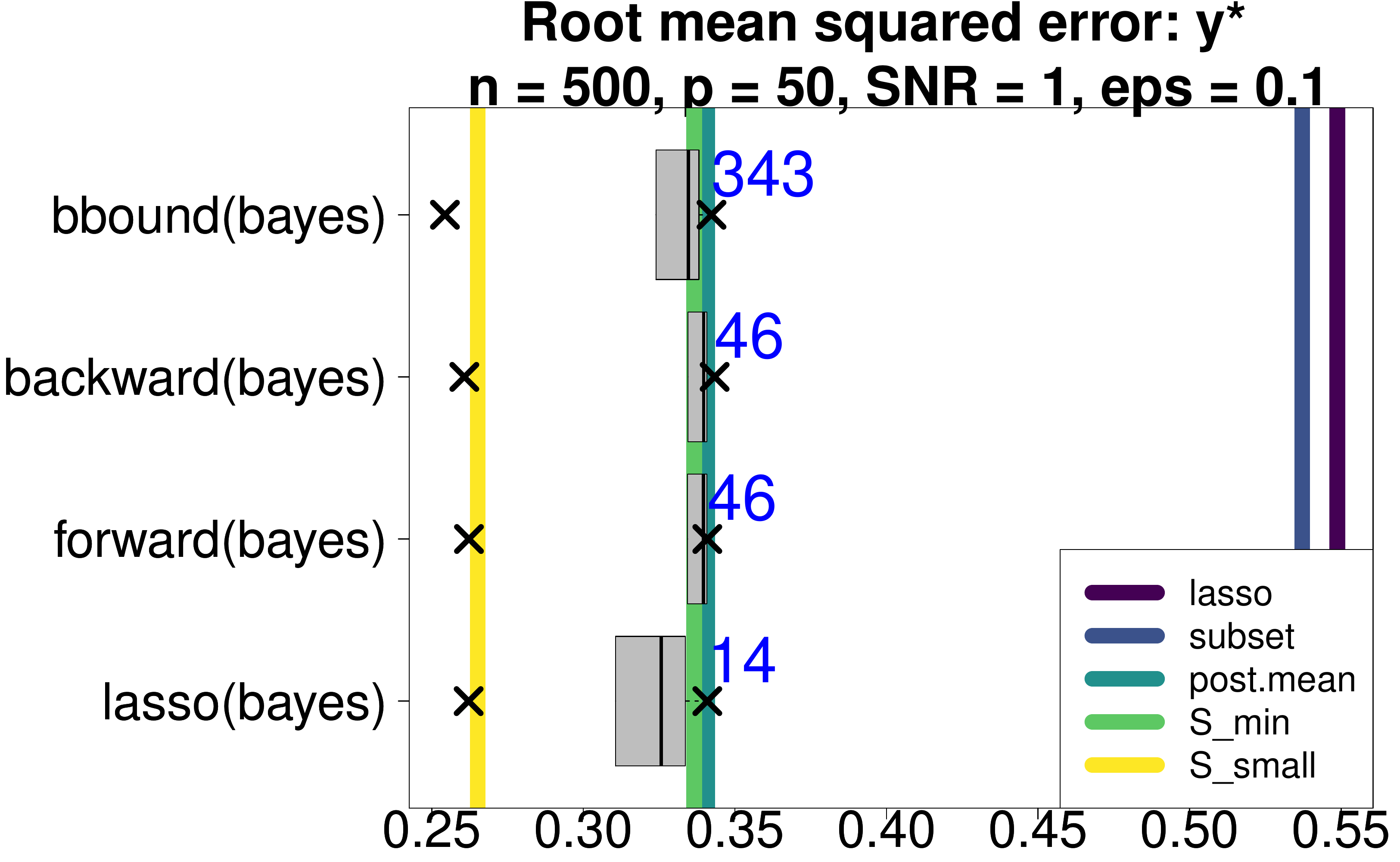}

\end{center}
\caption{ \small
Root mean squared errors (RMSEs) for predicting $y^*$ across $(n, p, \mbox{SNR})$ configurations. The boxplots summarize the RMSE quantiles for the subsets within each acceptable family, while the vertical lines denote RMSEs of competing methods. The average size of  each acceptable family is annotated. The proposed BBA search returns vastly more subsets that remain highly competitive, while $\mathcal{S}_{small}$ performs very well and substantially outperforms classical subset selection.
\label{fig:sims-pred}}
\end{figure}

Figure~\ref{fig:sims-pred} also summarizes the point prediction accuracy for several competing methods. First, we include the usual point estimate under $\mathcal{M}$ given by the posterior expectation of the regression coefficients $\bm \beta$ (\texttt{post.mean}), which does not include any variable or subset selection. Next, we compute point predictions for the key acceptable subsets $\mathcal{S}_{min}$ and $\mathcal{S}_{small}$ using the proposed \texttt{bbound(bayes)} approach; the competing search strategies produced similar results and are omitted.  Among frequentist estimators, we use the adaptive lasso (\texttt{lasso}; \citealp{zou2006adaptive}) with $\lambda$ chosen by 10-fold cross-validation and the one-standard-error rule \citep{hastie2009elements}. In addition, we include classical subset selection (\texttt{subset}) implemented using the \texttt{leaps} package in \texttt{R} with the final subset determined by AIC. When $p > 35$, we screen to the first 35 predictors that enter the model in the aforementioned adaptive lasso. Forward (\texttt{forward}) and backward (\texttt{backward}) selection were also considered, but were either noncompetitive or performed similarly to the adaptive lasso and are omitted from this plot. 

Most notably, the predictive performance of $\mathcal{S}_{small}$ is excellent, especially for high SNRs, and substantially outperforms the frequentist methods in all settings. Because $\mathcal{S}_{min}$ is overly conservative---i.e., it selects too many variables (see Table~\ref{tab:tprs})---it performs nearly identically to \texttt{post.mean}. Although $\mathcal{S}_{min}$ and \texttt{post.mean} outperform $\mathcal{S}_{small}$ when the SNR is low and the sample size is small, note that $\mathcal{S}_{small}$ is not targeted exclusively for optimal predictive performance; rather, it represents the smallest subset that obtains \emph{near-optimal}  performance. Lastly,  the regularization induced by $\mathcal{M}$ is greatly beneficial:  \emph{every} member of each acceptable family decisively outperforms classical subset selection across all scenarios. 

Next, we compare uncertainty quantification for the regression coefficients $\bm \beta^*$. We compute 90\% intervals using the predictive action $\bm{\tilde \delta}_\mathcal{S}$ from \eqref{sq-loss-p-action} for $\mathcal{S}_{small}$ and $\mathcal{S}_{min}$ based on the proposed \texttt{bbound(bayes)} procedure. In addition, we compute 90\% highest posterior density (HPD) credible intervals for $\bm \beta$  under $\mathcal{M}$ (\texttt{post.HPD})  and 90\% frequentist confidence intervals using \cite{Zhao2017}, which computes confidence intervals from a linear regression model that includes only the variables selected by the (adaptive) lasso (\texttt{lasso+lm}). The 90\% interval estimates are evaluated and compared in Figure~\ref{fig:sims-mciw} using interval widths and empirical coverage. The intervals from $\mathcal{S}_{small}$ using \eqref{p-action} are uniformly better (i.e., narrower) than the usual posterior HPD intervals under $\mathcal{M}$. In addition, the intervals from \cite{Zhao2017} are highly competitive, despite ignoring selection bias. In all cases, the intervals are overly conservative and achieve the nominal empirical coverage.

\begin{figure}[h!]
\begin{center}
\includegraphics[width=.32\textwidth]{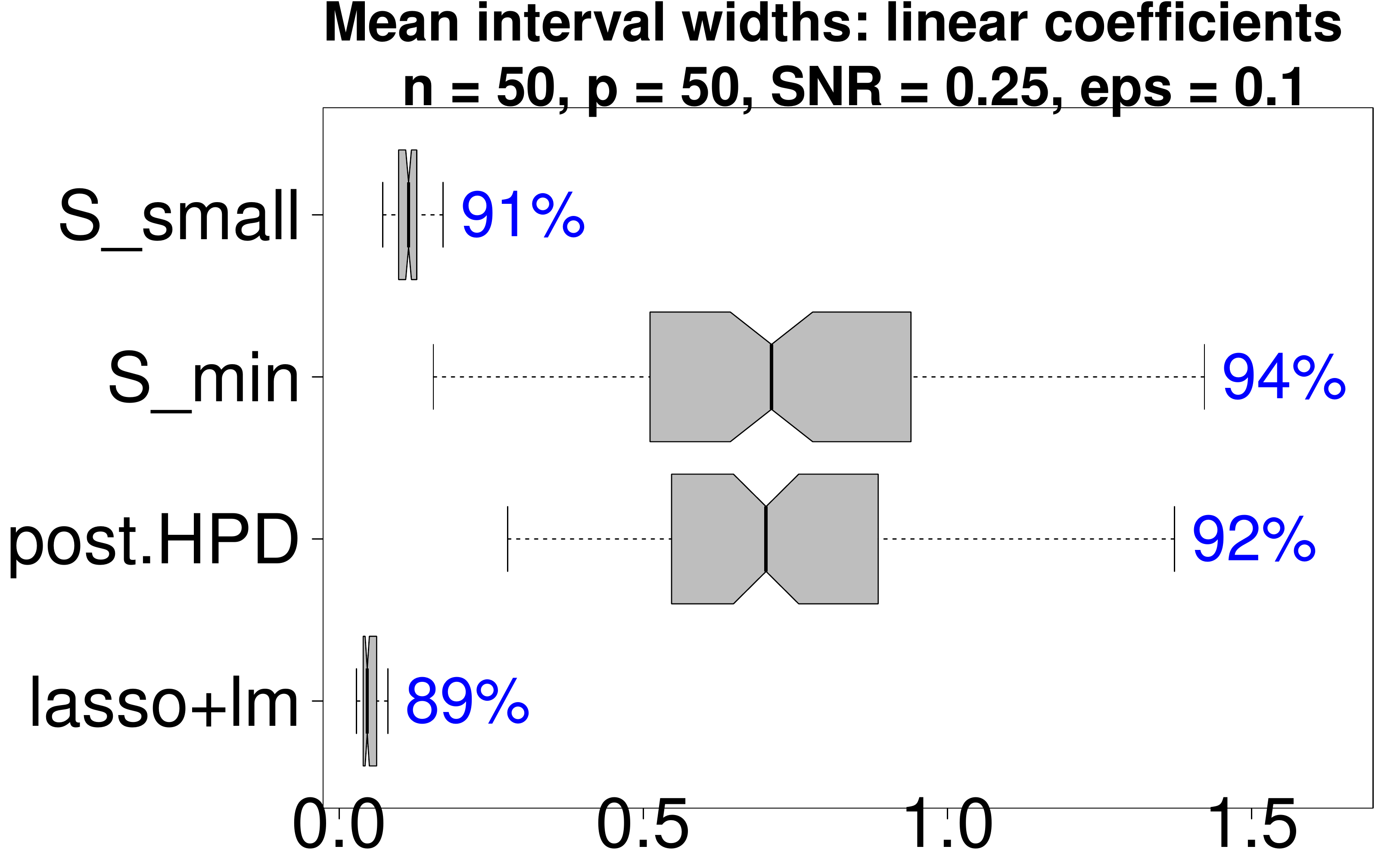}
\includegraphics[width=.32\textwidth]{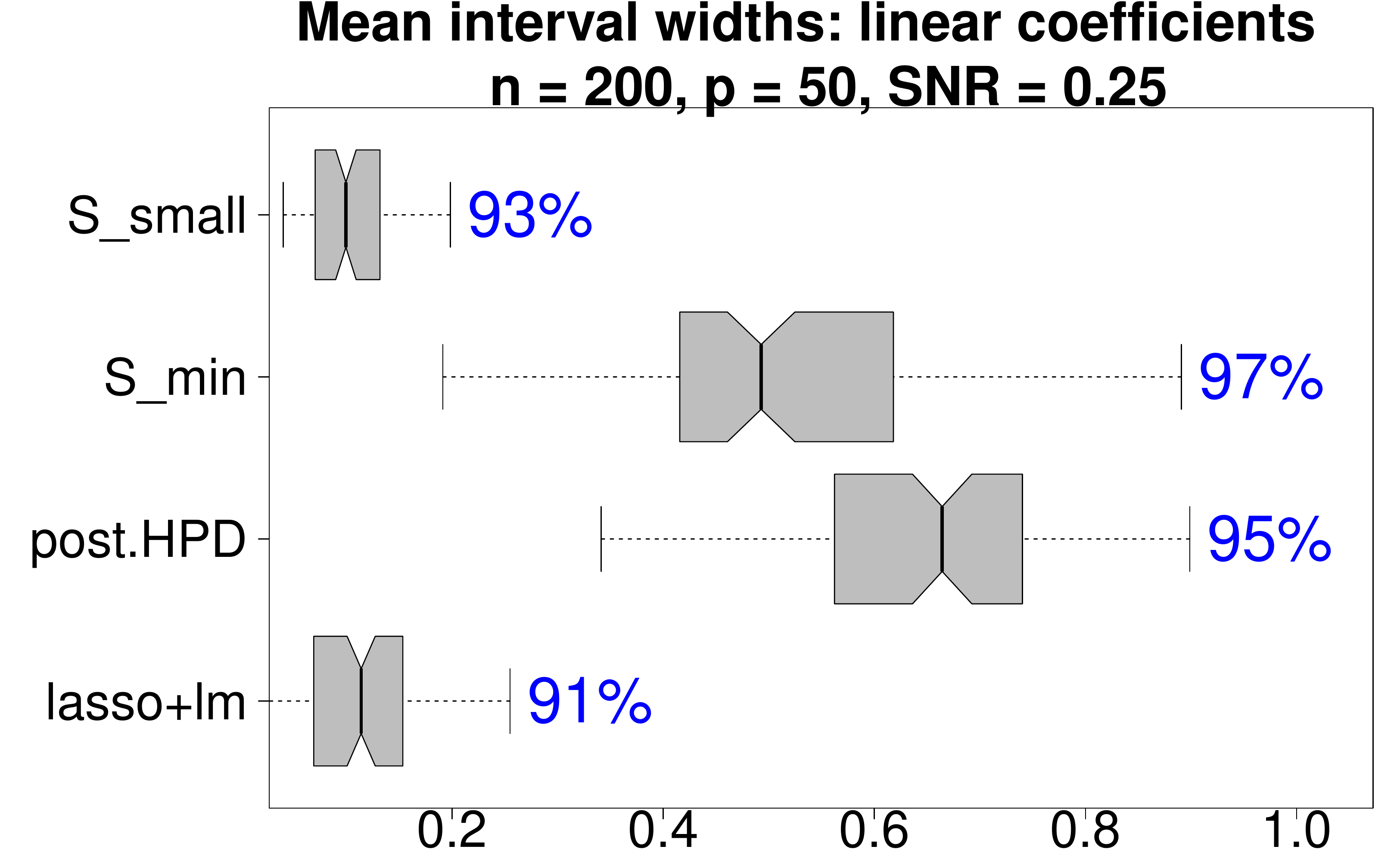}
\includegraphics[width=.32\textwidth]{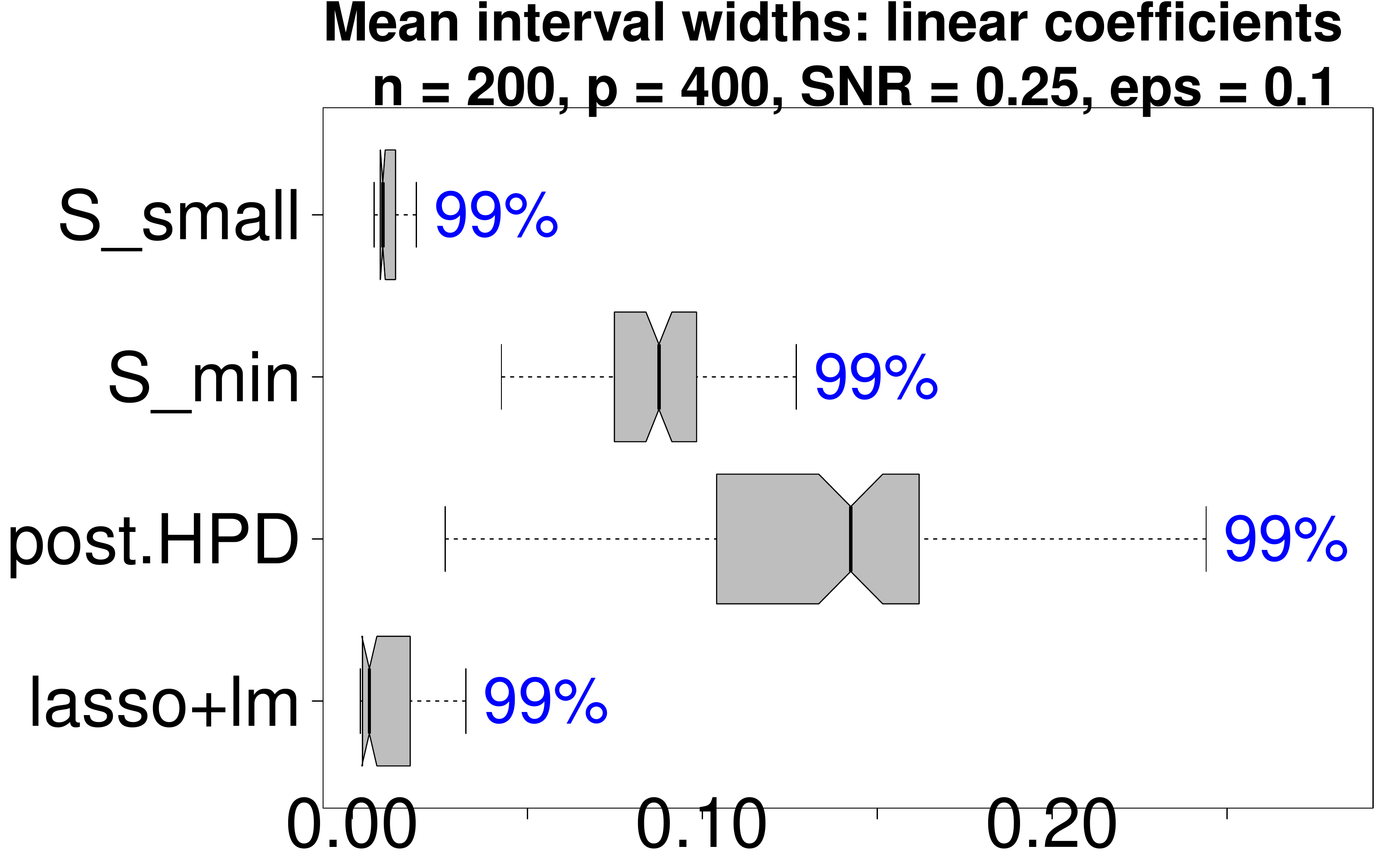}

\includegraphics[width=.32\textwidth]{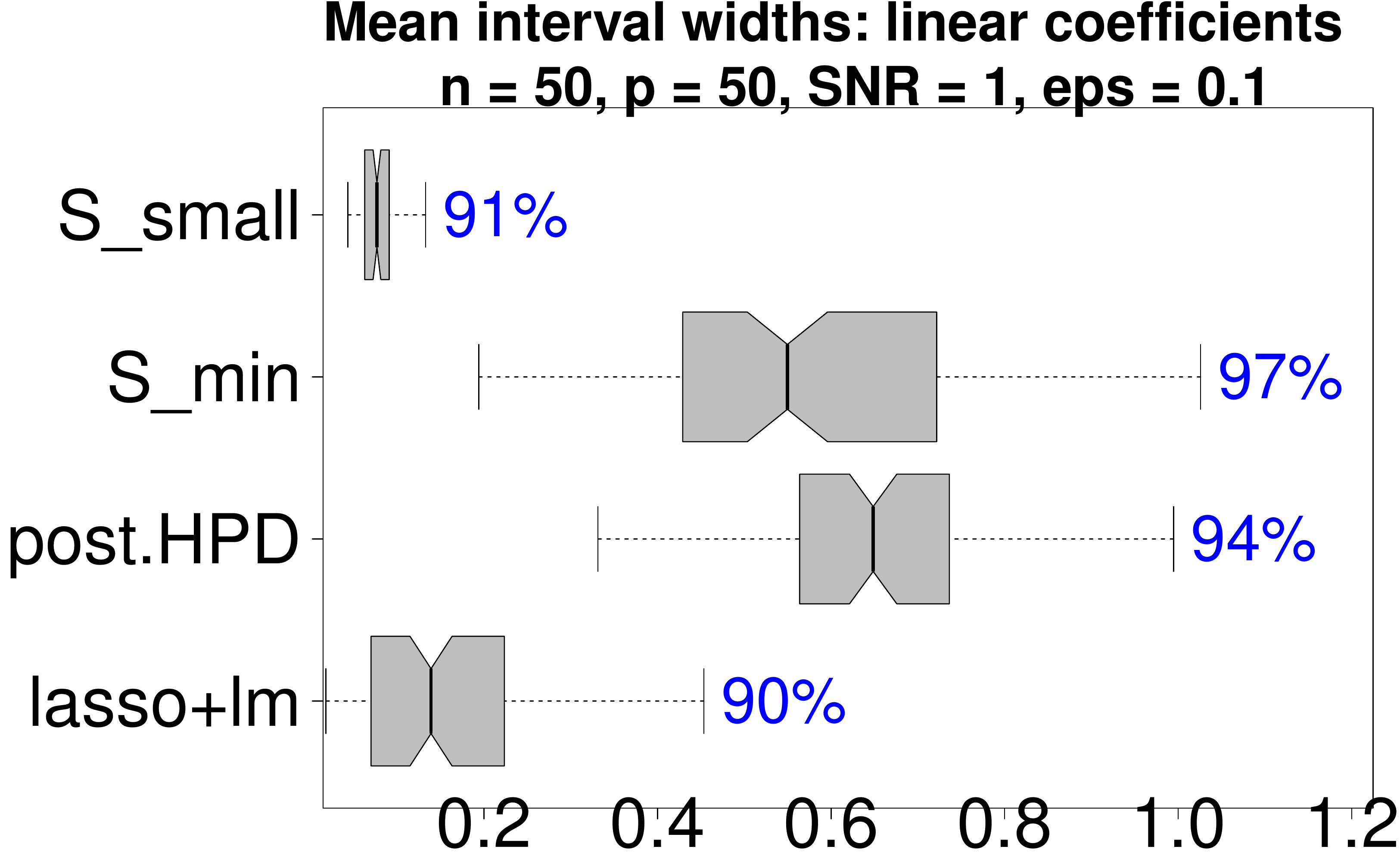}
\includegraphics[width=.32\textwidth]{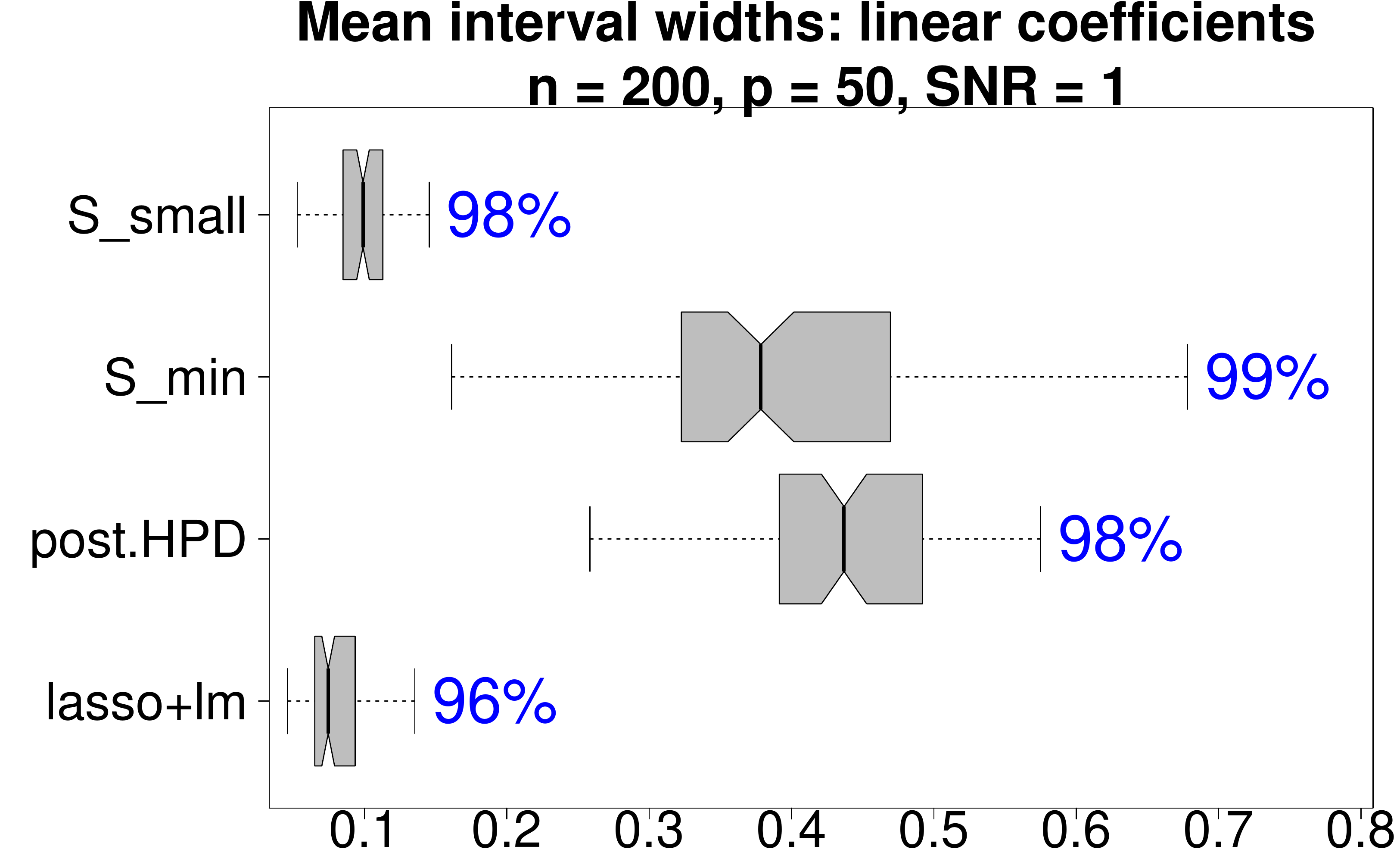}
\includegraphics[width=.32\textwidth]{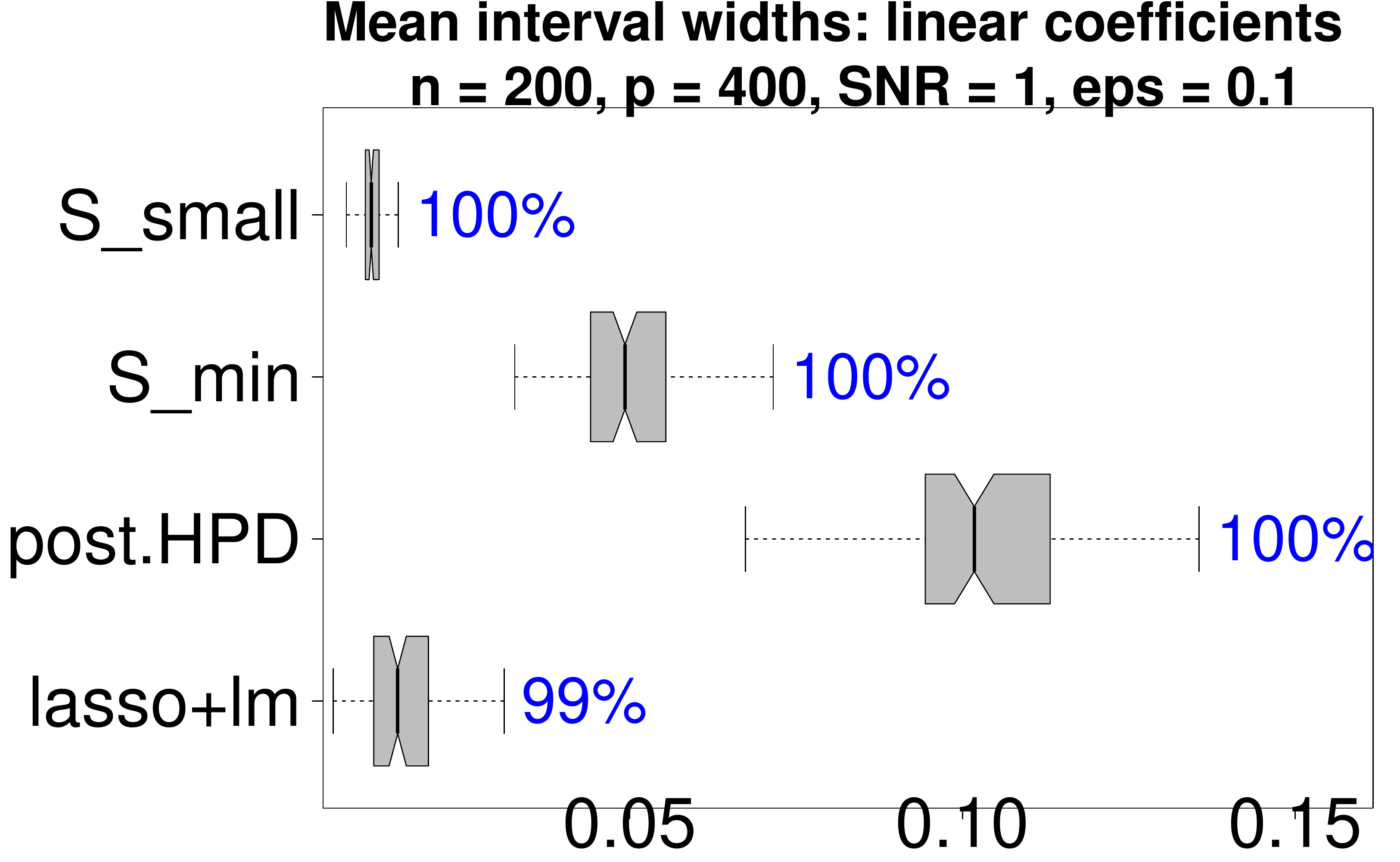}

\end{center}
\caption{ \small
Mean interval widths (boxplots) with empirical coverage (annotations) for $\bm \beta^*$. Non-overlapping notches indicate significant differences between medians. The proposed intervals based on $\mathcal{S}_{small}$ are significantly narrower than the usual HPD intervals under $\mathcal{M}$ yet maintain the empirical nominal 90\% coverage. 
\label{fig:sims-mciw}}
\end{figure}

Lastly, marginal variable selection is evaluated using true positive rates and true negative rates in Table~\ref{tab:tprs}. In addition to $\mathcal{S}_{min}$ and $\mathcal{S}_{small}$, we include the common Bayesian selection technique that includes a variable $j$ when the 95\% HPD interval for $\beta_j$ excludes zero (\texttt{posterior HPD}). 
The selection performance of $S_{small}$ is excellent and similar to the the adaptive lasso. Classical subset, forward, and backward selection are too aggressive, while the posterior interval-based selection is too conservative. Hence, despite the primary focus on the \emph{collection} of near-optimal subsets, the smallest acceptable subset $\mathcal{S}_{small}$ is a key member with excellent prediction, uncertainty quantification, and marginal selection performance.

\begin{table}[ht]
\centering
\caption*{$n=50, p = 50, \mbox{SNR} = 0.25$}\vspace{-3mm}
\begin{tabular}{rrrrrrrr}
  \hline
 & lasso & forward & backward & subset & posterior HPD & $S_{min}$  & $S_{small}$ \\ 
  \hline
TPR & 0.22 & 0.96 & 0.93 & 0.53 & 0.06 & 0.51 & 0.22 \\ 
  TNR & 0.98 & 0.06 & 0.10 & 0.63 & 1.00 & 0.81 & 0.98 \\ 
   \hline
\end{tabular}

\vspace{4mm} \caption*{$n=50, p = 50, \mbox{SNR} = 1$}\vspace{-1mm}
\begin{tabular}{rrrrrrrr}
  \hline
 & lasso & forward & backward & subset & posterior HPD & $S_{min}$  & $S_{small}$ \\ 
  \hline
TPR & 0.54 & 0.95 & 0.89 & 0.72 & 0.16 & 0.77 & 0.34 \\ 
  TNR & 0.92 & 0.06 & 0.11 & 0.64 & 1.00 & 0.78 & 0.99 \\
   \hline
\end{tabular}

\vspace{4mm} \caption*{$n=200, p = 400, \mbox{SNR} = 0.25$}\vspace{-1mm}
\begin{tabular}{rrrrrrrr}
  \hline
 & lasso & forward & backward & subset & posterior HPD & $S_{min}$  & $S_{small}$ \\ 
  \hline
TPR & 0.36 & 0.79 & 0.78 & 0.67 & 0.17 & 0.75 & 0.34 \\ 
  TNR & 1.00 & 0.52 & 0.54 & 0.96 & 1.00 & 0.95 & 1.00 \\ 
   \hline
\end{tabular}

\vspace{4mm} \caption*{$n=200, p = 400, \mbox{SNR} = 1$}\vspace{-1mm}
\begin{tabular}{rrrrrrrr}
  \hline
 & lasso & forward & backward & subset & posterior HPD & $S_{min}$  & $S_{small}$ \\ 
  \hline
TPR & 0.97 & 0.96 & 0.95 & 0.95 & 0.75 & 0.98 & 0.93 \\ 
  TNR & 0.99 & 0.57 & 0.60 & 0.96 & 1.00 & 0.95 & 1.00 \\ 
   \hline
\end{tabular}

\vspace{4mm} \caption*{$n=500, p = 50, \mbox{SNR} = 0.25$}\vspace{-1mm}
\begin{tabular}{rrrrrrrr}
  \hline
 & lasso & forward & backward & subset & posterior HPD & $S_{min}$  & $S_{small}$ \\ 
  \hline
TPR & 0.84 & 0.96 & 0.95 & 0.94 & 0.59 & 0.99 & 0.86 \\ 
  TNR & 0.98 & 0.83 & 0.83 & 0.83 & 1.00 & 0.70 & 0.98 \\
   \hline
\end{tabular}

\vspace{4mm} \caption*{$n=500, p = 50, \mbox{SNR} = 1$}\vspace{-1mm}
\begin{tabular}{rrrrrrrr}
  \hline
 & lasso & forward & backward & subset & posterior HPD & $S_{min}$  & $S_{small}$ \\ 
  \hline
TPR & 1.00 & 1.00 & 1.00 & 1.00 & 1.00 & 1.00 & 1.00 \\ 
  TNR & 1.00 & 0.84 & 0.83 & 0.83 & 1.00 & 0.67 & 0.99 \\ 
   \hline
\end{tabular}
\vspace{4mm}
\caption{\label{tab:tprs} \small
True positive rates (TPR) and true negative rates (TNR) for Gaussian synthetic data with $p^* + 1= 6$ active covariates including the intercept. The selection performance of $S_{small}$ is excellent and similar to the the adaptive lasso. Classical subset, forward, and backward selection are too aggressive, while the posterior interval-based selection is too conservative. 
}
\end{table}

The appendix contains simulation studies for classification (Appendix~\ref{sims-class}) and variations for $\varepsilon \in \{0.01, 0.1, 0.2\}$, $\{\bm{\tilde x}_i\}_{i=1}^{\tilde n}$, and the distributions of $\{\bm{x}_i\}_{i=1}^{n}$ and $\{\bm{\tilde x}_i\}_{i=1}^{\tilde n}$ (Appendix~\ref{sims-app}). All results are qualitatively similar to those presented here.

\section{Subset selection for predicting educational outcomes}\label{app}
Childhood educational outcomes are affected by adverse environmental exposures, such as poor air quality and lead, as well as social stressors, such as poverty, racial residential isolation, and neighborhood deviation. We study childhood educational development using end-of-grade reading test scores for a large cohort of fourth grade children in North Carolina \citep{ChildrensEnvironmentalHealthInitiative2020}. The reading scores are accompanied by a collection of student-level information detailed in Figure~\ref{fig:vars}, which includes air quality exposures, birth information, blood lead measurements, and social and economic factors (see also \citealp{KowalPRIME2020}). The goal is to identify subsets of these variables and interactions that offer near-optimal prediction of reading scores as well as accurate classification of at-risk students. 

\begin{figure}[h!]
\begin{center}
\includegraphics[width=1\textwidth]{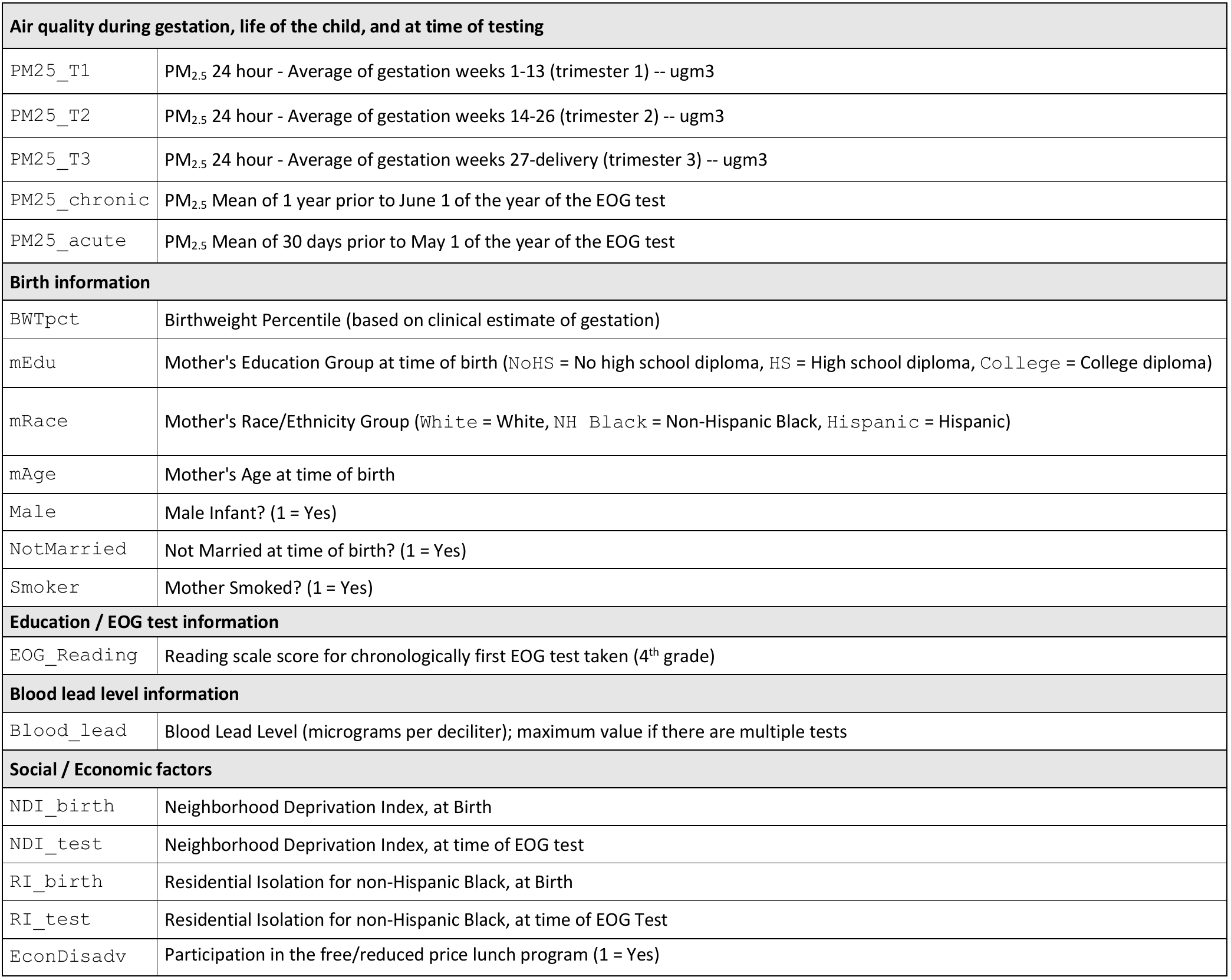}
\end{center}
\caption{ \small
Variables in the NC education dataset. Data are restricted to individuals with 37-42 weeks gestation, mother's age 15-44, $\texttt{Blood\_lead} \le 80$, birth order $\le 4$, no current limited English proficiency,  and residence in NC at time of birth and time of 4th grade test. 
 \label{fig:vars}}
\end{figure}

A prominent feature of the data is the correlation among the covariates. After centering and scaling the continuous covariates to mean zero and standard deviation 0.5, we augment the variables in Figure~\ref{fig:vars} (excluding the test scores) with interactions between  race and the social and economic factors. The resulting dimensions are $n=16806$ and $p=32$. Figure~\ref{fig:corrplot} displays the pairwise correlations among the covariates and the response variable. There are strong associations among the air quality exposures as well as among race and the social and economic factors. Due to the dependences among variables, it is likely that distinct subsets of similar predictive ability can be obtained by interchanging among these correlated covariates. Hence, it is advantageous to collect and study the near-optimal subsets.

We compute acceptable families for (i) the reading scores $y_i$ under squared error loss and (ii) the indicator $h(\tilde y_i) = \mathbb{I}\{\tilde y_i  \ge \tau_{0.1}\}$ under cross-entropy loss, where $\tau_{0.1}$ is the 0.1-quantile of the reading scores (see Appendix~\ref{app-class}). While task (i) broadly considers the spectrum of educational outcomes via reading scores, task (ii) targets at-risk students in the bottom 10\% of reading ability. Acceptable families and accompanying quantities for both tasks can be computed using the same Bayesian model $\mathcal{M}$: we focus on  Gaussian linear regression with horseshoe priors. 
The acceptable families are computed using the proposed BBA search with  $\eta = 0\%$, $\varepsilon = 0.1$, $m_k=100$  and $\{\bm{\tilde x}_i\}_{i=1}^{\tilde n} = \{\bm{ x}_i\}_{i=1}^n$; results for other $\eta$ values, $\varepsilon =0.2$, and $m_k=15$ are noted, while alternative target covariates $\{\bm{\tilde x}_i\}_{i=1}^{\tilde n}$ are in  Section~\ref{sec-oos}.

\subsection{Subset selection for predicting reading scores}
First, we predict reading scores using a linear model for $\mathcal{M}$ and squared error loss for $L$. 
Since acceptable family is defined based on $\widetilde{{D}}_{\mathcal{S}_{min},\mathcal{S}}^{out}$ in \eqref{accept}, we summarize its distribution in Figure~\ref{fig:d-lm}. For each $\mathcal{S} \in \mathbb{S}$, we display 80\% intervals, expectations, and the empirical analog ${{D}}_{\mathcal{S}, \mathcal{S}_{min}}^{out} \coloneqq 100\times({{L}}_{\mathcal{S}}^{out} - {{L}}_{\mathcal{S}_{min}}^{out})/{{L}}_{\mathcal{S}_{min}}^{out}$. The smaller subsets of sizes four to six demonstrate clear separation for certain subsets. Along with the intercept and the race indicators, these subsets include \texttt{mEdu} (college diploma), \texttt{EconDisadv}, and  \texttt{mEdu} (completed high school) in sequence. 
However, larger subsets are needed to procure near-optimal predictions for smaller margins, such as $\eta < 2\%$. While the best subset $\vert\mathcal{S}_{min}\vert = 29$ includes nearly all of the covariates, many  subsets with $\vert \mathcal{S}\vert > 10$ achieve within $\eta = 1\%$ of the accuracy of $\mathcal{S}_{min}$.

\begin{figure}[h!]
\begin{center}
\includegraphics[width=.6\textwidth]{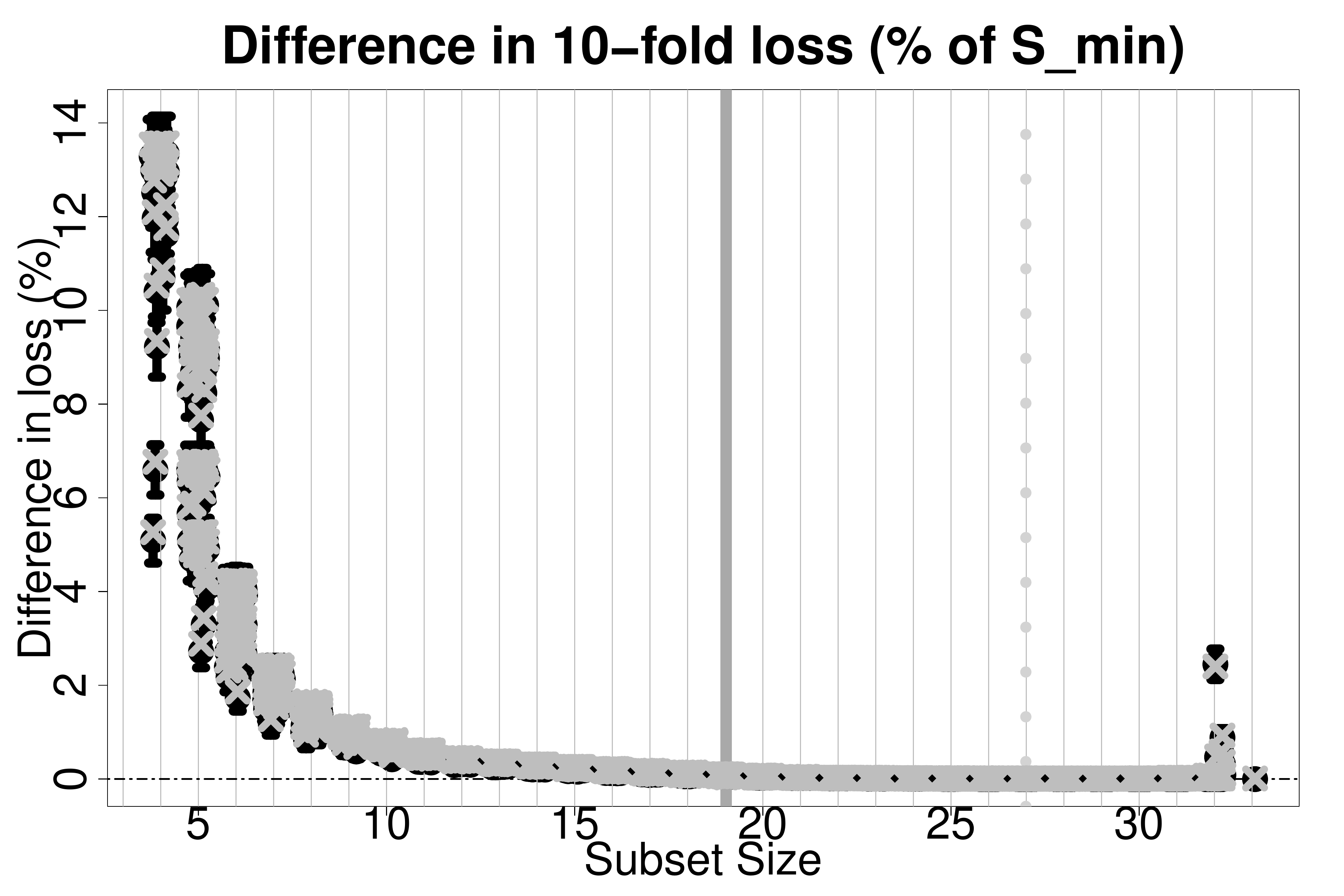}
\end{center}
\caption{ \small
The 80\% intervals (bars) and expected values (circles) for $\widetilde{{D}}_{\mathcal{S}_{min},\mathcal{S}}^{out}$ with ${{D}}_{\mathcal{S}, \mathcal{S}_{min}}^{out}$ 
(x-marks) under squared error loss for each subset size $\vert \mathcal{S}\vert$ with $\mathcal{S} \in \mathbb{S}$. We annotate $\mathcal{S}_{min}$ (dashed gray line) and $\mathcal{S}_{small}$ (solid gray line) 
for $\varepsilon = 0.10$ and $\eta = 0$ and jitter the subset sizes for clarity of presentation.  
\label{fig:d-lm}}
\end{figure}

 Among the $\vert \mathbb{S} \vert = 2761$ candidate subsets identified from the BBA search, there are $\vert \mathbb{A}_{0, 0.1}\vert = 1183$ acceptable subsets. We summarize $\mathbb{A}_{\eta, \varepsilon}$ via the co-variable importance metrics $\mbox{VI}_{\rm incl}(j)$ and  $\mbox{VI}_{\rm incl}(j, \ell)$   in Figures~\ref{fig:vi-lm}~and~\ref{fig:co-vi-lm}, respectively. Unlike many variable importance metrics that measure effect sizes, 
$\mbox{VI}_{\rm incl}(j)$ instead quantifies whether each covariate $j$ is a component of all, some, or no competitive subsets. There are many keystone covariates that appear in (nearly) all acceptable subsets, including environmental exposures (prenatal air quality and blood lead levels), economic and social factors (\texttt{EconDisadv}, mother's education level, neighborhood deprivation at time of test), and demographic information (race, gender), among others.

\begin{figure}[h!]
\begin{center}
\includegraphics[width=.75\textwidth]{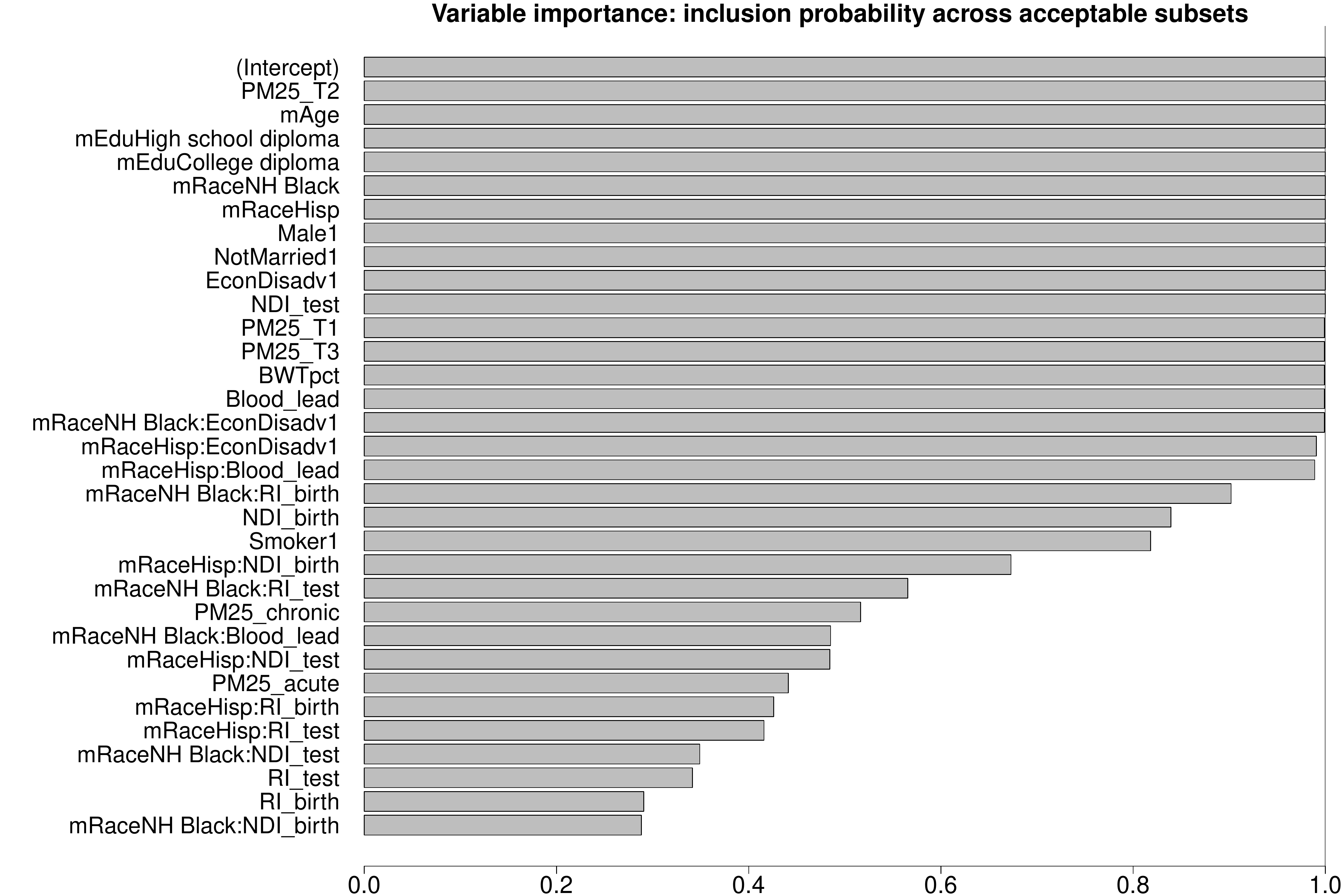}
\end{center}
\caption{ \small
Variable importance $\mbox{VI}_{\rm incl}(j)$ for prediction. There are several tiers: variables appear in (nearly) all,  many ($>70$\%), some ($>30$\%), or no acceptable subsets.
\label{fig:vi-lm}}
\end{figure} 

Interestingly,  chronic and acute $\mbox{PM}_{2.5}$ exposure each belong to nearly 50\% of acceptable subsets (Figure~\ref{fig:vi-lm}), yet rarely appear in the same acceptable subset (Figure~\ref{fig:co-vi-lm}). The pairwise correlations (Figure~\ref{fig:corrplot}) offer a reasonable explanation: these variables are weakly correlated with reading scores but highly correlated with one another. Similar results persist for neighborhood deprivation and racial residential isolation both at birth and time of test. Moreover, this analysis was conducted after removing one acute $\mbox{PM}_{2.5}$ outlier (50\% larger than all other values). When that outlier is kept in the data, acute $\mbox{PM}_{2.5}$ no longer belongs to \emph{any} acceptable subset, while $\mbox{VI}_{\rm incl}(j)$ for chronic $\mbox{PM}_{2.5}$ increases. These results are encouraging: the acceptable family identifies redundant yet distinct predictive explanations, but prefers the more stable covariate in the presence of outliers.

Next, we analyze the smallest acceptable subset $\mathcal{S}_{small}$ and incorporate uncertainty quantification for the accompanying linear coefficients. The $\vert \mathcal{S}_{small}\vert = 19$ selected covariates  are displayed in Figure~\ref{fig:coef-lm} alongside the point and 90\% intervals based on the proposed approach, the Bayesian model $\mathcal{M}$, and the adaptive lasso. The estimates and intervals for covariates excluded from $\mathcal{S}_{small}$ are identically zero for the proposed approach (and, in this case, the adaptive lasso as well), while the estimates and HPD intervals from $\mathcal{M}$ are dense for all covariates. Despite the encouraging simulation results for the  \cite{Zhao2017} frequentist intervals, these  intervals often exclude the adaptive lasso point estimates, which undermines  interpretability.

The estimates from $\mathcal{S}_{small}$ and $\mathcal{M}$ are similar with anticipated directionality: higher mother's education levels, lower blood lead levels in the child, less neighborhood deprivation, and absence of economic disadvantages predict higher reading scores. Prenatal air quality exposure ($\mbox{PM}_{2.5})$ is significant: due to seasonal effects, the 1st and 3rd trimester exposures have negative coefficients, while the 2nd trimester has a positive effect. Naturally, these effects can only be interpreted jointly. Among interaction terms, the negative effect of \texttt{NH Black} $\times $ \texttt{RI\_birth} suggests the lower reading scores among non-Hispanic Black students are accentuated by racial residential isolation in the neighborhood of the child's birth. Since we do not force all main effects into each subset, $\mathcal{S}_{small}$ does not contain an estimated effect of \texttt{RI\_birth} for other race groups. Other interactions, such as the positive effects of \texttt{Hisp} and \texttt{NH Black} by \texttt{EconDisadv} and \texttt{Hisp} $\times$ \texttt{Blood\_lead}, must also be interpreted carefully: the vast majority of Hispanic and non-Hispanic Black students belong to the \texttt{EconDisadv} group and have much higher blood lead levels on average, while each of  \texttt{NH Black}, \texttt{EconDisadv}, and \texttt{Blood\_lead} has a strong negative main effect.

\begin{figure}[h!]
\begin{center}
\includegraphics[width=.9\textwidth]{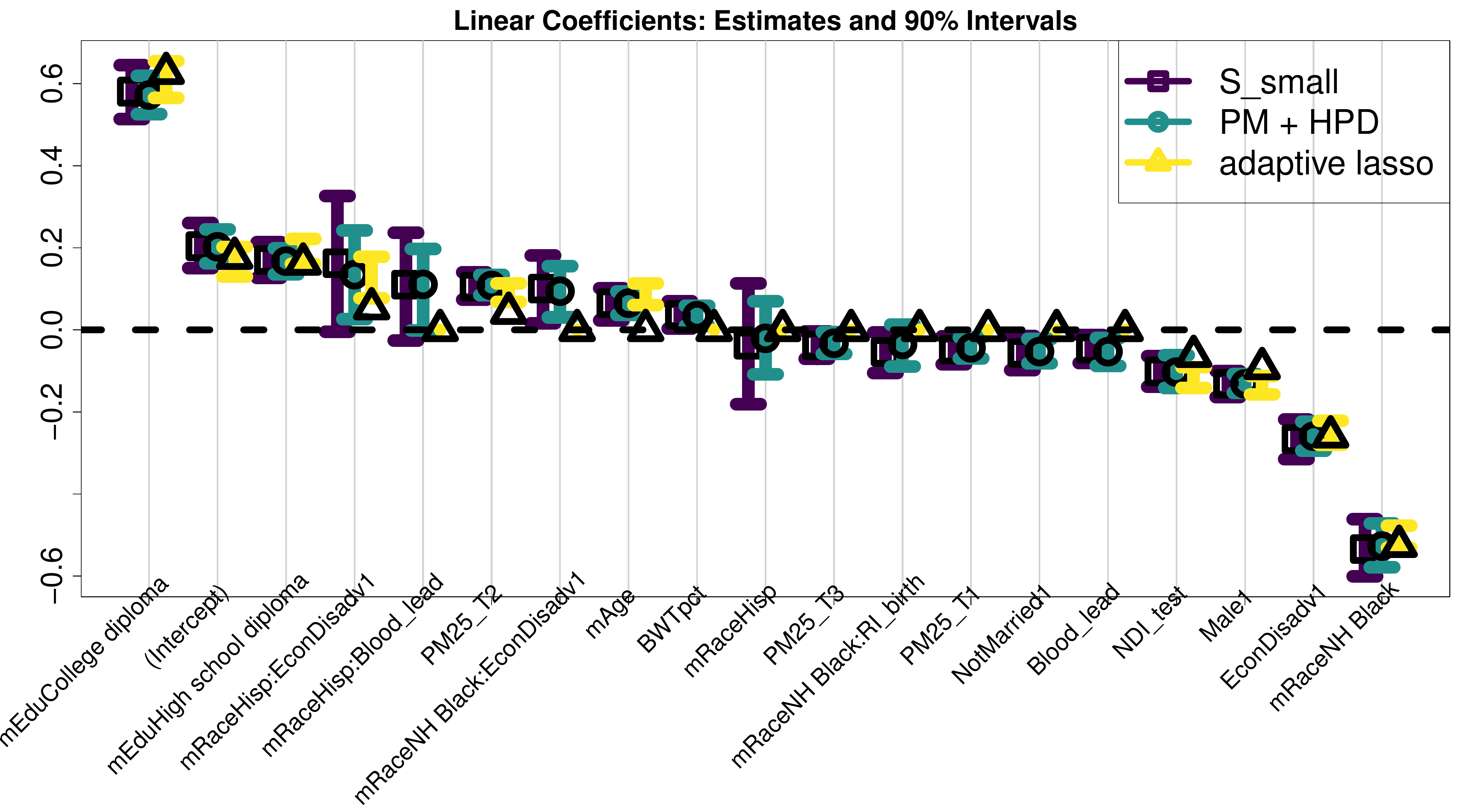}
\end{center}
\caption{ \small
Estimated linear effects and 90\% intervals for the variables in $\mathcal{S}_{small}$ based on the proposed approach, model $\mathcal{M}$, and the adaptive lasso. 
\label{fig:coef-lm}}
\end{figure}

The results are not particularly sensitive to $m_k$  or $\varepsilon$. 
When $m_k = 15$, there are $\vert \mathbb{S} \vert = 436$ candidate subsets and $\vert \mathbb{A}_{0, 0.1}\vert = 197$ acceptable subsets. The variable importance metrics broadly agree with Figures~\ref{fig:vi-lm}~and~\ref{fig:co-vi-lm}, while $\mathcal{S}_{small}$---and therefore Figure~\ref{fig:coef-lm}---is unchanged. When $\varepsilon = 0.2$ (and $m_k=100$ as before), the acceptable family reduces slightly to $\vert \mathbb{A}_{0, 0.1}\vert = 977$ members and $\mathcal{S}_{small}$ differs only in the addition of \texttt{Smoker} and \texttt{Hisp} $\times$ \texttt{NDI\_test}.



\subsection{Out-of-sample prediction}\label{sec-oos}
We evaluate the predictive capabilities of the proposed approach for 20 training/testing splits of the NC education data. The same competing methods are adopted from Section~\ref{sims}, including the distinct search strategies for collecting near-optimal subsets. Since $\mathcal{S}_{small}$ and $\mathcal{S}_{min}$ are reasonably robust to $m_k$, we select $m_k = 15$ for computational efficiency. In addition, we include the acceptable family defined by setting $\{\bm{\tilde x}_i\}_{i=1}^{\tilde n}$ to be the testing data covariate values (\texttt{bbound(Xtilde)}), which is otherwise identical to \texttt{bbound(bayes)}. Root mean squared prediction errors (RMSPEs)  are used for evaluation.

The results are presented in Figure~\ref{fig:app-oos}. Among single subset methods, $\mathcal{S}_{small}$ outperforms all competitors---including the classical \texttt{forward} and \texttt{backward} estimators and the smallest acceptable subsets from  \texttt{lasso(bayes)}, \texttt{forward(bayes)}, and \texttt{backward(bayes)} discussed in Section~\ref{sims} (not shown). The adaptive lasso selects fewer variables and is not competitive. Among search methods, Figure~\ref{fig:app-oos} confirms the results from the simulation study: the proposed BBA strategy (\texttt{bbound(bayes)}) identifies  10-25 times the number of subsets as the other search strategies, yet does not sacrifice any predictive accuracy in this expanded collection. Clearly, \texttt{bbound(bayes)} provides a more complete predictive picture, as the competing search strategies omit a massive number of subsets that \emph{do} offer near-optimal prediction. The acceptable family based on the out-of-sample covariates \texttt{bbound(Xtilde)} is much larger, which is reasonable: the subsets are computed and evaluated on covariates for which the accompanying response variables are not available. This collection of subsets sacrifices some predictive accuracy relative to the other search strategies, yet still outperforms the adaptive lasso---yet another testament to the importance of the regularization induced by $\mathcal{M}$.

\begin{figure}[h!]
\begin{center}
\includegraphics[width=.7\textwidth]{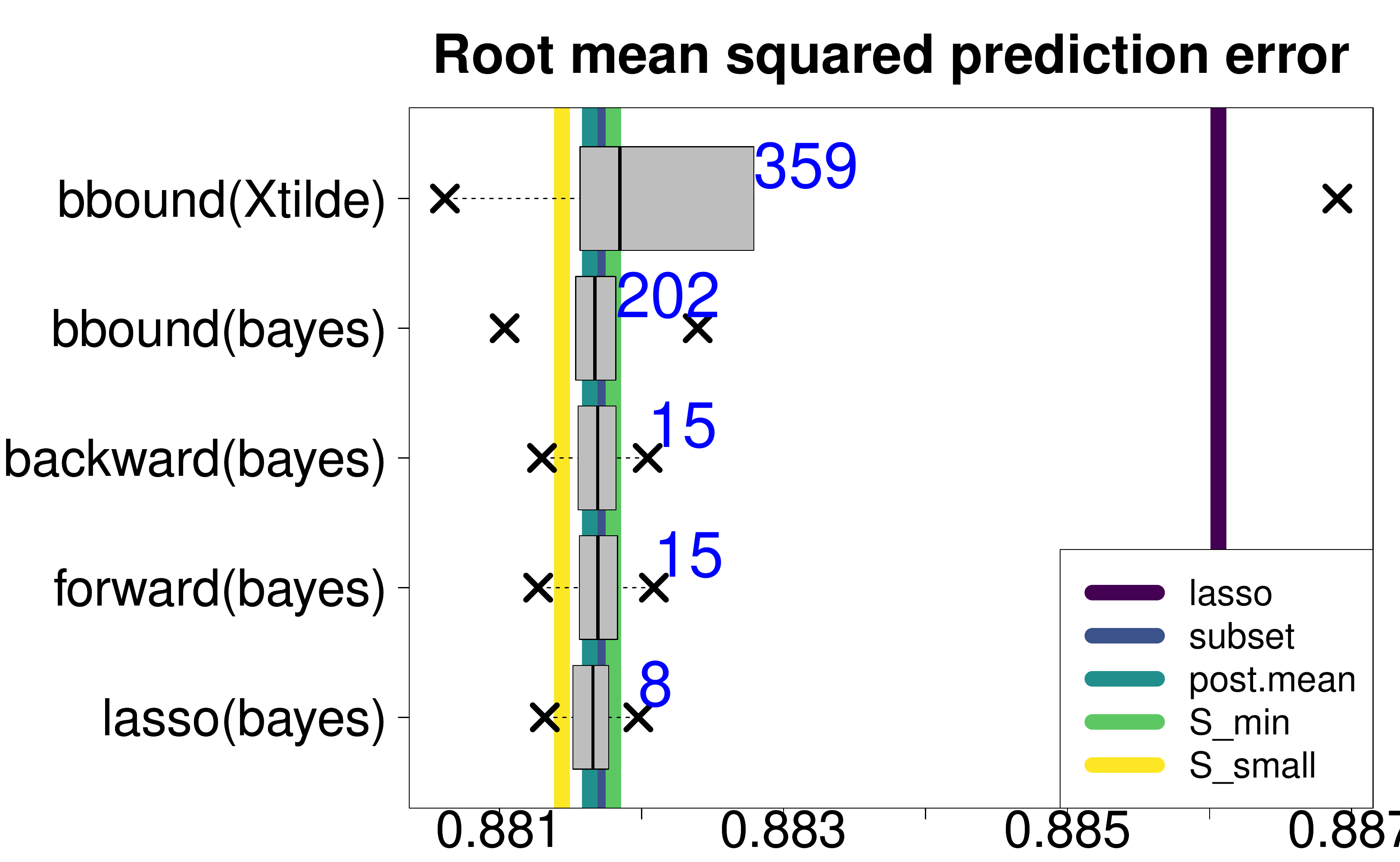}
\end{center}
\caption{ \small
Root mean squared prediction errors (RMSPEs) across 20 training/testing splits of the NC education data. The boxplots summarize the RMSPE quantiles for the subsets within each acceptable family, while the vertical lines denote RMSPEs of competing methods. The average size of each acceptable family is annotated. The proposed BBA search returns vastly more subsets that remain highly competitive, while the accompanying $\mathcal{S}_{small}$ subset performs best overall. 
\label{fig:app-oos}}
\end{figure} 

Although the RMSPE differences appear to be small, minor improvements are practically relevant: even a single point on a standardized test score can be the difference between progression to the next grade level (on the low end) or eligibility for intellectually gifted programs (on the high end). More generally, education data are prone to weak signals and small effect sizes, which are conditions under which \emph{many} methods may offer similar predictive performance. Nonetheless, our primary goal is not to substantially \emph{improve} prediction, but rather to identify and analyze a large collection of near-optimal subsets.


The importance of the broader BBA search strategy is further highlighted in comparison  \texttt{lasso(bayes)}, which uses a lasso search path  for decision analytic Bayesian variable selection \citep{hahn2015decoupling,KowalPRIME2020}. In particular,  \texttt{lasso(bayes)} generates only $\vert \mathbb{S} \vert = 25$ candidate subsets and  $\vert \mathbb{A}_{0, 0.1}\vert = 9$ acceptable subsets. By comparison, \texttt{bbound(bayes)} returns more than 100 times the number of candidate subsets \emph{and} acceptable subsets. Figure~\ref{fig:app-oos} shows that the subsets omitted by \texttt{lasso(bayes)} yet discovered by \texttt{bbound(bayes)} are indeed near-optimal. Further, the restrictive search path of \texttt{lasso(bayes)} does not guarantee greater stability: the smallest acceptable subsets for  \texttt{bbound(bayes)} and \texttt{lasso(bayes)} are nearly identical, yet the interquartile range of $\vert \mathcal{S}_{small} \vert$ across the 20 training/testing splits under \texttt{bbound(bayes)} is 0, while the same quantity under  \texttt{lasso(bayes)} is 2. Indeed, by this metric, \texttt{bbound(bayes)} actually provides the \emph{most stable} smallest acceptable subset across all search strategies considered.




\section{Discussion}\label{disc}
We developed decision analysis tools for Bayesian subset search, selection, and (co-) variable importance. The proposed strategy is outlined in Algorithm~\ref{alg:overview}.
Building from a Bayesian predictive model $\mathcal{M}$, we derived optimal linear actions for any subset of covariates. We explored the space of subsets using an adaptation of the branch-and-bound algorithm. After filtering to a manageable collection of promising subsets, we identified the \emph{acceptable family} of near-optimal subsets for linear prediction or classification. The acceptable family was summarized by a new (co-) variable importance metric---the frequency with which variables (co-) appear in all, some, or no acceptable subsets---and individual member subsets, including the ``best" and smallest subsets. Using the posterior predictive distribution from $\mathcal{M}$, we developed point and interval estimates for the linear coefficients of any subset. Simulation studies demonstrated better prediction, interval estimation, and variable selection for the proposed approach compared to existing Bayesian and frequentist selection methods---even for high-dimensional data with $p > n$.

\begin{algorithm}[ht]
\SetAlgoLined 
\begin{enumerate}
\item Fit a \emph{Bayesian predictive model} $\mathcal{M}$; 
\item Specify a \emph{loss function} for prediction or classification, including design points (covariate values) $\{\bm{\tilde x}_i\}_{i=1}^{\tilde n}$ and local weights $\{\omega(\bm{\tilde x}_i)\}_{i=1}^{\tilde n}$;
\item Filter to a family of \emph{candidate subsets} of covariates: 
\begin{enumerate}
\item Screen to the best $s_{max} \le p$ covariates based on $\mathcal{M}$;
\item Apply the branch-and-bound algorithm to select the best $m_k \le {p \choose k}$ subsets of each size $k=1,\ldots,p$;
\end{enumerate}
\item Collect the \emph{acceptable family} of subsets that offer near-optimal prediction or classification;
\item Summarize the acceptable family via \emph{(co-) variable importance}, the \emph{best} subset, and the \emph{smallest} subset.
\item  Obtain posterior predictive samples of the coefficients $\bm{\tilde \delta}_\mathcal{S}$ for any subset $\mathcal{S}$ of interest.
\end{enumerate}
 \caption{Bayesian subset selection for interpretable prediction and classification} \label{alg:overview}
\end{algorithm}

We applied these tools to a large education dataset to study the factors that predict educational outcomes. The analysis identified several keystone covariates that appeared in (almost) every near-optimal subset, including environmental exposures, economic and social factors, and demographic information. The co-variable importance metrics highlighted an interesting phenomenon where certain pairs of covariates belonged to many acceptable subsets, yet rarely appeared in the same acceptable subset. Hence, these variables are effectively interchangeable for prediction, which provides valuable context for interpreting their respective effects. We showed that the smallest acceptable subset offers excellent prediction of end-of-grade reading scores and classification of at-risk students using substantially fewer covariates. The corresponding linear coefficients described new and important effects, for example that greater racial residential isolation among non-Hispanic Black students is predictive of lower reading scores.  However, our results also caution against overreliance on any particular subset: we identified 
over 200 distinct subsets of variables that offer near-optimal out-of-sample predictive accuracy.
  


Future work will attempt to generalize these tools via the loss functions and the  actions. Alternatives to squared error and cross-entropy loss can be incorporated with an IRLS approximation strategy similar to Section~\ref{sec-logit}, which would maintain the methodology and algorithmic infrastructure from the proposed approach. Similarly, the class of parametrized actions can be expanded to include nonlinear predictors, such as trees or additive models, with acceptable families constructed in the same way. 



\acks{Research was sponsored by the Army Research Office (W911NF-20-1-0184) and the National Institute of Environmental Health Sciences of the National Institutes of Health (R01ES028819). The content, views, and conclusions  contained in this document are those of the authors and should not be interpreted as representing the official policies, either expressed or implied, of the Army Research Office, the North Carolina Department of Health and Human Services, Division of Public Health, the National Institutes of Health, or the U.S. Government. The U.S. Government is authorized to reproduce and distribute reprints for Government purposes notwithstanding any copyright notation herein.}


\appendix

\numberwithin{equation}{section}
\numberwithin{figure}{section}
\numberwithin{table}{section}

\section{Approximations for out-of-sample quantities}\label{sir-approx}
The acceptable family $\mathbb{A}_{\eta, \varepsilon}$ in \eqref{accept} is based on the out-of-sample predictive discrepancy metric $\widetilde{{D}}_{\mathcal{S}_1,\mathcal{S}_2}^{out}$ and the best subset $\bm{\hat \delta}_{\mathcal{S}_{min}}$ in \eqref{best}. These quantities both depend on the out-of-sample {empirical} and {predictive} losses ${L}_{\mathcal{S}}^{out}(k)$ and $\widetilde{{L}}_{\mathcal{S}}^{out}(k)$, respectively, in
\eqref{out-loss-k}. Hence, it is required to compute (i) the optimal action on the training data, $\bm{\hat \delta}_{\mathcal{S}}^{-\mathcal{I}_k} \coloneqq \arg\min_{\bm\delta_\mathcal{S}} \mathbb{E}_{\bm{\tilde y} | \bm y^{-\mathcal{I}_k}} \mathcal{L}(\{{\tilde y}_i\}_{i \in \mathcal{I}_k}, \bm \delta_\mathcal{S}) $, and (ii) samples from the out-of-sample predictive distribution $\tilde y_i^{-\mathcal{I}_k} \sim p_{\mathcal{M}}({\tilde y}_i | \bm y^{-\mathcal{I}_k})$. Because of the simplifications afforded by \eqref{sq-loss-action-1}-\eqref{sq-loss-action-2} and \eqref{opt-logit},  computing $\bm{\hat \delta}_{\mathcal{S}}^{-\mathcal{I}_k}$ only requires the out-of-sample expectations $\hat y_i^{-\mathcal{I}_k} \coloneqq  \mathbb{E}_{\bm{\tilde y} | \bm y^{-\mathcal{I}_k}} \{\bm{\tilde y}(\bm{\tilde x}_i)\}$ or $\hat h_i^{-\mathcal{I}_k}  \coloneqq \mathbb{E}_{\bm{\tilde y} | \bm y^{-\mathcal{I}_k}} h( \tilde y_i)$ for the classification setting. For many models $\mathcal{M}$, there is a further simplification that $\mathbb{E}_{\bm{\tilde y} | \bm y^{-\mathcal{I}_k}} \{\bm{\tilde y}(\bm{\tilde x}_i)\} = \mathbb{E}_{\bm \theta | \bm y^{-\mathcal{I}_k}} [\mathbb{E}_{\bm{\tilde y} | \bm \theta} \{\bm{\tilde y}(\bm{\tilde x}_i)\}]$, where often $\mathbb{E}_{\bm{\tilde y} | \bm \theta} \{\bm{\tilde y}(\bm{\tilde x}_i)\}$ has an explicit form (such as in regression models). Absent such simplifications, the expectations can be computed by averaging the draws of $\tilde y_i^{-\mathcal{I}_k} \sim p_{\mathcal{M}}({\tilde y}_i | \bm y^{-\mathcal{I}_k})$.

Although these terms can be computed by repeatedly re-fitting the Bayesian model $\mathcal{M}$ for each training/validation split $k=1,\ldots,K$, such an approach is computationally demanding. Instead, we apply a sampling-importance resampling (SIR) algorithm  to approximate these out-of-sample quantities. Notably, the SIR algorithm requires only a single fit of model $\mathcal{M}$ to the complete data---which is already needed for posterior inference---and therefore contributes minimally to the computational cost of the aggregate analysis.

The details are provided in Algorithm~\ref{alg:loo}.  By construction, the samples $\{{\tilde y}_i^{\tilde s}\}_{\tilde s = \tilde s_1}^{\tilde S}$ are from the out-of-sample predictive distribution $ p_{\mathcal{M}}({\tilde y}_i | \bm y^{-\mathcal{I}_k})$.  Based on Algorithm~\ref{alg:loo}, it is straightforward to compute draws of $\widetilde{{D}}_{\mathcal{S}_1,\mathcal{S}_2}^{out}$ for any $\mathcal{S}_1, \mathcal{S}_2 \in \mathbb{S}$, as well as the best subset $\bm{\hat \delta}_{\mathcal{S}_{min}}$ in \eqref{best}. Therefore,  the acceptable family $\mathbb{A}_{\eta, \varepsilon}$ is readily computable for any $\eta, \varepsilon$. By default,  we use $\tilde S = \lfloor S/2 \rfloor$ SIR samples.

\begin{algorithm}[h]
\SetAlgoLined
\begin{enumerate}
\item  Obtain posterior samples $\{\bm \theta^s\}_{s=1}^S \sim p_\mathcal{M}(\bm \theta | \bm y)$; 
\item For each training set $k = 1,\ldots,K$:
\begin{enumerate}
\item Compute $\log w_k^s \stackrel{c}{=} - \log p_\mathcal{M}(\bm y^{\mathcal{I}_k} | \bm \theta^s) = - \sum_{i \in \mathcal{I}_k} \log p_\mathcal{M}(\bm y_i | \bm \theta^s)$ (up to a constant);
\item Sample $\{\tilde s_1,\ldots, \tilde s_{\tilde S}\}$ without replacement from $\{1,\ldots, S\}$  with probability weights $\{w_k^1,\ldots, w_k^S\}$; 
\item Sample ${\tilde y}_i^{\tilde s} \sim p_\mathcal{M}( {\tilde y}_i | \bm \theta^{\tilde s})$ for $\tilde s = \tilde s_1,\ldots, \tilde s_{\tilde S}$ and  $i =1,\ldots,n$; 
\item Compute $\hat y_j^{-\mathcal{I}_k} \approx {\tilde S}^{-1} \sum_{\tilde s = \tilde s_1}^{\tilde S} {\tilde y}_j^{\tilde s}$  for $j \not \in \mathcal{I}_k$;
\item Compute $\bm{\hat \delta}_{\mathcal{S}}^{-\mathcal{I}_k}$ for each $\mathcal{S} \in \mathbb{S}$ by solving \eqref{sq-loss-action-2} using the training data covariates $\{\bm x_i\}_{i \not \in \mathcal{I}_k}$, the weights $\{\omega(\bm x_i)\}_{i \not \in \mathcal{I}_k}$, and the pseudo-data $\{\hat y_j^{-\mathcal{I}_k}\}_{j\not\in\mathcal{I}_k}$; 
\item Compute ${L}_{\mathcal{S}}^{out}(k)$ and $\{\widetilde{{L}}_{\mathcal{S}}^{out, \tilde s}(k)\}_{\tilde s = \tilde s_1}^{\tilde S}$ in \eqref{out-loss-k} using $\bm{\hat \delta}_{\mathcal{S}}^{-\mathcal{I}_k}$ and $\{\tilde y_i^{\tilde s}\}_{\tilde s = \tilde s_1}^{\tilde S}$; 
\end{enumerate}
\item Compute ${L}_{\mathcal{S}}^{out} = K^{-1} \sum_{k=1}^K {L}_{\mathcal{S}}^{out}(k)$ and $ \widetilde{{L}}_\mathcal{S}^{out, \tilde s}  = K^{-1} \sum_{k=1}^K\widetilde{{L}}_\mathcal{S}^{out, \tilde s}(k)$  for $\tilde s = \tilde s_1,\ldots,\tilde S$.
\end{enumerate}
 \caption{Out-of-sample predictive evaluations.} \label{alg:loo}
\end{algorithm}

Algorithm~\ref{alg:loo} recycles computations: steps 1-2(d) are shared among all subsets $\mathcal{S}$ and loss functions $L$. Hence, the algorithm is efficient even when the number of candidate subsets $\vert \mathbb{S}\vert$ is large---and notably does not require re-fitting the Bayesian model $\mathcal{M}$. Modifications for classification (Section~\ref{sec-logit}) are straightforward: steps (c), (d), and (e) are replaced by  (c$'$) $\bm{\tilde y}_i^{\tilde s} \sim p_\mathcal{M}( {\tilde y}_i | \bm \theta^{\tilde s})$ for $\tilde s = \tilde s_1,\ldots, \tilde s_{\tilde S}$ and  $i =1,\ldots,n$; (d$'$) $\hat h_j^{-\mathcal{I}_k} \approx {\tilde S}^{-1} \sum_{\tilde s = \tilde s_1}^{\tilde S} h(\bm{\tilde y}_j^{\tilde s})$ for $j \not \in \mathcal{I}_k$;
 and (e$'$) compute $\bm{\hat \delta}_{\mathcal{S}}^{-\mathcal{I}_k}$ by solving \eqref{opt-logit} using the training data covariates $\{\bm x_i\}_{i \not \in \mathcal{I}_k}$, the weights $\{\omega(\bm x_i)\}_{i \not \in \mathcal{I}_k}$, and the pseudo-data $\{\hat h_j^{-\mathcal{I}_k}\}_{j \not \in \mathcal{I}_k}$.

\section{Simulation study for classification} \label{sims-class}
The synthetic data-generating process for classification mimics that for prediction: the only difference is that the data are generated as $y_i \stackrel{indep}{\sim}\mbox{Bernoulli}(\pi_i^*)$ with  $\pi_i^* \coloneqq \{1 + \exp(-y_i^*)\}^{-1}$ and $y_i^* \coloneqq \bm x_i'\bm\beta^*$ as before. For the Bayesian model $\mathcal{M}$, we use a logistic regression model with horseshoe priors and estimated using \texttt{rstanarm} \citep{rstanarm}. The competing estimators are constructed similarly as before, now using cross-entropy loss \eqref{cross-ent} for the proposed approach and the logistic likelihood for the adaptive lasso. 

The classification performance is evaluated using cross-entropy loss for $\pi_i^*$ in Figure~\ref{fig:sims-class-loss} (top), using the same competing search methods as in Figure~\ref{fig:sims-pred}. As in the regression case, the proposed \texttt{bbound(bayes)} search procedure returns vastly more subsets in the acceptable family, yet maintains excellent classification performance within this collection. In addition, classification based on $\mathcal{S}_{min}$ and $\mathcal{S}_{small}$ substantially outperform the adaptive lasso. We also include the 90\% interval comparisons in Figure~\ref{fig:sims-class-loss} (bottom), which confirm the results from the prediction scenario: $\mathcal{S}_{small}$ and \cite{Zhao2017} (modified for the logistic case) achieve the nominal coverage with the narrowest intervals.


\begin{figure}[h!]
\begin{center}
\includegraphics[width=.49\textwidth]{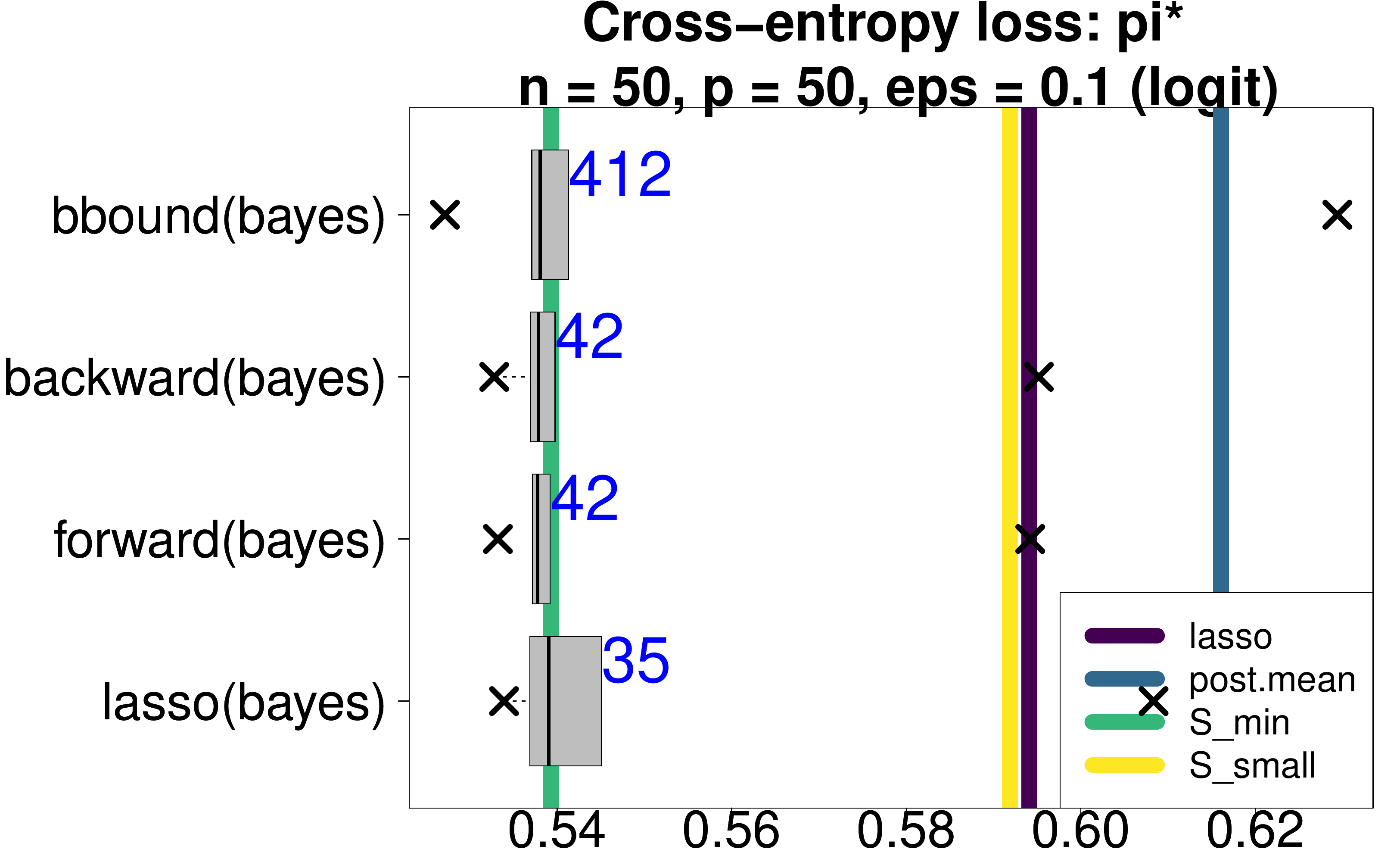}
\includegraphics[width=.49\textwidth]{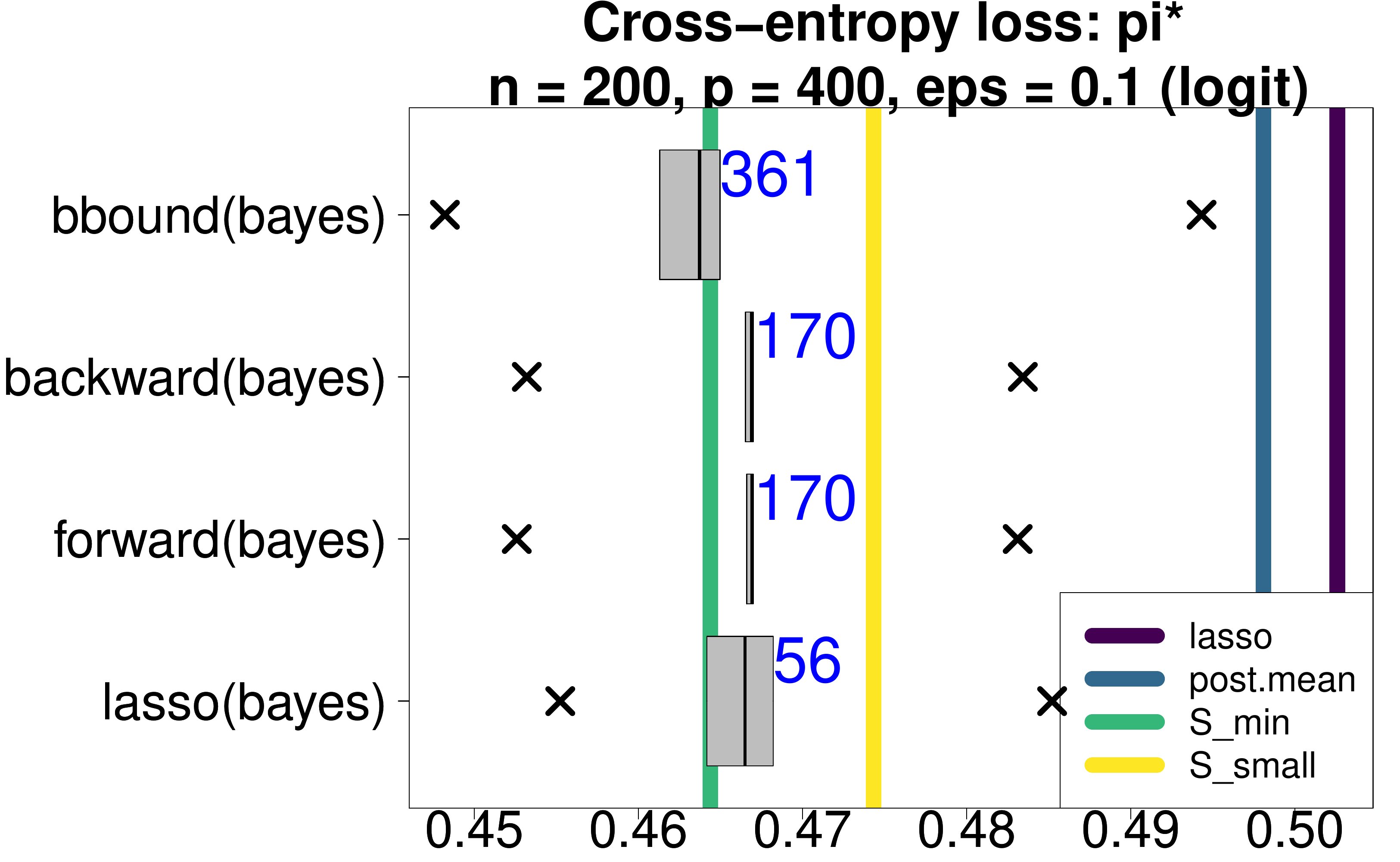} \vspace{2mm}

\includegraphics[width=.49\textwidth]{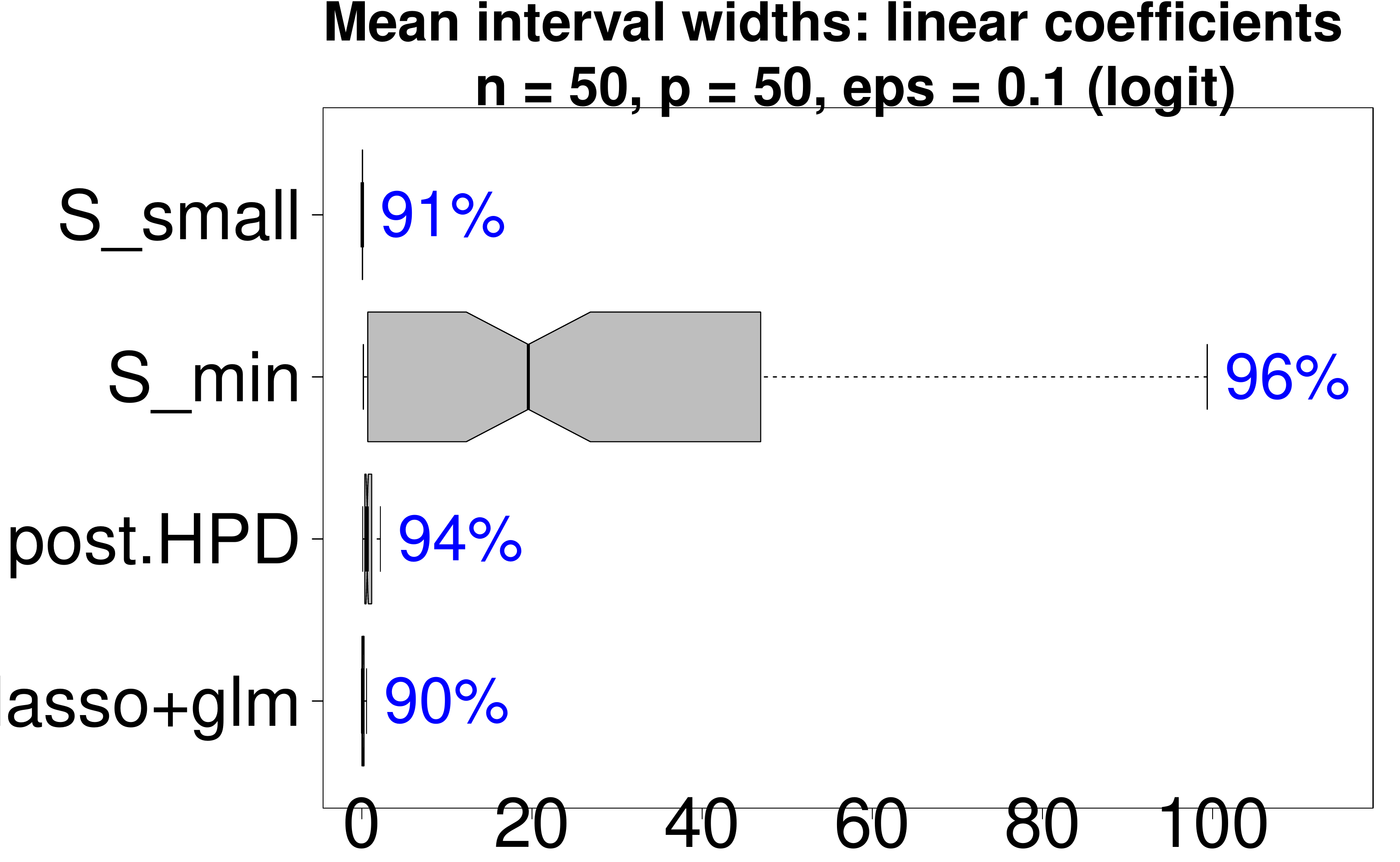}
\includegraphics[width=.49\textwidth]{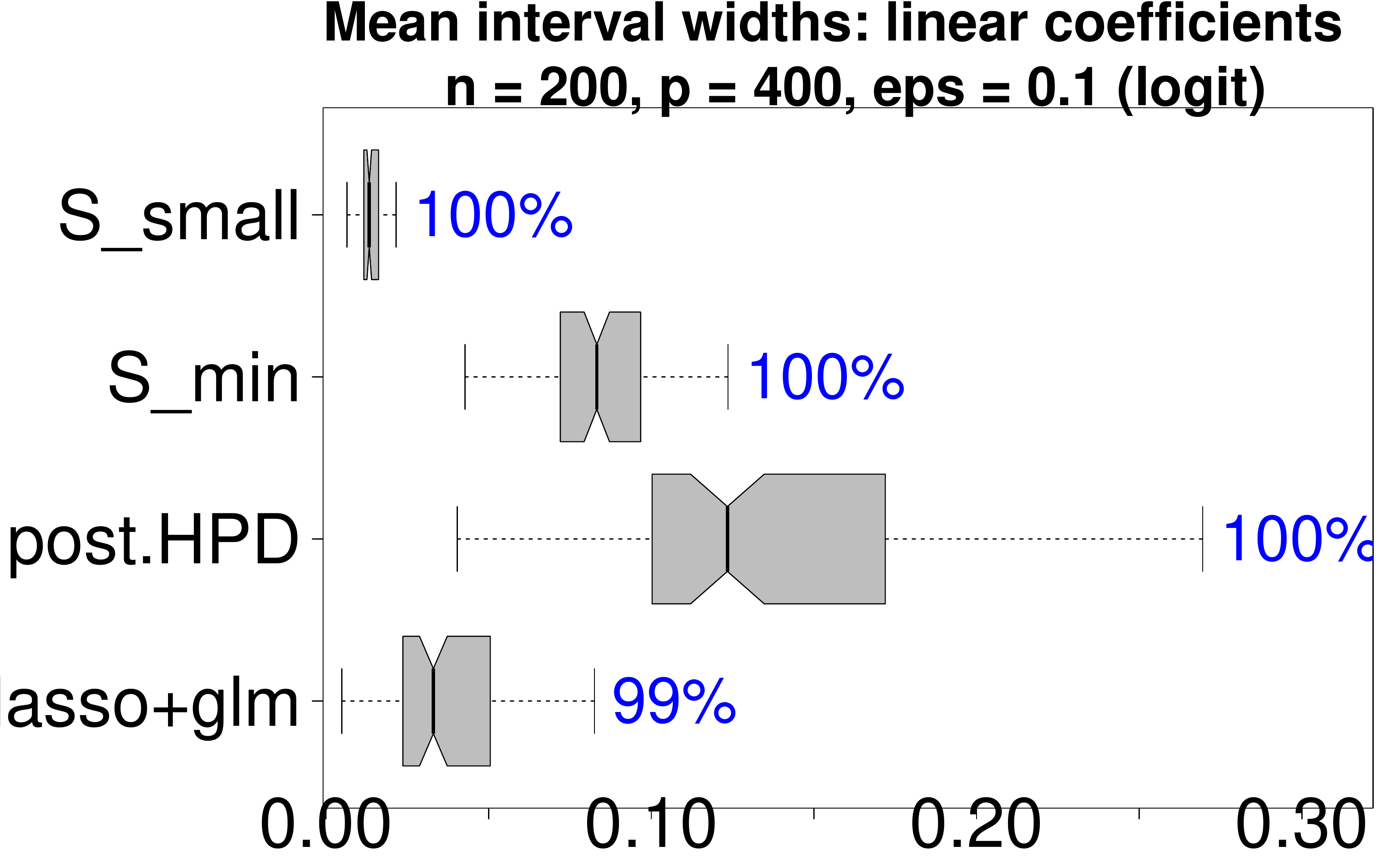}
\end{center}
\caption{ \small
{\bf Top:} cross-entropy loss for $\pi_i^*$. The boxplots summarize the cross-entropy quantiles for the subsets within each acceptable family, while the vertical lines denote cross-entropy of competing methods. The average size of  each acceptable family is annotated. {\bf Bottom:} Mean interval widths (boxplots) with empirical coverage (annotations) for $\bm \beta^*$. 
\label{fig:sims-class-loss}}
\end{figure}

 \section{Simulation study for sensitivity analysis}\label{sims-app}

We revisit the comparisons in Figure~\ref{fig:sims-pred} to consider sensitivities to $\varepsilon$ and $\{\bm{\tilde x}_i\}_{i=1}^{\tilde n}$. In Figure~\ref{fig:sims-pred-eps}, we report results for the acceptable families with  $\varepsilon = 0.01$ and $\varepsilon = 0.2$; the frequentist estimators and \texttt{post.mean} are unchanged from Figure~\ref{fig:sims-pred}.  Compared also to the analogous case in Figure~\ref{fig:sims-pred} with $\varepsilon = 0.1$, we see the expected ordering in the cardinalities of the acceptable families: larger values of $\varepsilon$ produce a more stringent criterion and therefore fewer acceptable subsets. Notably, the main results are not sensitive to the choice of $\varepsilon \in \{0.01, 0.1, 0.2\}$, and $\mathcal{S}_{small}$ performs exceptionally well in all cases.

\begin{figure}[h!]
\begin{center}

\includegraphics[width=.49\textwidth]{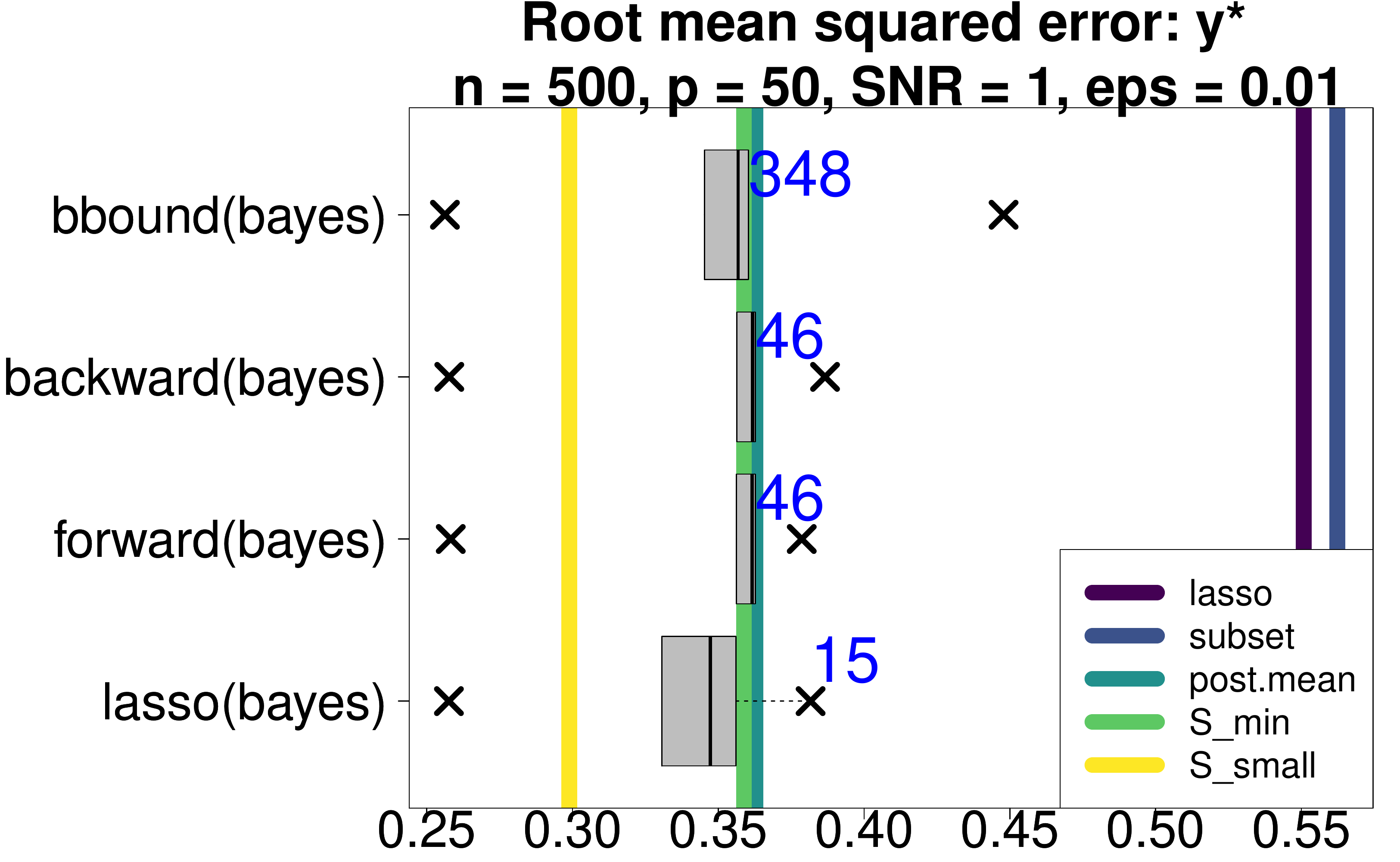}
\includegraphics[width=.49\textwidth]{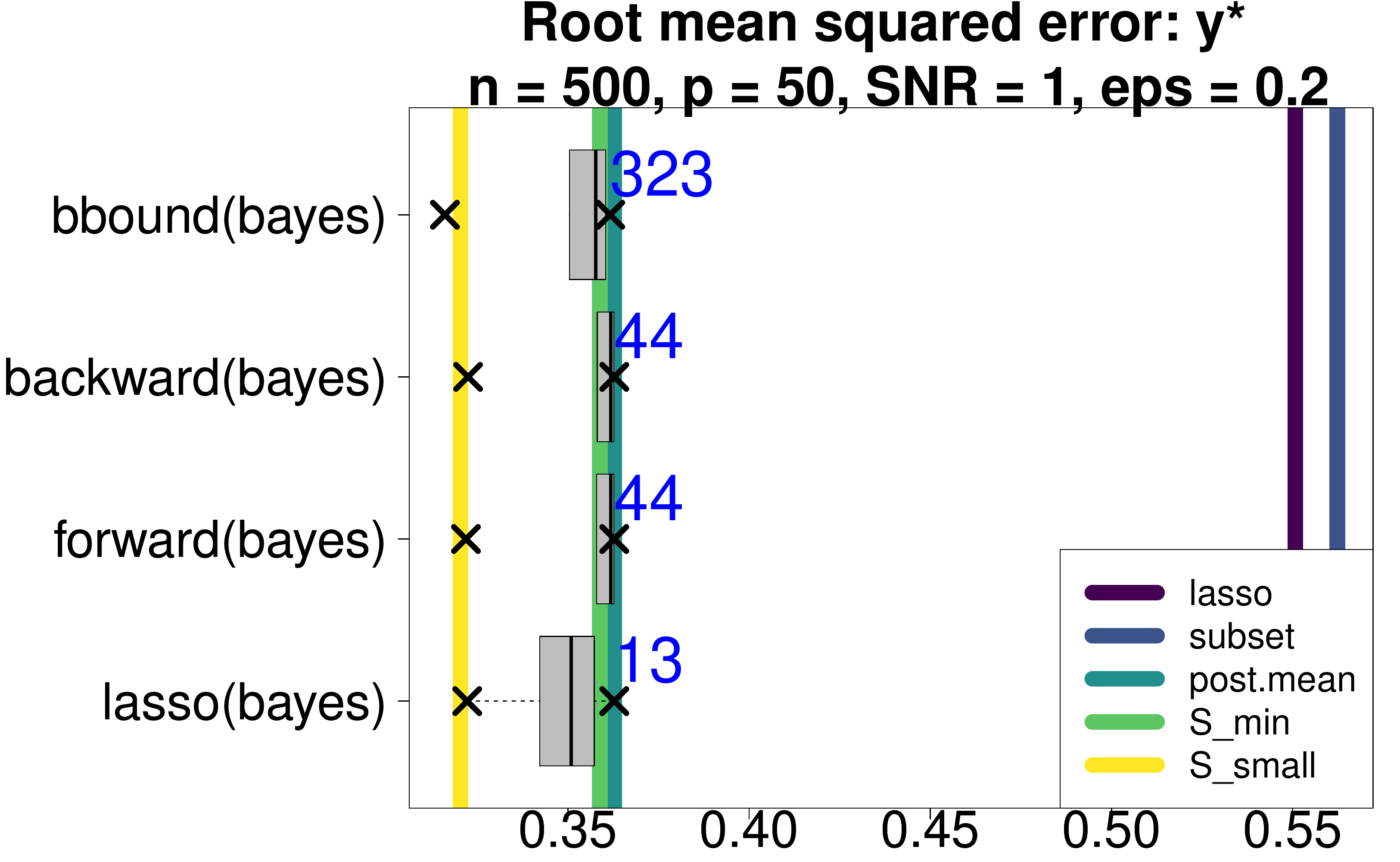}

\end{center}
\caption{ \small
Root mean squared errors (RMSEs) for predicting $y^*$. The boxplots summarize the RMSE quantiles for the subsets within each acceptable family; here, the acceptable families use $\varepsilon = 0.01$ (left) and $\varepsilon = 0.20$ (right). The vertical lines denote RMSEs of competing methods and the average size of  each acceptable family is annotated. The results suggest only minor sensitivity to the choice of $\varepsilon$. 
\label{fig:sims-pred-eps}}
\end{figure}

Next, we again modify the comparisons in Figure~\ref{fig:sims-pred} to evaluate prediction at a \emph{new} collection of covariates $\{\bm{\tilde x}_i\}_{i=1}^{\tilde n}$. We allow variations in the covariate data-generating process---either correlated standard normals from Section~\ref{sims} or iid Uniform(0,1)---and whether the observed covariates $\{\bm{ x}_i\}_{i=1}^{ n}$ follow the same data-generating process as the target covariates $\{\bm{\tilde x}_i\}_{i=1}^{\tilde n}$. These configurations generate four scenarios, all with $\{\bm{\tilde x}_i\}_{i=1}^{\tilde n} \ne \{\bm{ x}_i\}_{i=1}^{n}$. The competing acceptable families are each generated using this choice of $\{\bm{\tilde x}_i\}_{i=1}^{\tilde n}$ and the point predictions are evaluated for $y^*(\bm{\tilde x}_i) = \bm{\tilde x}_i '\bm \beta^*$ for $i=1,\ldots, \tilde n = 1000$. The results are in Figure~\ref{fig:sims-pred-xtilde}.  Naturally, performance worsens across the board when observed and target covariates differ in either their values ($\{\bm{x}_i\}_{i=1}^{n} \ne \{\bm{\tilde x}_i\}_{i=1}^{\tilde n}$) or their distributions. When the target covariates are iid uniform (top right, bottom left), all methods actually perform  better---likely due to the light-tailed (bounded) and independent covariate values. Yet perhaps most remarkably, the relative performance among the methods remains consistent, while $\mathcal{S}_{small}$ outperforms all competitors in all but one scenario---and decisively outperforms the frequentist methods in all cases.

\begin{figure}[h!]
\begin{center}

\includegraphics[width=.49\textwidth]{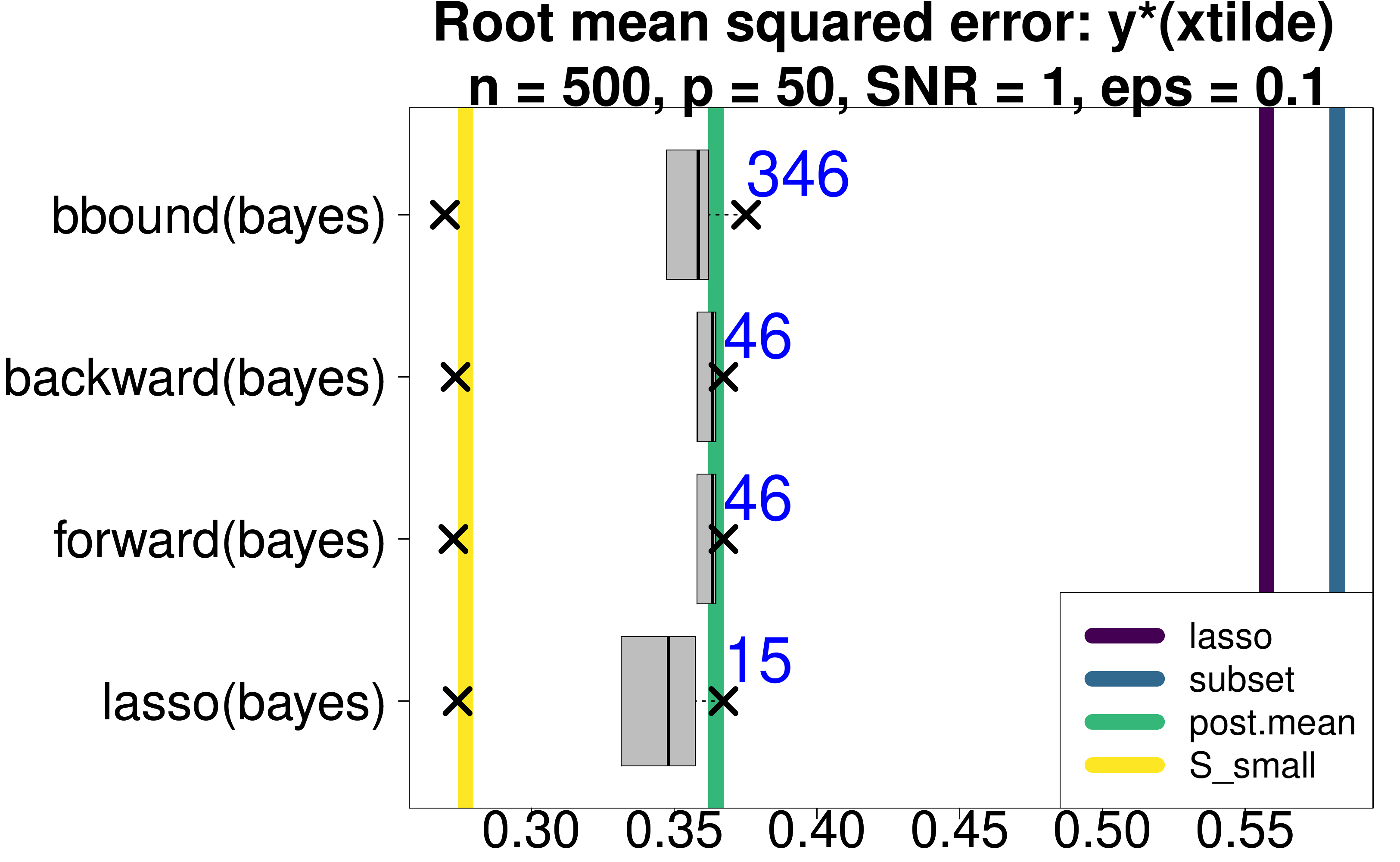}
\includegraphics[width=.49\textwidth]{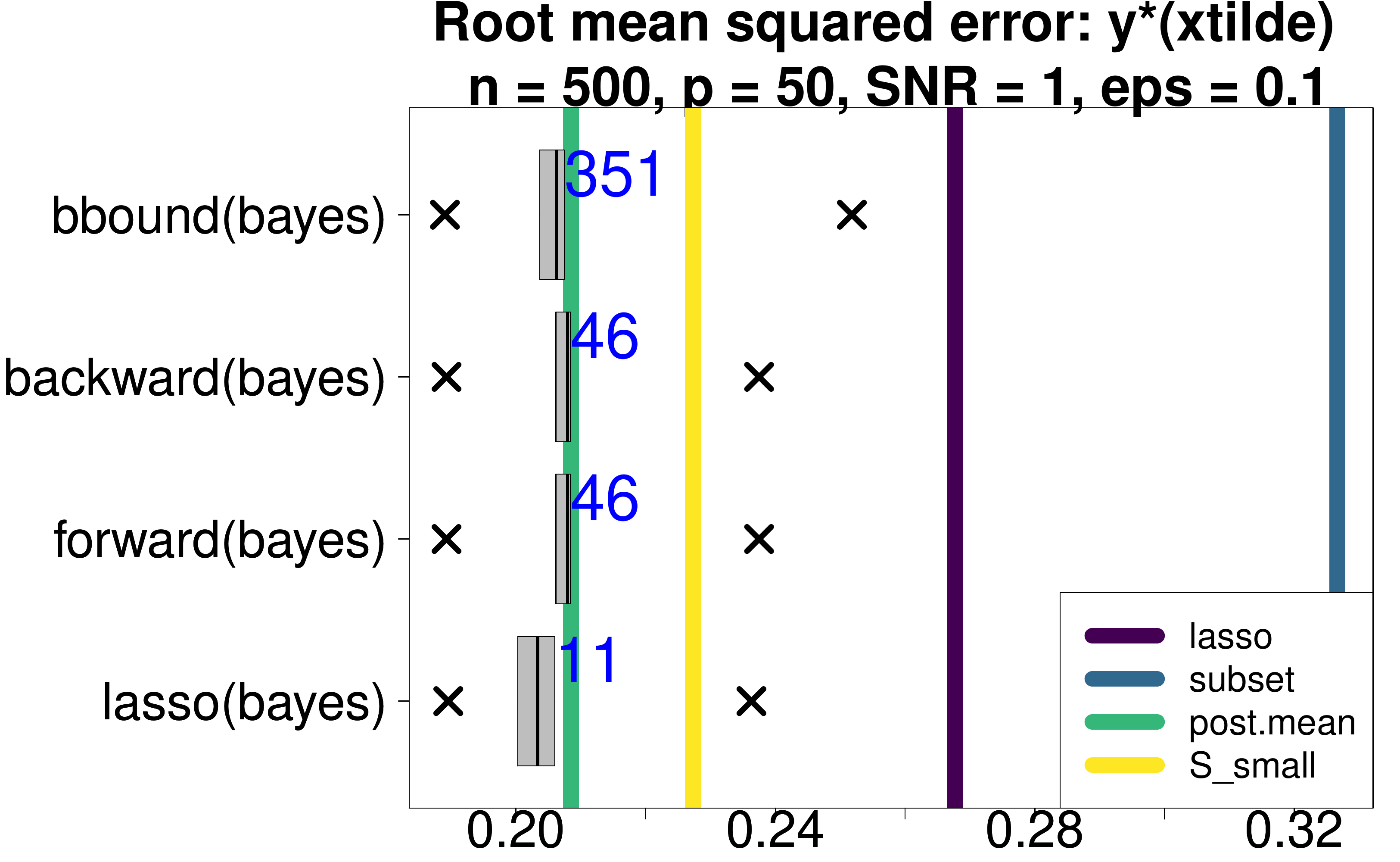} \vspace{2mm}

\includegraphics[width=.49\textwidth]{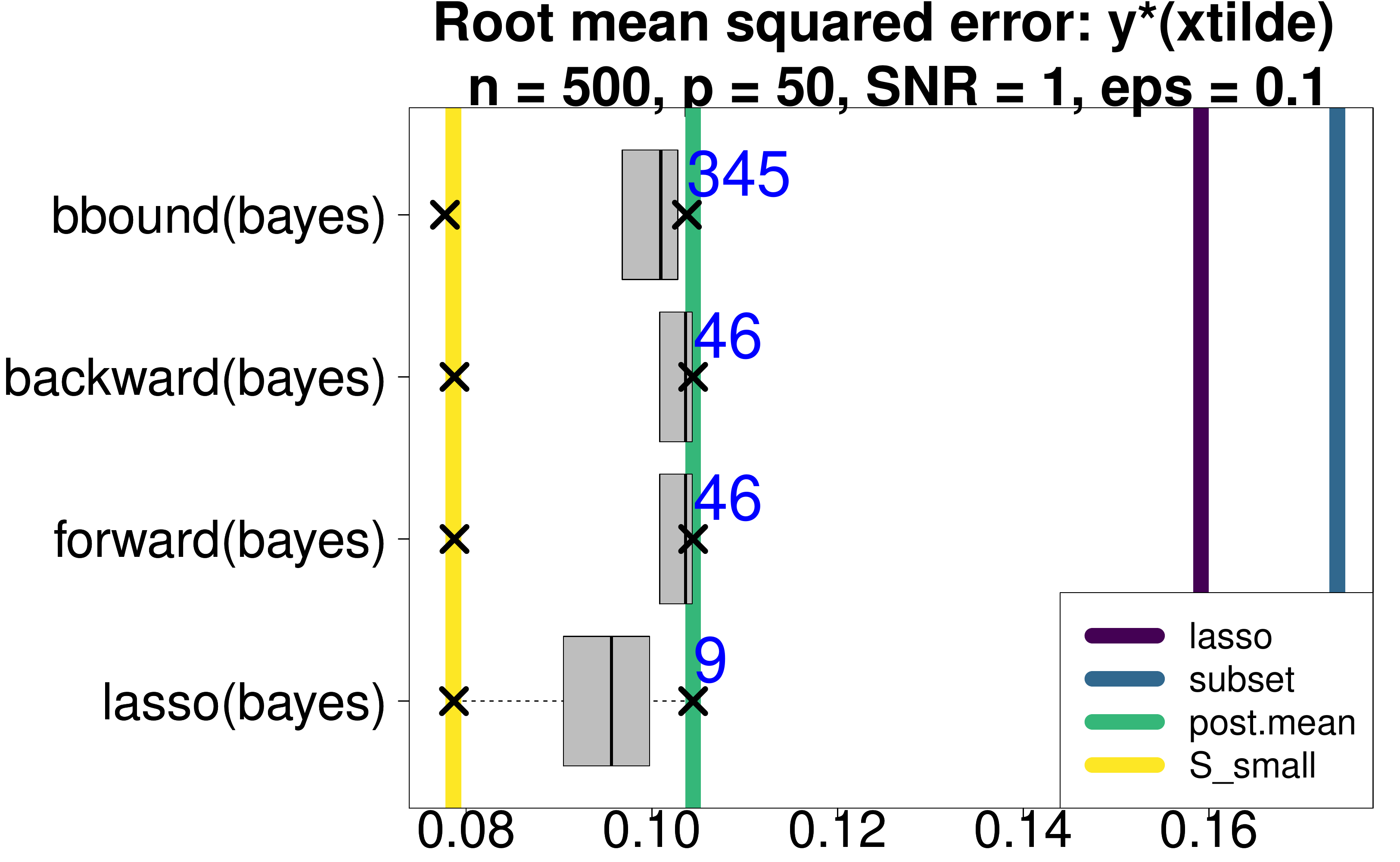}
\includegraphics[width=.49\textwidth]{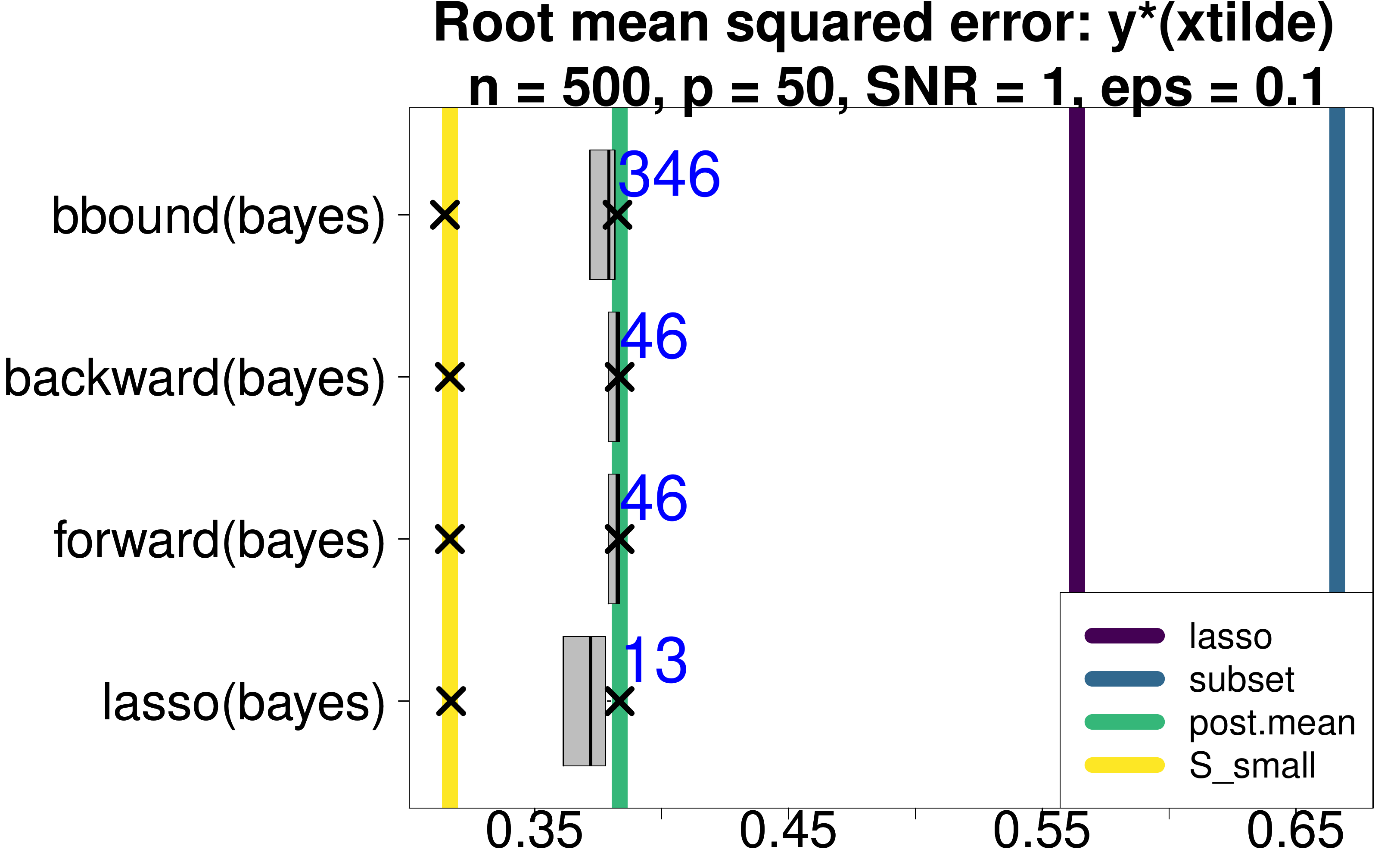}
\end{center}
\caption{ \small
Root mean squared errors (RMSEs) for predicting $y^*(\bm{\tilde x}_i) = \bm{\tilde x}_i' \bm \beta^*$ for covariates $\{\bm{\tilde x}_i\}_{i=1}^{\tilde n =1000}$  when the observed covariates $\{\bm{x}_i\}_{i=1}^n$ are correlated standard normals (top) or iid uniforms (bottom) and when the target covariates $\{\bm{\tilde x}_i\}_{i=1}^{\tilde n =1000}$ have the same distribution (left) or a different distribution (right) from $\{\bm{x}_i\}_{i=1}^n$. The boxplots summarize the RMSE quantiles for the subsets within each acceptable family. The vertical lines denote RMSEs of competing methods and the average size of  each acceptable family is annotated. 
\label{fig:sims-pred-xtilde}}
\end{figure}

\section{Subset selection for classifying at-risk students}\label{app-class}
We now focus the analysis on at-risk students via the functional $h(\tilde y_i) = \mathbb{I}\{\tilde y_i  \ge \tau_{0.1}\}$ under cross-entropy loss, where $\tau_{0.1}$ is the 0.1-quantile of the reading scores. The posterior predictive variables $\tilde y_i$ derive from the same model $\mathcal{M}$ as in Section~\ref{app}, which eliminates the need for additional model specification, fitting, and diagnostics. However, it is also possible to fit a  separate logistic regression model to the empirical functionals $h(y_i)$, which we do for the logistic adaptive lasso competitor.

It is most informative to study how the acceptable families $\mathbb{A}_{0, 0.1}$
differ for targeted classification relative to prediction. The comparative distributions of $\widetilde{{D}}_{\mathcal{S}_{min},\mathcal{S}}^{out}$ (Figures~\ref{fig:d-lm}~and~\ref{fig:d-logit}) clearly show that classification  admits smaller subsets capable  of matching the performance of $\mathcal{S}_{min}$  within 1\%.

\begin{figure}[h!]
\begin{center}
\includegraphics[width=.6\textwidth]{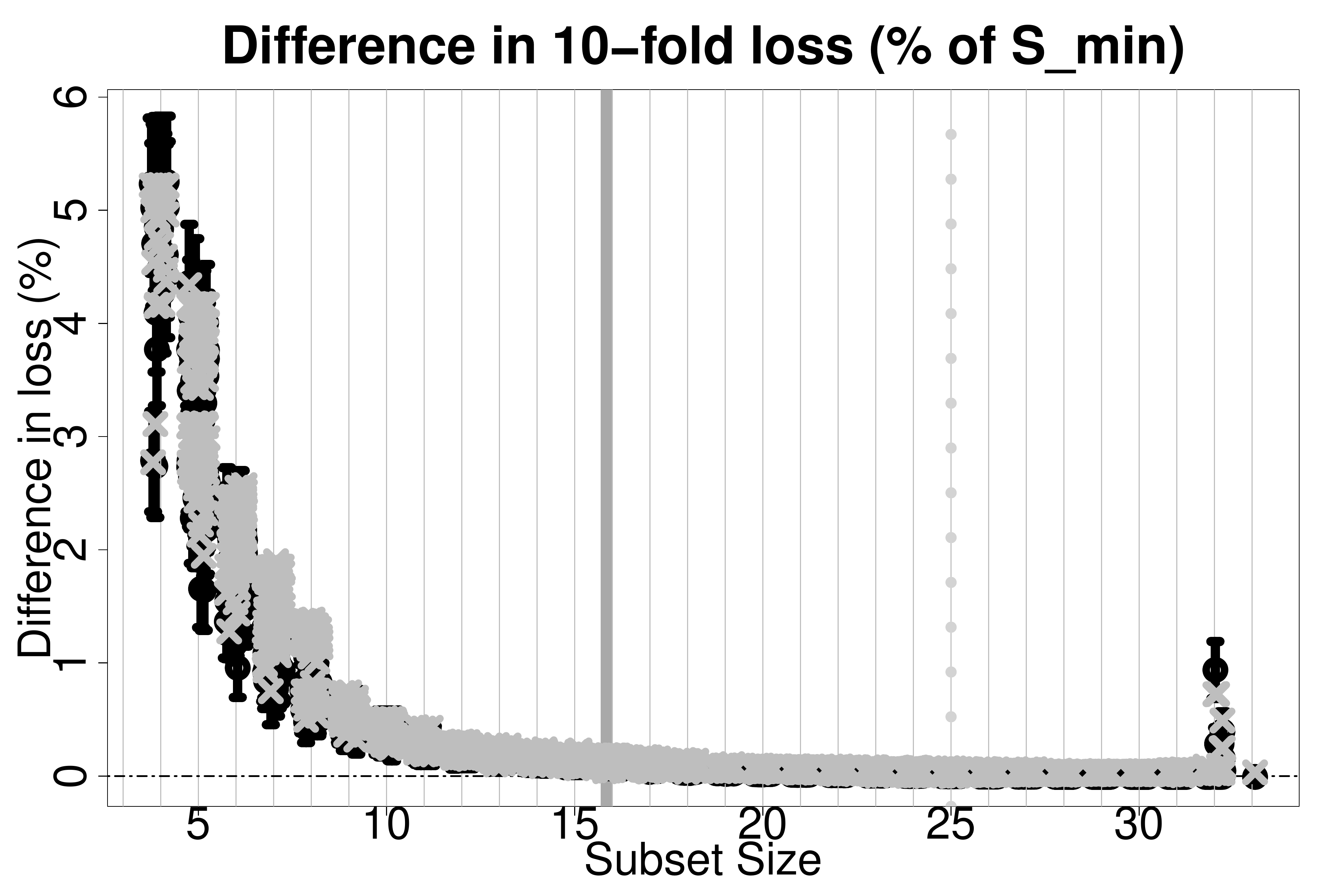}
\end{center}
\caption{ \small
The 80\% intervals (bars) and expected values (circles) for $\widetilde{{D}}_{\mathcal{S}_{min}, \mathcal{S}}^{out}$ with ${{D}}_{\mathcal{S}, \mathcal{S}_{min}}^{out}$ 
(x-marks) under cross-entropy for each subset size $\vert \mathcal{S}\vert$ with $\mathcal{S} \in \mathbb{S}$. We annotate $\mathcal{S}_{min}$ (dashed gray line) and $\mathcal{S}_{small}$ (solid gray line) 
 and jitter the subset sizes for clarity of presentation.  
\label{fig:d-logit}}
\end{figure}

 The variable importance metrics (Figures~\ref{fig:vi-logit}~and~\ref{fig:co-vi-logit}) identify the same keystone covariates as in the prediction setting. However, although there are many more acceptable subsets for classification ($\vert \mathbb{A}_{0, 0.1}\vert = 1547$) than for prediction ($\vert \mathbb{A}_{0, 0.1}\vert = 1183$),  each $\mbox{VI}_{\rm incl}(j)$ is generally much smaller for classification. 
Indeed, the smallest acceptable subset for classification is smaller, $\vert \mathcal{S}_{small}\vert = 16$, and is accompanied by 21 acceptable subsets of this size. In aggregate, these results suggest that near-optimal linear classification of at-risk students is achievable for a broader variety of smaller subsets of covariates.

\begin{figure}[h!]
\begin{center}
\includegraphics[width=.75\textwidth]{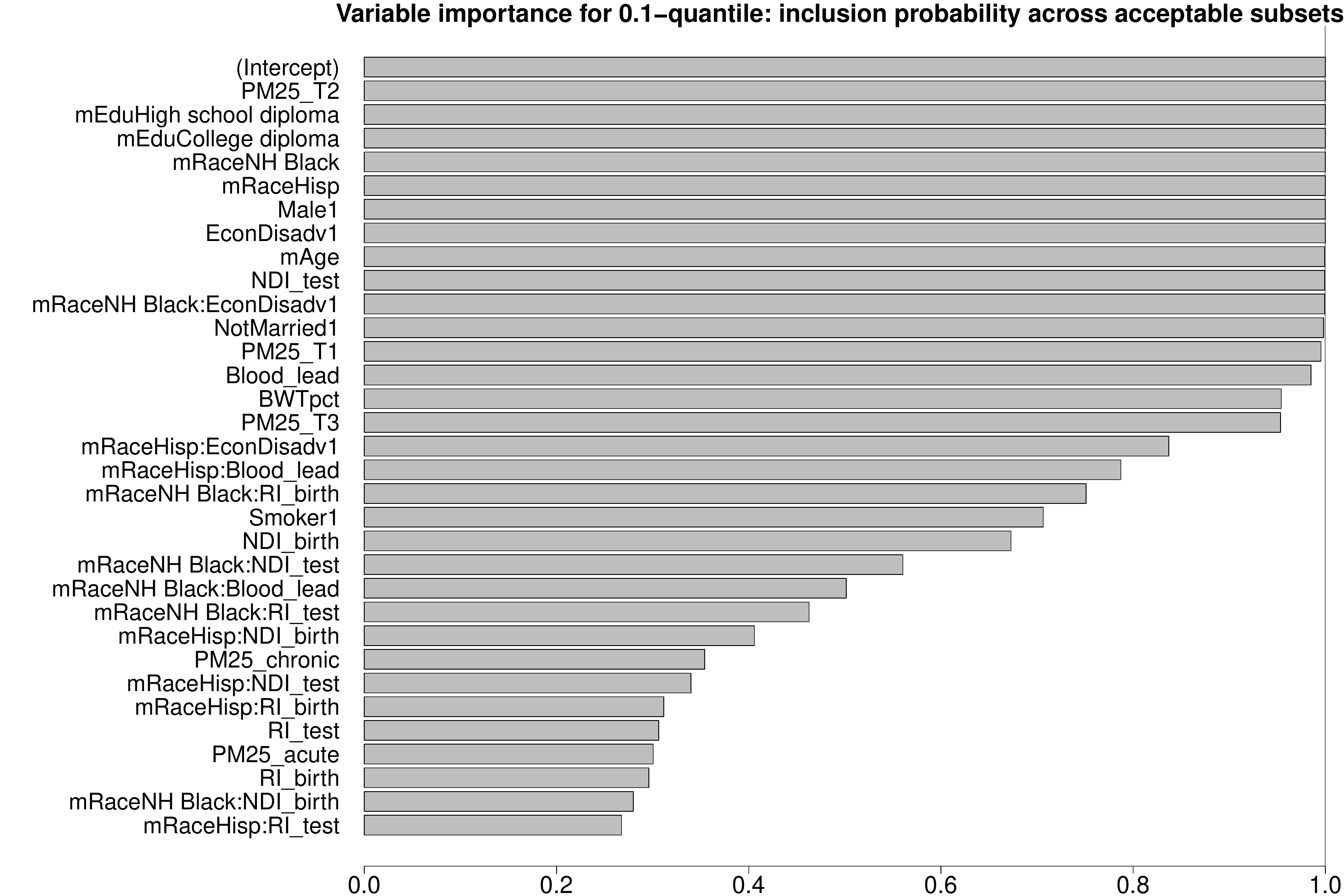}
\end{center}
\caption{ \small
Variable importance $\mbox{VI}_{\rm incl}(j)$ for classification of the 0.1-quantile ($m_k = 100$). 
\label{fig:vi-logit}}
\end{figure} 

Lastly, the selected covariates from $\mathcal{S}_{small}$ are plotted with 90\% intervals in Figure~\ref{fig:coef-logit}. Compared to the prediction version, $\mathcal{S}_{small}$ for classification replaces 1st trimester  $\mbox{PM}_{2.5}$ exposure with acute $\mbox{PM}_{2.5}$ exposure  and omits the interactions \texttt{Hisp} $\times$ \texttt{Blood\_lead}  and \texttt{Hisp} $\times$ \texttt{EconDisadv}, but is otherwise the same. The posterior expectations under $\mathcal{M}$ are excluded because they are not targeted for prediction of $h(\tilde y_i)$ and therefore are not directly comparable. However, there is some disagreement between $\mathcal{S}_{small}$ and the adaptive lasso, the latter of which excludes several keystone covariates and again suffers from inconsistency between the point and interval estimates.


\begin{figure}[h!]
\begin{center}
\includegraphics[width=.9\textwidth]{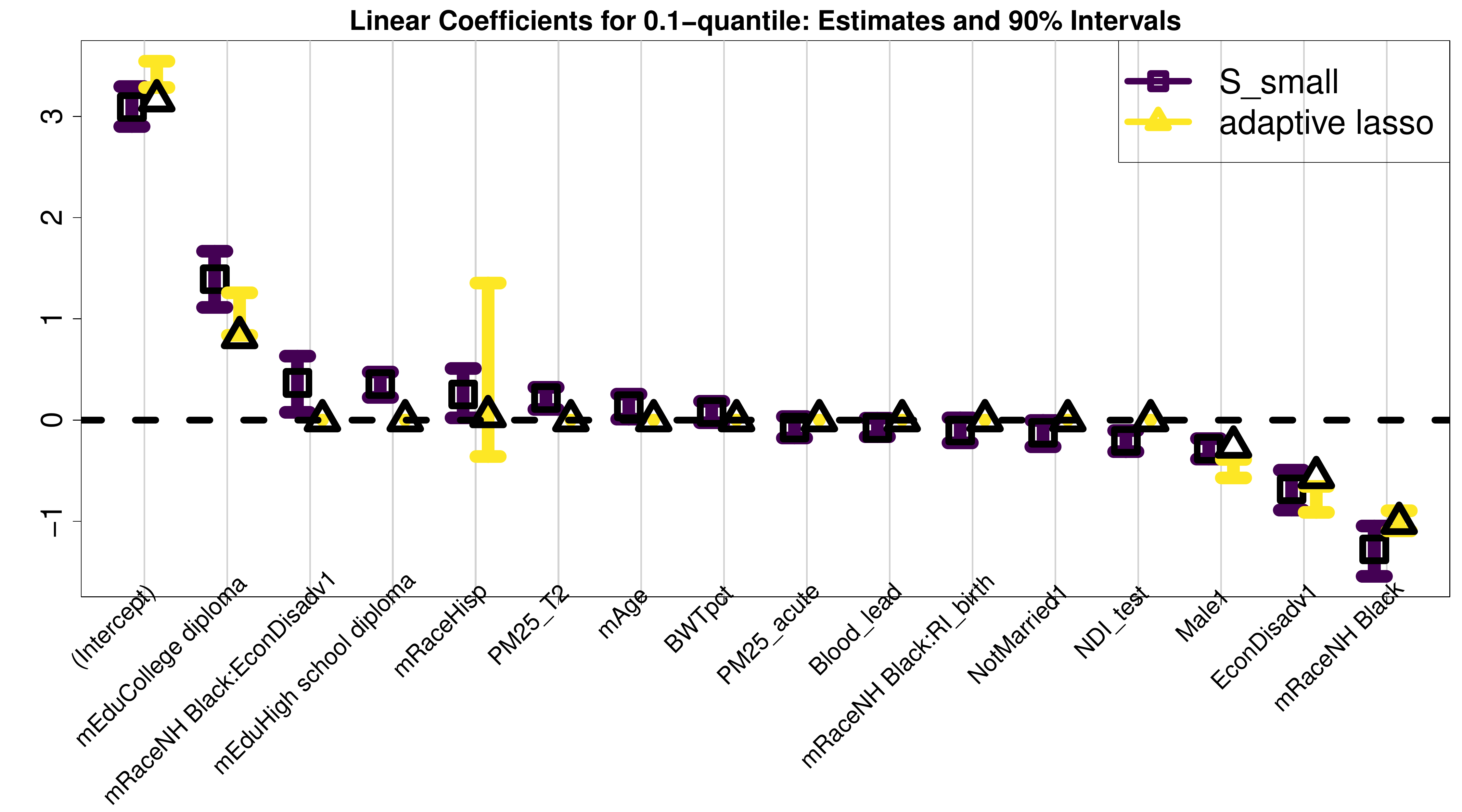}
\end{center}
\caption{ \small
Estimated linear effects and 90\% intervals for the variables in $\mathcal{S}_{small}$ based on the proposed approach  and the adaptive lasso. 
\label{fig:coef-logit}}
\end{figure}

\section{Supporting figures}\label{app-add}
These figures include: pairwise correlations among the covariate (Figure~\ref{fig:corrplot}) and co-variable importance for prediction   (Figure~\ref{fig:co-vi-lm}) and classification (Figure~\ref{fig:co-vi-logit}).

\begin{figure}[h!]
\begin{center}
\includegraphics[width=1\textwidth]{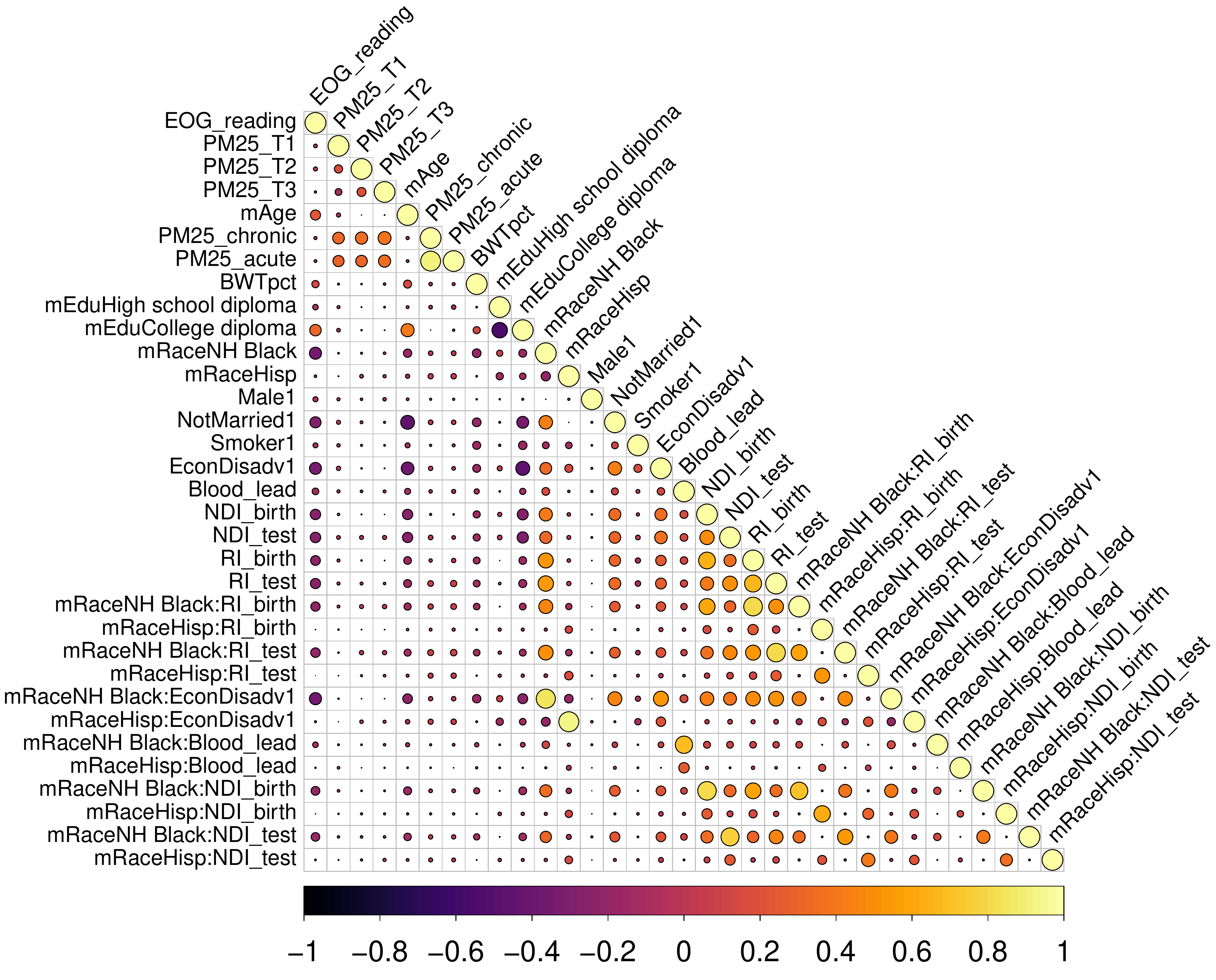}
\end{center}
\caption{ \small
Pairwise correlations among the covariates in the NC education data.
 \label{fig:corrplot}}
\end{figure}

\begin{figure}[h!]
\begin{center}
\includegraphics[width=1\textwidth]{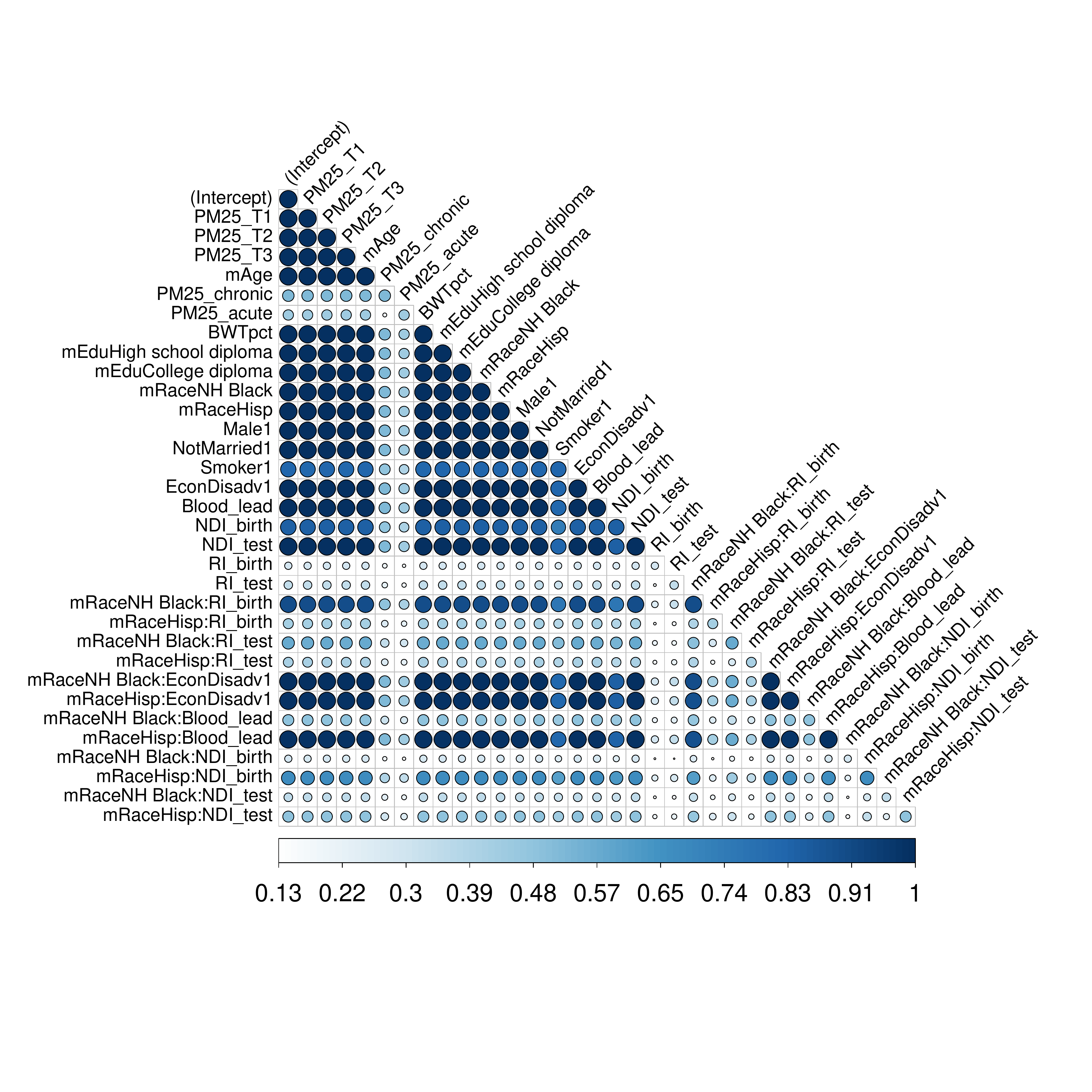}
\end{center}
\caption{ \small
Co-variable importance $\mbox{VI}_{\rm incl}(j, \ell)$ for prediction. For certain pairs of variables (chronic and acute $\mbox{PM}_{2.5}$ exposure; neighborhood deprivation and racial residential isolation both at birth and time of test), it is common for one---but not both---to appear in an acceptable subset. 
 \label{fig:co-vi-lm}}
\end{figure}

\begin{figure}[h!]
\begin{center}
\includegraphics[width=1\textwidth]{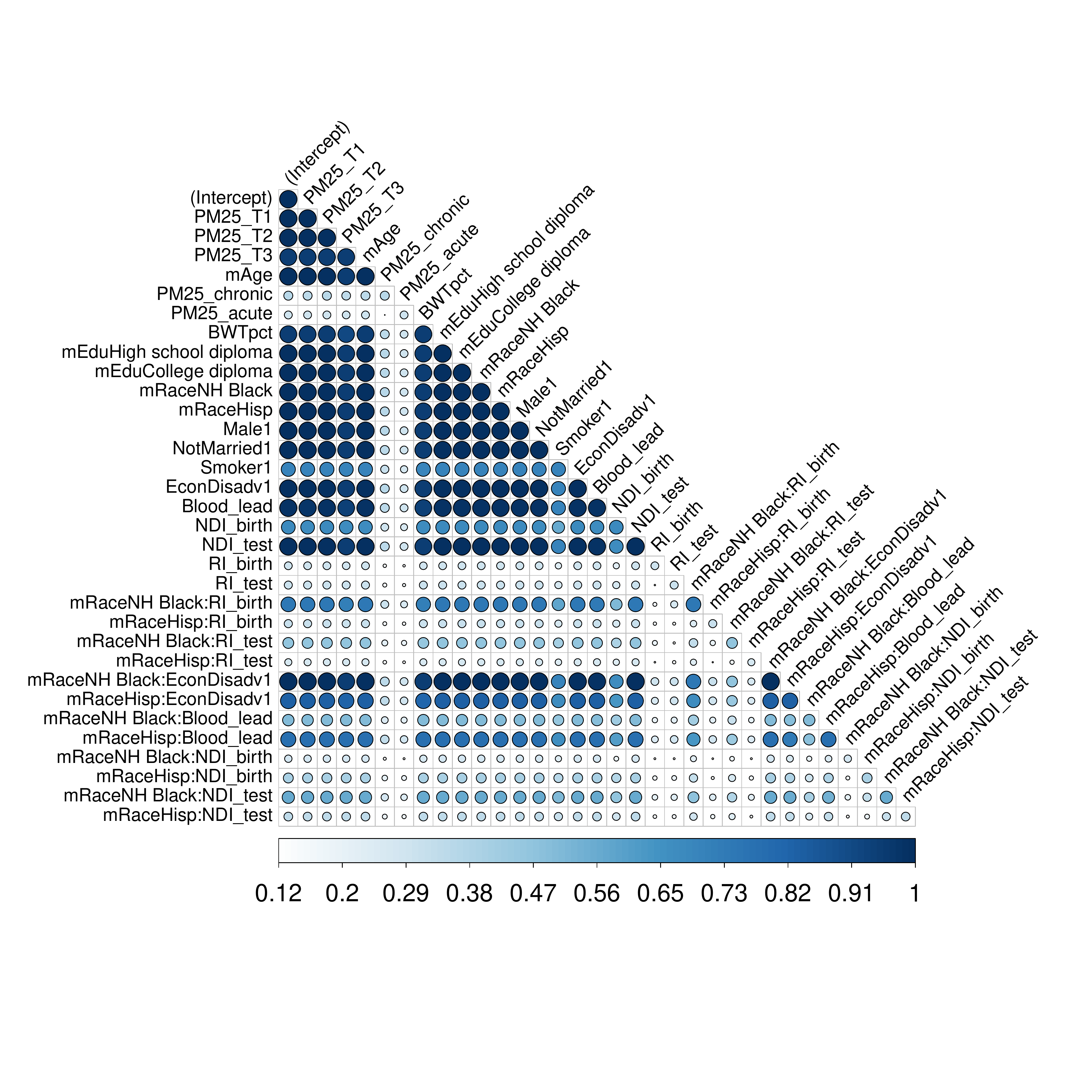}
\end{center}
\caption{ \small
Co-variable importance $\mbox{VI}_{\rm incl}(j, \ell)$ for classification of the 0.1-quantile.
 \label{fig:co-vi-logit}}
\end{figure}

\clearpage
\bibliography{refs.bib}

\begin{thebibliography}{51}
\providecommand{\natexlab}[1]{#1}
\providecommand{\url}[1]{\texttt{#1}}
\expandafter\ifx\csname urlstyle\endcsname\relax
  \providecommand{\doi}[1]{doi: #1}\else
  \providecommand{\doi}{doi: \begingroup \urlstyle{rm}\Url}\fi

\bibitem[Afrabandpey et~al.(2020)Afrabandpey, Peltola, Piironen, Vehtari, and
  Kaski]{Afrabandpey2020}
Homayun Afrabandpey, Tomi Peltola, Juho Piironen, Aki Vehtari, and Samuel
  Kaski.
\newblock {A decision-theoretic approach for model interpretability in Bayesian
  framework}.
\newblock \emph{Machine Learning}, 109\penalty0 (9):\penalty0 1855--1876, 2020.
\newblock ISSN 1573-0565.

\bibitem[Barbieri and Berger(2004)]{barbieri2004optimal}
Maria~Maddalena Barbieri and James~O. Berger.
\newblock {Optimal predictive model selection}.
\newblock \emph{Annals of Statistics}, 32\penalty0 (3):\penalty0 870--897,
  2004.
\newblock ISSN 00905364.
\newblock \doi{10.1214/009053604000000238}.

\bibitem[Bashir et~al.(2019)Bashir, Carvalho, Hahn, and Jones]{Bashir2019}
Amir Bashir, Carlos~M. Carvalho, P.~Richard Hahn, and M.~Beatrix Jones.
\newblock {Post-processing posteriors over precision matrices to produce sparse
  graph estimates}.
\newblock \emph{Bayesian Analysis}, 14\penalty0 (4):\penalty0 1075--1090, 2019.
\newblock ISSN 19316690.
\newblock \doi{10.1214/18-BA1139}.

\bibitem[Bernardo and Smith(2009)]{Bernardo2009}
Jos{\'{e}}~M Bernardo and Adrian F~M Smith.
\newblock \emph{{Bayesian theory}}, volume 405.
\newblock John Wiley and Sons, Inc., 2009.
\newblock ISBN 047031771X.

\bibitem[Bertsimas et~al.(2016)Bertsimas, King, and Mazumder]{Bertsimas2016}
Dimitris Bertsimas, Angela King, and Rahul Mazumder.
\newblock {Best subset selection via a modern optimization lens}.
\newblock \emph{Annals of statistics}, 44\penalty0 (2):\penalty0 813--852,
  2016.
\newblock ISSN 0090-5364.

\bibitem[Bondell and Reich(2012)]{bondell2012consistent}
Howard~D. Bondell and Brian~J. Reich.
\newblock {Consistent high-dimensional Bayesian variable selection via
  penalized credible regions}.
\newblock \emph{Journal of the American Statistical Association}, 107\penalty0
  (500):\penalty0 1610--1624, 2012.
\newblock ISSN 01621459.
\newblock \doi{10.1080/01621459.2012.716344}.

\bibitem[Breiman(2001)]{breiman2001statistical}
Leo Breiman.
\newblock {Statistical Modeling: The Two Cultures (with comments and a
  rejoinder by the author)}.
\newblock \emph{Statistical Science}, 16\penalty0 (3):\penalty0 199--231, 2001.
\newblock ISSN 0883-4237.
\newblock \doi{10.1214/ss/1009213726}.

\bibitem[Carvalho et~al.(2010)Carvalho, Polson, and
  Scott]{carvalho2010horseshoe}
Carlos~M. Carvalho, Nicholas~G. Polson, and James~G. Scott.
\newblock {The horseshoe estimator for sparse signals}.
\newblock \emph{Biometrika}, 97\penalty0 (2):\penalty0 465--480, 2010.
\newblock ISSN 00063444.
\newblock \doi{10.1093/biomet/asq017}.

\bibitem[{Children's Environmental Health
  Initiative}(2020)]{ChildrensEnvironmentalHealthInitiative2020}
{Children's Environmental Health Initiative}.
\newblock {Linked Births, Lead Surveillance, Grade 4 End-Of-Grade (EoG) Scores
  [Data set]}, 2020.
\newblock URL \url{https://doi.org/10.25614/COHORT_2000}.

\bibitem[Datta and Ghosh(2013)]{datta2013asymptotic}
Jyotishka Datta and Jayanta~K Ghosh.
\newblock {Asymptotic properties of Bayes risk for the horseshoe prior}.
\newblock \emph{Bayesian Analysis}, 8\penalty0 (1):\penalty0 111--132, 2013.

\bibitem[Dong and Rudin(2019)]{Dong2019}
Jiayun Dong and Cynthia Rudin.
\newblock {Variable importance clouds: A way to explore variable importance for
  the set of good models}.
\newblock \emph{arXiv preprint arXiv:1901.03209}, 2019.

\bibitem[Fan and Lv(2008)]{Fan2008}
Jianqing Fan and Jinchi Lv.
\newblock {Sure independence screening for ultrahigh dimensional feature
  space}.
\newblock \emph{Journal of the Royal Statistical Society. Series B: Statistical
  Methodology}, 70\penalty0 (5):\penalty0 849--911, 2008.
\newblock ISSN 13697412.
\newblock \doi{10.1111/j.1467-9868.2008.00674.x}.

\bibitem[Fan and Lv(2010)]{Fan2010}
Jianqing Fan and Jinchi Lv.
\newblock {A selective overview of variable selection in high dimensional
  feature space}.
\newblock \emph{Statistica Sinica}, 20\penalty0 (1):\penalty0 101--148, 2010.
\newblock ISSN 10170405.

\bibitem[Furnival and Wilson(2000)]{Furnival2000}
George~M Furnival and Robert~W Wilson.
\newblock {Regressions by leaps and bounds}.
\newblock \emph{Technometrics}, 42\penalty0 (1):\penalty0 69--79, 2000.
\newblock ISSN 0040-1706.

\bibitem[Garc{\'{i}}a-Donato and Paulo(2021)]{GarciaDonato2021}
Gonzalo Garc{\'{i}}a-Donato and Rui Paulo.
\newblock {Variable selection in the presence of factors: a model selection
  perspective}.
\newblock \emph{Journal of the American Statistical Association}, pages 1--27,
  2021.
\newblock ISSN 0162-1459.

\bibitem[Gatu and Kontoghiorghes(2006)]{Gatu2006}
Cristian Gatu and Erricos~John Kontoghiorghes.
\newblock {Branch-and-bound algorithms for computing the best-subset regression
  models}.
\newblock \emph{Journal of Computational and Graphical Statistics}, 15\penalty0
  (1):\penalty0 139--156, 2006.
\newblock ISSN 1061-8600.

\bibitem[Gelfand et~al.(1992)Gelfand, Dey, and Chang]{gelfand1992model}
A~E Gelfand, D~K Dey, and H~Chang.
\newblock {Model determination using predictive distributions, with
  implementation via sampling-based methods (with discussion)}.
\newblock \emph{Bayesian Statistics 4}, 4:\penalty0 147--167, 1992.
\newblock ISSN 01621459.

\bibitem[Goodrich et~al.(2018)Goodrich, Gabry, Ali, and Brilleman]{rstanarm}
Ben Goodrich, Jonah Gabry, Imad Ali, and Sam Brilleman.
\newblock {rstanarm: Bayesian applied regression modeling via Stan.}, 2018.
\newblock URL \url{http://mc-stan.org/}.

\bibitem[Goutis and Robert(1998)]{goutis1998model}
Constantinos Goutis and Christian~P. Robert.
\newblock {Model choice in generalised linear models: A Bayesian approach via
  Kullback-Leibler projections}.
\newblock \emph{Biometrika}, 85\penalty0 (1):\penalty0 29--37, 1998.
\newblock ISSN 00063444.
\newblock \doi{10.1093/biomet/85.1.29}.

\bibitem[Hahn and Carvalho(2015)]{hahn2015decoupling}
P.~Richard Hahn and Carlos~M. Carvalho.
\newblock {Decoupling shrinkage and selection in bayesian linear models: A
  posterior summary perspective}.
\newblock \emph{Journal of the American Statistical Association}, 110\penalty0
  (509):\penalty0 435--448, 2015.
\newblock ISSN 1537274X.
\newblock \doi{10.1080/01621459.2014.993077}.

\bibitem[Hahn et~al.(2019)Hahn, He, and Lopes]{Hahn2019}
P.~Richard Hahn, Jingyu He, and Hedibert~F. Lopes.
\newblock {Efficient Sampling for Gaussian Linear Regression With Arbitrary
  Priors}.
\newblock \emph{Journal of Computational and Graphical Statistics}, 28\penalty0
  (1):\penalty0 142--154, 2019.
\newblock ISSN 15372715.
\newblock \doi{10.1080/10618600.2018.1482762}.

\bibitem[Hastie et~al.(2009)Hastie, Tibshirani, and
  Friedman]{hastie2009elements}
Trevor Hastie, Robert Tibshirani, and Jerome Friedman.
\newblock \emph{{The Elements of Statistical Learning}}, volume~2.
\newblock Springer, 2009.
\newblock ISBN 0387848584.

\bibitem[Hastie et~al.(2020)Hastie, Tibshirani, and Tibshirani]{Hastie2020}
Trevor Hastie, Robert Tibshirani, and Ryan Tibshirani.
\newblock {Best Subset, Forward Stepwise or Lasso? Analysis and Recommendations
  Based on Extensive Comparisons}.
\newblock \emph{Statistical Science}, 35\penalty0 (4):\penalty0 579--592, 2020.
\newblock ISSN 0883-4237.

\bibitem[Hosmer et~al.(1989)Hosmer, Jovanovic, and Lemeshow]{Hosmer1989}
David~W Hosmer, Borko Jovanovic, and Stanley Lemeshow.
\newblock {Best subsets logistic regression}.
\newblock \emph{Biometrics}, pages 1265--1270, 1989.
\newblock ISSN 0006-341X.

\bibitem[Huber et~al.(2020)Huber, Koop, and Onorante]{Huber2020}
Florian Huber, Gary Koop, and Luca Onorante.
\newblock {Inducing Sparsity and Shrinkage in Time-Varying Parameter Models}.
\newblock \emph{Journal of Business and Economic Statistics}, pages 1--48,
  2020.
\newblock ISSN 15372707.
\newblock \doi{10.1080/07350015.2020.1713796}.

\bibitem[Huggins et~al.(2018)Huggins, Campbell, Kasprzak, and
  Broderick]{huggins2018practical}
Jonathan~H. Huggins, Trevor Campbell, Miko{\l}aj Kasprzak, and Tamara
  Broderick.
\newblock {Practical bounds on the error of Bayesian posterior approximations:
  A nonasymptotic approach}.
\newblock \emph{arXiv preprint arXiv:1809.09505}, 2018.
\newblock URL \url{http://arxiv.org/abs/1809.09505}.

\bibitem[Ishwaran and Rao(2005)]{ishwaran2005spike}
Hemant Ishwaran and J.~Sunil Rao.
\newblock {Spike and slab variable selection: Frequentist and Bayesian
  strategies}.
\newblock \emph{Annals of Statistics}, 33\penalty0 (2):\penalty0 730--773,
  2005.
\newblock ISSN 00905364.
\newblock \doi{10.1214/009053604000001147}.

\bibitem[Jiang et~al.(2008)Jiang, Rao, Gu, and Nguyen]{Jiang2008}
Jiming Jiang, J~Sunil Rao, Zhonghua Gu, and Thuan Nguyen.
\newblock {Fence methods for mixed model selection}.
\newblock \emph{The Annals of Statistics}, 36\penalty0 (4):\penalty0
  1669--1692, 2008.
\newblock ISSN 0090-5364.

\bibitem[Jin and Goh(2020)]{Jin2020}
Shiqiang Jin and Gyuhyeong Goh.
\newblock {Bayesian selection of best subsets via hybrid search}.
\newblock \emph{Computational Statistics}, pages 1--17, 2020.
\newblock ISSN 0943-4062.

\bibitem[Kissel and Mentch(2021)]{Kissel2021}
Nicholas Kissel and Lucas Mentch.
\newblock {Forward Stability and Model Path Selection}.
\newblock \emph{arXiv preprint arXiv:2103.03462}, 2021.

\bibitem[Kowal(2021)]{Kowal2020target}
Daniel~R. Kowal.
\newblock {Fast, Optimal, and Targeted Predictions using Parametrized Decision
  Analysis}.
\newblock \emph{Journal of the American Statistical Association}, 2021.
\newblock \doi{10.1080/01621459.2021.1891926}.

\bibitem[Kowal and Bourgeois(2020)]{kowal2019bayesianfosr}
Daniel~R Kowal and Daniel~C Bourgeois.
\newblock {Bayesian Function-on-Scalars Regression for High-Dimensional Data}.
\newblock \emph{Journal of Computational and Graphical Statistics}, pages
  1--10, 2020.
\newblock ISSN 1061-8600.

\bibitem[Kowal et~al.(2020)Kowal, Bravo, Leong, Griffin, Ensor, and
  Miranda]{KowalPRIME2020}
Daniel~R. Kowal, Mercedes Bravo, Henry Leong, Robert~J. Griffin, Katherine~B.
  Ensor, and Marie~Lynn Miranda.
\newblock {Bayesian Variable Selection for Understanding Mixtures in
  Environmental Exposures}.
\newblock 2020.

\bibitem[Lei(2019)]{Lei2019}
Jing Lei.
\newblock {Cross-Validation With Confidence}.
\newblock \emph{Journal of the American Statistical Association}, 0\penalty0
  (0):\penalty0 1--53, 2019.
\newblock ISSN 0162-1459.
\newblock \doi{10.1080/01621459.2019.1672556}.
\newblock URL \url{https://doi.org/10.1080/01621459.2019.1672556}.

\bibitem[Lei et~al.(2018)Lei, G'Sell, Rinaldo, Tibshirani, and
  Wasserman]{Lei2018}
Jing Lei, Max G'Sell, Alessandro Rinaldo, Ryan~J Tibshirani, and Larry
  Wasserman.
\newblock {Distribution-free predictive inference for regression}.
\newblock \emph{Journal of the American Statistical Association}, 113\penalty0
  (523):\penalty0 1094--1111, 2018.
\newblock ISSN 0162-1459.

\bibitem[Liang et~al.(2013)Liang, Song, and Yu]{Liang2013}
Faming Liang, Qifan Song, and Kai Yu.
\newblock {Bayesian subset modeling for high-dimensional generalized linear
  models}.
\newblock \emph{Journal of the American Statistical Association}, 108\penalty0
  (502):\penalty0 589--606, 2013.
\newblock ISSN 0162-1459.

\bibitem[Lindley(1968)]{lindley1968choice}
D.~V. Lindley.
\newblock {The Choice of Variables in Multiple Regression}.
\newblock \emph{Journal of the Royal Statistical Society: Series B
  (Methodological)}, 30\penalty0 (1):\penalty0 31--53, 1968.
\newblock \doi{10.1111/j.2517-6161.1968.tb01505.x}.

\bibitem[Miller(1984)]{Miller1984}
Alan~J Miller.
\newblock {Selection of subsets of regression variables}.
\newblock \emph{Journal of the Royal Statistical Society: Series A (General)},
  147\penalty0 (3):\penalty0 389--410, 1984.
\newblock ISSN 0035-9238.

\bibitem[Nott and Leng(2010)]{nott2010bayesian}
David~J. Nott and Chenlei Leng.
\newblock {Bayesian projection approaches to variable selection in generalized
  linear models}.
\newblock \emph{Computational Statistics and Data Analysis}, 54\penalty0
  (12):\penalty0 3227--3241, 2010.
\newblock ISSN 01679473.
\newblock \doi{10.1016/j.csda.2010.01.036}.

\bibitem[O'Hara and Sillanp{\"{a}}{\"{a}}(2009)]{o2009review}
R.~B. O'Hara and M.~J. Sillanp{\"{a}}{\"{a}}.
\newblock {A review of Bayesian variable selection methods: What, how and
  which}.
\newblock \emph{Bayesian Analysis}, 4\penalty0 (1):\penalty0 85--117, 2009.
\newblock ISSN 19360975.
\newblock \doi{10.1214/09-BA403}.

\bibitem[Piironen et~al.(2020)Piironen, Paasiniemi, and
  Vehtari]{piironen2018projective}
Juho Piironen, Markus Paasiniemi, and Aki Vehtari.
\newblock {Projective inference in high-dimensional problems: Prediction and
  feature selection}.
\newblock \emph{Electronic Journal of Statistics}, 14\penalty0 (1):\penalty0
  2155--2197, 2020.
\newblock ISSN 1935-7524.
\newblock URL \url{http://arxiv.org/abs/1810.02406}.

\bibitem[Polson and Scott(2010)]{polson2010shrink}
Nicholas~G. Polson and James~G. Scott.
\newblock {Shrink Globally, Act Locally: Sparse Bayesian Regularization and
  Prediction}.
\newblock \emph{Bayesian Statistics}, 9780199694:\penalty0 501--538, 2010.
\newblock \doi{10.1093/acprof:oso/9780199694587.003.0017}.

\bibitem[Puelz et~al.(2017)Puelz, Hahn, and Carvalho]{Puelz2017}
David Puelz, P.~Richard Hahn, and Carlos~M. Carvalho.
\newblock {Variable selection in seemingly unrelated regressions with random
  predictors}.
\newblock \emph{Bayesian Analysis}, 12\penalty0 (4):\penalty0 969--989, 2017.
\newblock ISSN 19316690.
\newblock \doi{10.1214/17-BA1053}.

\bibitem[Ribeiro et~al.(2016)Ribeiro, Singh, and Guestrin]{Ribeiro2016}
Marco~Tulio Ribeiro, Sameer Singh, and Carlos Guestrin.
\newblock {"Why should I trust you?" Explaining the predictions of any
  classifier}.
\newblock In \emph{Proceedings of the 22nd ACM SIGKDD international conference
  on knowledge discovery and data mining}, pages 1135--1144, 2016.

\bibitem[Semenova and Rudin(2019)]{Semenova2019}
Lesia Semenova and Cynthia Rudin.
\newblock {A study in Rashomon curves and volumes: A new perspective on
  generalization and model simplicity in machine learning}.
\newblock \emph{arXiv preprint arXiv:1908.01755}, 2019.
\newblock URL \url{http://arxiv.org/abs/1908.01755}.

\bibitem[Tran et~al.(2012)Tran, Nott, and Leng]{tran2012predictive}
Minh~Ngoc Tran, David~J. Nott, and Chenlei Leng.
\newblock {The predictive Lasso}.
\newblock \emph{Statistics and Computing}, 22\penalty0 (5):\penalty0
  1069--1084, 2012.
\newblock ISSN 09603174.
\newblock \doi{10.1007/s11222-011-9279-3}.

\bibitem[Vehtari and Ojanen(2012)]{vehtari2012survey}
Aki Vehtari and Janne Ojanen.
\newblock {A survey of Bayesian predictive methods for model assessment,
  selection and comparison}.
\newblock \emph{Statistics Surveys}, 6\penalty0 (1):\penalty0 142--228, 2012.
\newblock ISSN 19357516.
\newblock \doi{10.1214/12-ss102}.

\bibitem[Woody et~al.(2020)Woody, Carvalho, and Murray]{woody2019model}
Spencer Woody, Carlos~M. Carvalho, and Jared~S. Murray.
\newblock {Model interpretation through lower-dimensional posterior
  summarization}.
\newblock \emph{Journal of Computational and Graphical Statistics}, pages 1--9,
  2020.
\newblock ISSN 1061-8600.
\newblock URL \url{http://arxiv.org/abs/1905.07103}.

\bibitem[Wu et~al.(2007)Wu, Harris, and McAuley]{Wu2007}
Shaohua Wu, T~J Harris, and K~B McAuley.
\newblock {The use of simplified or misspecified models: Linear case}.
\newblock \emph{The Canadian Journal of Chemical Engineering}, 85\penalty0
  (4):\penalty0 386--398, 2007.
\newblock ISSN 0008-4034.

\bibitem[Zhao et~al.(2017)Zhao, Shojaie, and Witten]{Zhao2017}
Sen Zhao, Ali Shojaie, and Daniela Witten.
\newblock {In defense of the indefensible: A very naive approach to
  high-dimensional inference}.
\newblock \emph{arXiv preprint arXiv:1705.05543}, 2017.

\bibitem[Zou(2006)]{zou2006adaptive}
Hui Zou.
\newblock {The adaptive lasso and its oracle properties}.
\newblock \emph{Journal of the American Statistical Association}, 101\penalty0
  (476):\penalty0 1418--1429, 2006.
\newblock ISSN 01621459.
\newblock \doi{10.1198/016214506000000735}.

\end{thebibliography}

\end{document}